\documentclass[twoside]{article}
\usepackage[accepted]{aistats2024}
\usepackage{bm}
\usepackage{array}
\usepackage{dsfont}
\usepackage{amsthm}
\usepackage{amsmath}
\usepackage{amssymb}
\usepackage{mathtools}
\usepackage{todonotes}
\usepackage[ruled, vlined]{algorithm2e}
\usepackage{algpseudocode}
\usepackage{thmtools}

\usepackage{hyperref}
\usepackage{xcolor}
\definecolor{Bleu}{RGB}{30,144,245}
\definecolor{Red}{HTML}{FF617B}
\hypersetup{colorlinks,citecolor=Bleu,linkcolor=Red}

\theoremstyle{plain}
\newtheorem{remark}{Remark}

\newtheorem{theorem}{Theorem}%[section]
\newtheorem{lemma}{Lemma}
\newtheorem{definition}{Definition}
\newtheorem{proposition}{Proposition}

\SetKw{Pull}{pull}
\SetKw{Sample}{sample}
\SetKw{Update}{update}
\SetKw{Observe}{observe}
\SetKwInOut{Output}{return}
\SetKwInOut{Input}{initialize}

\newcommand{\T}{\intercal}
\DeclareMathOperator{\m}{m}
\DeclareMathOperator{\M}{M}

\newcommand{\bP}{\mathbb{P}}
\newcommand{\bA}{\mathbb{A}}

\newcommand{\bE}{\mathbb{E}}

\newcommand{\bR}{\mathbb{R}}

\newcommand{\bN}{\mathbb{N}}
\newcommand{\bX}{\mathbf{X}}

\newcommand{\cN}{\mathcal{N}}
\newcommand{\cH}{\mathcal{H}}
\DeclareMathOperator{\KL}{KL}

\newcommand{\cS}{\mathcal{S}}
\newcommand{\cA}{\mathcal{A}}
\newcommand{\cB}{\mathcal{B}}

\newcommand{\cL}{\mathcal{L}}
\newcommand{\cF}{\mathcal{F}}
\newcommand{\cD}{\mathcal{D}}
\newcommand{\cO}{\mathcal{O}}
\newcommand{\cP}{\mathcal{P}}

\newcommand{\ind}{\mathbb{I}}
\newcommand{\cE}{\mathcal{E}}

\newcommand{\iid}{{ \it i.i.d }}
\newcommand{\tm}[1]{\textrm{#1}}
\newcommand{\impl}{\Longrightarrow}
\DeclareMathOperator*{\argmax}{argmax}
\DeclareMathOperator*{\argmin}{argmin}

\DeclareMathOperator{\mh}{{{m}}}
\DeclareMathOperator{\Mh}{{{M}}}
\DeclareMathOperator{\OPT}{\mathrm{OPT}}

\newcommand{\muh}{\ensuremath{{\widehat{\theta}}}} % Temp 

\newcommand{\vmu}{\ensuremath{\boldsymbol{\theta}}} % Temp 

\newcommand{\vmuh}{\ensuremath{\boldsymbol{\widehat{\theta}}}} % Temp 

\renewcommand{\complement}{\mathsf{c}}
\renewcommand{\mu}{\theta} %% Temp 
\renewcommand{\bA}{[K]} %Temp

\usepackage{hyperref}
\usepackage{enumerate}
\usepackage{natbib}
\usepackage{multirow}
\usepackage{caption}
\usepackage{subcaption}
\usepackage{apptools}
\usepackage{thm-restate}
\usepackage{float}
\usepackage{graphicx}
\usepackage{caption}
\newcommand{\ege}{\hyperref[alg:three]{EGE}}
\newcommand{\egesrk}{\hyperref[alg:amp]{EGE-SR-$k$}}
\newcommand{\apefb}{\hyperref[alg:pucbe]{APE-FB}}

\begin{document}



\twocolumn[

\aistatstitle{Bandit Pareto Set Identification: the Fixed Budget Setting}

\aistatsauthor{Cyrille Kone$^1$ \And Emilie Kaufmann$^1$ \And  Laura Richert$^2$}
%\aistatsauthor{Cyrille Kone \And Emilie Kaufmann \And  Laura Richert}

\aistatsaddress{$^1$ Univ. Lille, Inria, CNRS, Centrale Lille, UMR 9198-CRIStAL, F-59000 Lille, France \\ 
	$^2$ Univ. Bordeaux, Inserm, Inria, BPH, U1219, Sistm, F-33000 Bordeaux, France} ]
\runningauthor{Cyrille Kone, \;Emilie Kaufmann, \;Laura Richert}
\begin{abstract} 
We study a multi-objective pure exploration problem in a multi-armed bandit model. Each arm is associated to an unknown multi-variate distribution and the goal is to identify the distributions whose mean is not uniformly worse than that of another distribution: the Pareto optimal set. We propose and analyze the first algorithms for the \emph{fixed budget} Pareto Set Identification task. We propose Empirical Gap Elimination, a family of algorithms combining a careful estimation of the ``hardness to classify'' each arm in or out of the Pareto set with a generic elimination scheme. We prove that two particular instances, EGE-SR and EGE-SH, have a probability of error that decays exponentially fast with the budget, with an exponent supported by an information theoretic lower-bound. We complement these findings with an empirical study using real-world and synthetic datasets, which showcase the good performance of our algorithms.\end{abstract}

%We study the identification of the Pareto optimal set in a multi-objective multi-armed bandit model in the fixed-budget setting.

% We propose a generic gap-based algorithm for this problem and we show that its probability of error decreases exponentially fast with the budget. We derive an information-theoretic lower bound to justify the complexity terms involved in our results. We also analyze the mis-identification probability when one seeks to identify at most $k$ Pareto optimal arms, showing that our results can be improved for this relaxed setting.  We showcase the good practical performance of our algorithms on real-world and synthetic datasets.  

\section{INTRODUCTION}

The multi-armed bandit problem has been extensively studied in the literature, predominantly as a single-objective stochastic optimization problem. In this framework, an agent sequentially collects samples from a set of $K$ (unknown) probability distributions, called arms. Two common goals are either to maximize the sum of observations (regret minimization, \citep{lattimore_bandit_2020}) or to identify the arm with the largest expected value (best arm identification, \cite{even-dar_pac_2002, audibert_best_2010}). In many real-world problems, however, there are multiple (possibly conflicting) objectives to optimize simultaneously and there might not exist a unique arm that maximizes all these objectives. This leads to a scenario in which the observations are multi-dimensional, and there may be several \emph{Pareto optimal} arms. An arm belongs to the Pareto optimal set $\cS^\star$ if it is not dominated by any other arm (in the sense that the mean values of all the objectives would be larger for the latter). While some line of work aims at minimizing the \emph{Pareto regret} by strategically selecting arms from the Pareto set along the agent's trajectory \citep{drugan_designing_2013}, we are interested in the pure exploration counterpart of this problem, first studied by \cite{auer_pareto_2016}, in which the goal is to identify (relaxations of) the set $\cS^\star$, as quickly and accurately as possible. This multi-objective exploration problem has a wide range of applications including adaptive clinical trials (with multiple indicators of efficacy), software and hardware design (with multiple criteria such as speed and energy consumption), or A/B/n testing for recommender systems (with different metrics of users engagement).

Notably, prior work \citep{auer_pareto_2016, ararat_vector_2021, kone2023adaptive} has studied the Pareto Set Identification (PSI) problem under the \emph{fixed-confidence setting}: given a risk $\delta \in (0, 1)$, the agent stops at a random time $\tau$ and recommends a set $\widehat S_\tau$ such that $\bP(\widehat S_\tau \neq \cS^\star) \leq \delta$, while minimizing its expected stopping time $\bE[\tau]$. However, to the best of our knowledge, no algorithms have been proposed for the dual \emph{fixed-budget setting} in which the agent should output a guess for the Pareto set after a given budget $T$ while minimizing its probability of error $\bP(\widehat S_T \neq \cS^\star)$. For some applications, the fixed-budget setting can be viewed as more practical: e.g., in clinical trials the maximal number of patients is typically defined in advance (due to some financial constraints), and the stopping time used in the fixed-confidence formulation could be prohibitively large. 
  
We fill this gap by proposing the first algorithm(s) for Pareto Set Identification in the fixed-budget setting. Our approach may be viewed as an extension of {Successive~Rejects} (SR) \citep{audibert_best_2010} and {Sequential~Halving} (SH) \citep{karnin_almost_2013}, two state-of-the-art algorithms for fixed-budget best arm identification (i.e. PSI in the uni-dimensional setting), that are based on arm eliminations. We propose a novel arm elimination criterion based on the computation of an \emph{Empirical Gap}. This quantity can be viewed as a measure of the hardness to classify an arm as Pareto optimal or sub-optimal and is inspired by the gaps featured in the analysis of fixed-confidence algorithms \citep{auer_pareto_2016,kone2023adaptive}. While SR and SH successively eliminate arms with smallest empirical means, in Empirical Gap Elimination (EGE), the arms with the largest empirical gaps are eliminated and appropriately classified as optimal or sub-optimal. 

We prove an upper bound on the error probability of a generic version of EGE using any arm allocations satisfying some budget constraints (Theorem~\ref{thm:main-res}). We explicit our bound for two instances using the same arm allocations as SR and SH respectively, called EGE-SR and EGE-SH, showing that their error probability decrease exponentially fast with the budget (Corollary~\ref{thm:sr_sh}). We derive an information-theoretic lower bound showing that these algorithms are essentially optimal in the worse case (Theorem~\ref{thm:thm-lbd-consistent}). On the practical side, we report in Section~\ref{sec:exp} some experimental results showing that EGE-SR/SH outperform uniform sampling and are much more robust than \apefb{}, an adaptation of an existing fixed-confidence algorithm that requires some prior knowledge of the problem complexity as input (discussed in Appendix~\ref{sec:pucbe}). Finally, we propose in Section~\ref{sec:relax}, \egesrk{}, a variant of EGE-SR that tackles a relaxation of PSI first considered by \cite{kone2023adaptive} in which one should identify at most $k$ Pareto optimal arms, with a provably reduced probability of error. 
%% plus de place pour parler des applis 
%% de la diff avec bai

\paragraph{Related work} 
The Best Arm Identification problem (BAI) is one of the most studied pure exploration problems in the bandit literature. The fixed-confidence formulation of BAI is by now well understood as there are efficient algorithms with guaranteed bounds on their sample complexity, $\bE[\tau]$ (e.g., \citep{kalyanakrishnan_pac_2012,gabillon_best_2012,jamieson_lil_2013}) and some algorithms that attain the minimal sample complexity on any (parametric) problem instance in the asymptotic regime $\delta \rightarrow 0$ (e.g., \citep{garivier_optimal_2016, you23a}). In contrast, the theoretical understanding of fixed-budget BAI is more elusive. While the error probability of any reasonable strategy (e.g. uniform sampling) decays exponentially fast with the budget $T$, \cite{degenne23a} shows that there is no optimal exponent that can be attained by a single algorithm on all bandit instances. %In the same flavor,  \cite{wang2023uniformly} showed that for 2-armed Bernoulli bandits, for any algorithm there exists some instances where its performs strictly worse than the uniform allocation strategy.    
Still, the Successive~Rejects and Sequential~Halving algorithms mentioned above have their error probability scaling in $\exp(-T/\log(K)H(\nu))$ for some complexity quantity $H(\nu)$ depending on the bandit instance $\nu$ and \cite{carpentier_tight_2016} proved a (worst-case) lower bound showing that there exists instances $\nu$ in which this rate is un-improvable (up to constant factors in the exponent) when the budget $T$ is large.

In recent years, there has been a growing interest in pure exploration problems in a multi-dimensional setting. Adaptive algorithms for Pareto Set Identification were first proposed in the literature on black-box multi-objective optimization. Besides several evolutionary heuristics (e.g., \cite{nsga,knowles_parego_2006}), the works of \cite{zuluaga_active_2013, zuluaga_e-pal_2016} have studied PSI under a Gaussian Process modeling assumption, obtaining some error bounds that decay polynomially in $T$ and depend on some notion of information gain that is specific to the GP setting. In the bandit literature, the first fixed-confidence PSI algorithm proposed by \cite{auer_pareto_2016} is based on uniform sampling and an accept/reject mechanism. \cite{kone2023adaptive} then proposed a more adaptive LUCB-like algorithm for PSI and some relaxations introduced therein. \cite{ararat_vector_2021} studied an extension of PSI to any $\bR^D$ partial order defined by an ordering cone, still in the fixed confidence setting. 

Alongside PSI, other multi-dimensional pure exploration problems have been investigated such as \emph{feasible arm identification} \citep{katz-samuels_feasible_2018} in which the goal is to identify the arms which belong to a given polyhedron $\cP\subset \bR^D$ %, generalizing the \emph{thresholding bandit problem} \citep{locatelli_optimal_2016} to the multi-dimensional setting. 
or \emph{top feasible arm identification} \citep{katz-samuels_top_2019} in which one seeks a feasible arm that further maximizes some linear combinations of the different objectives. While the former work considers the fixed-budget setting and also proposes some counterpart of Successive~Rejects, the latter considers the fixed-confidence setting. In the fixed-budget setting, a special case was recently considered by \cite{faizal2022constrained} in a  2-dimensional setting in which the goal is to find the arm maximizing one attribute under the constraint that the second attribute is larger than a given threshold. They propose a SR-like algorithm that is similar in spirit to our Empirical Gap Elimination, but with a very different notion of empirical gap. 

\section{SETTING}
We formalize the fixed budget Pareto Set Identification problem and introduce some notation. 

We are given $K$ distributions (or arms) $\nu_1, \dots, \nu_K$ over $\bR^D$ with means (resp.) $\vmu_1, \dots, \vmu_K$.  
Let $\nu:=(\nu_1, \dots, \nu_K)$ be the bandit instance and $\bm\Theta := (\vmu_1 \dots \vmu_K)^\T$. We use boldfaced symbols for $\bR^D$ elements. As in prior work on multi-objective pure exploration, we assume that the marginal distributions of each arm are all $\sigma$-subgaussian. 
 We recall that a random variable $X$ is $\sigma$-subgaussian if $\forall\; u\in \bR$, $\bE[\exp\left( u (X -\bE[X]) \right)] \leq \exp\left( {u^2\sigma^2}/2\right)$. 
%Let $\mathbb{S}_{D-1} := \{ \bm u \in \bR^D : \|\bm u\|_2 = 1\}$ and denote by $ \| \cdot\|_\text{op}$ the operator norm on $\bR^{D\times D}$. 
%\begin{definition}
%	 is $\sigma$-subgaussian if for any $\bm u\in \mathbb{S}_{D-1}$ and $\lambda \in \bR$, 
%	 	$$\bE[\exp\left( \lambda \bm u^\T (\bX-\bE[\bm X]) \right)] \leq \exp\left( \frac{\lambda^2\sigma^2}2\right).$$ 
%\end{definition}
%This includes multivariate normal variables, and distributions whose support belongs to a bounded subset of $\bR^D$. Letting $\bX$ be a random vector with covariance $\Sigma$, it is straightforward to show from the definition that if $\bX$ is $\sigma$-subgaussian then $\|\Sigma\|_\text{op} \leq \sigma^2$. 
%In particular, $\cN(\bm \mu, \Sigma)$ is $\sigma$-subgaussian with  $\sigma^2 = \| \Sigma\|_\text{op}$. 
\begin{definition}
Given two arms $i,j \in [K]$, $i$ is weakly (Pareto) dominated by $j$ (denoted by $\vmu_i \leq  \vmu_j$) if for any $d\in \{1, \dots, D\}$, $\mu_{i}^d \leq \mu_j^d$. The arm $i$ is (Pareto) dominated by $j$ ($\vmu_i \preceq \vmu_j$ or $i\preceq j$) if $i$ is weakly dominated by $j$ and there exists $d \in \{1, \dots, D\}$ such that $\mu_i^d < \mu_j^d$. The arm $i$ is strictly (Pareto) dominated by $j$ ($\vmu_i\prec \vmu_j$ or $i \prec j$) if for any $d \in \{1, \dots, D\}$, $\mu_i^d < \mu_j^d$. 
\end{definition}
 The {Pareto set} $\cS^\star(\nu)$ is \begin{equation*}
    \cS^\star(\nu) := \left\{ i \in [K] \tm{ s.t } \nexists j \in [K]: \vmu_i  \prec \vmu_j \right\},
\end{equation*}
and will be denoted by $\cS^\star$ when $\nu$ is clear from the context. Any arm $a\in \cS^\star$ will be called (Pareto) \emph{optimal} and an arm $a\notin \cS^\star$ is called \emph{sub-optimal}.

At each time $t=1, 2, \dots,T$, the agent chooses an arm $a_t$ and observes an independent draw $\bX_t \sim \nu_{a_t}$ with $\bE(\bX_{a_t}) = \vmu_{a_t}$. We denote by $\bP_{\nu}$ the law of the stochastic process $\{\bX_t\}_{t\geq 1}$ and by $\bE_{\nu}$, the expectation under $\bP_{\nu}$. Let $\cF_t:= \sigma(a_1, \bX_1, \dots, a_t, \bX_t, a_{t+1})$ be the $\sigma$-algebra representing the history of the process up to time $t$ and $\cF:= \{\cF_t\}_{t\geq 1}$, the filtration of the process. The agent's strategy is adaptive in the sense that $a_t$ is $\cF_{t-1}$ measurable. At time $T$, the agent outputs a guess $\widehat{S}_T$ (which is $\cF_T$ measurable) for $\cS^\star$ which should minimize the error probability $e_T(\nu):= \bP_{\nu}(\widehat S_T\neq \cS^\star(\nu))$. 
% The goal of the learner is either to predict $\cS^\star$ with a risk at most $\delta$, while minimizing its expected stopping time $\bE_{\nu}[\tau_\delta]$ (fixed-confidence) or to minimize its probability of error given a total budget $T$ (fixed-budget). In this paper, we focus on this latter objective. Given a budget $T$, the learner outputs a set $\widehat S$ and its goal is thus to minimize $e_T(\nu):= \bP_{\nu}(\widehat S\neq \cS^\star)$. 
% % intro des gaps 

To characterize the difficulty of the problem, we introduce 
below some quantities that allow to measure how much an arm is dominated. For any arms $i,j$, we let 
\begin{eqnarray*}
	\m(i,j) &:=& \min_d \left[\mu_j^d - \mu_i^d\right], \\
	\M(i,j) &:=& \max_d \left[ \mu_i^d - \mu_j^d\right]. 
\end{eqnarray*}
%where $(x)^+ := \max(x, 0)$ for any real $x$ and it is applied component-wise for a vector. 
If $\vmu_i\nprec \vmu_j$, $\M(i,j)$ is the smallest uniform increase of $j$ that makes it dominate $i$. If $\vmu_i\prec \vmu_j$ then $\m(i,j)$ is the smallest increase of any component of $i$ which makes it non-dominated by $j$. 
\cite{auer_pareto_2016} proposed an algorithm in the fixed-confidence setting whose sample complexity is characterized by some sub-optimality gaps $\Delta_i$'s defined as follows. For a sub-optimal arm $i\notin \cS^\star$, 
\begin{equation}
\label{eq:def-gap-sub}
	\Delta_i := \Delta_i^\star:= \max_{j \in \cS^\star} \m(i,j), 
\end{equation}
which is the smallest quantity that should be added component-wise to $\vmu_i$ to make $i$ appear Pareto optimal w.r.t $\{\vmu_k: k\in [K]\setminus \{i\} \}$.
For an optimal arm $i \in \cS^\star$,
\begin{equation}
\label{eq:def-gap-opt}
    \Delta_i := 
    \min (\delta_i^+, \delta_i^-)
\end{equation}
where 
\begin{eqnarray*}
	\delta_i^+ &:=& \min_{j\in \cS^\star \setminus \{i\}} \min(\M(i,j), \M(j,i))\;, \\
	\delta_i^- &:=& \min_{j\in [K]\setminus \cS^\star}[(\M(j,i))^+ + \Delta_{j}]\;,
\end{eqnarray*} with the convention $\min_\emptyset = +\infty$. 
$\delta_i^+$ accounts for how much $i$ is close to dominate (or to be dominated by) another optimal arm while $\delta_i^-$ translates in a sense the smallest ``margin" from an optimal arm $i$ to the sub-optimal arms. We illustrate these quantities in Appendix~\ref{sec:lowerBound}. %We recall that for $x\in \bR, (x)^+:= \max(x, 0)$. Remark that both Equation \eqref{eq:def-gap-sub} and \eqref{eq:def-gap-opt} depends on $\cS^\star$. 
%We use the same gaps in the fixed-budget setting and we characterize the difficulty of the problem with 
\cite{auer_pareto_2016} show that the difficulty (i.e. near-optimal sample complexity) of fixed-confidence PSI is characterized by the complexity term
$$ H(\nu) :=  \sum_{a=1}^K  \frac{1}{\Delta_{a}^{2}}.$$
In this work, we show that it is also a relevant complexity measure for the fixed-budget setting.
%When $D=1$, direct calculation shows that we recover the classical definition of the hardness term for Gaussian (single-objective) bandit in BAI \citep{audibert_best_2010, garivier_optimal_2016}.  

%% 
%% comment 
%While PSI (Pareto set identification) is a natural extension of the best arm identification problem in standard bandits, PSI is actually a more challenging problem. In PSI we don't know the number of optimal arms beforehand: it can go from $1$ to $K$, while in BAI the classical assumption is that there is a single optimal arm. When relaxations are further taken into account, there are multiple valid solutions (e.g. for $\varepsilon_1$-PSI-$k$ when then Pareto set is larger than $k$, any subset of size at $k$ is a valid guess), which creates some additional difficulties even in the uni-dimensional setting [10].

\section{ALGORITHMS}
In this section, we introduce the family of Empirical Gap Elimination (EGE) algorithms. %, and two particular instances. 
% We present a generic bound on their error probability, and specialize it to two instances, EGE-SR and SG-SH which may be viewed as extensions of the Successive Reject \cite{audibert_best_2010} and Sequential Halving \cite{karnin_almost_2013} for fixed budget BAI.

\subsection{Empirical Gap Elimination}
An Empirical Gap Elimination algorithm uses a round-based structure. The algorithm is parameterized by a number of rounds $R>0$, an arm schedule vector $\bm\lambda = (\lambda_1,\dots,\lambda_R,\lambda_{R+1})\in [K]^{R+1}$ satisfying $\lambda_{r} > \lambda_{r+1}$ where $\lambda_r$ indicates how many arms are active in round $r$, and an allocation vector $\bm t= (t_1,\dots,t_R) \in [T]^{R}$, where $t_r$ is the number of samples gathered from each active arm in round $r$. Given the maximal budget $T$, these vectors should further satisfy the following:
\begin{align}
&\lambda_1 = K \quad \text{and}\quad \lambda_{R+1} \in \{0, 1\} 
\label{eq:cond1} \\
&\text{and} \quad \sum_{r=1}^R \lambda_r t_r \leq T,	\label{eq:cond2}
\end{align}
ideally with an equality.   

%\hyperref[alg:three]{EGE}.
In each round $r$, the algorithm maintains a set of active arms, denoted by $A_r$, that is of size $\lambda_r$, and collects $t_r$ new samples from each arm in $A_r$. The total number of samples from each arm in $A_r$ gathered in the first $r$ rounds\footnote{We emphasize that EGE algorithms do not discard samples between rounds} is therefore 
$$n_r = \sum_{s= 1}^r t_s.$$
For each $a\in A_r$, the empirical estimate of $\bm\theta_a$ at the end of round $r$ is denoted by
$\vmuh_a(r) := \frac{1}{n_r}\sum_{s=1}^{n_r} \bX_{a, s}$
where $\bX_{a,s}$ denote the $s$-th observation drawn \!\!\!\iid from distribution $\nu_a$. These estimates are used to carefully decide which arm to explore in the next round, based on an appropriate notion of \emph{empirical gap}. 

These gaps are empirical variants of the (fixed-confidence) gaps introduced in the previous section. We first propose a rewriting of these quantities to remove the explicit dependency on $\cS^\star(\nu)$.

\begin{restatable}{lemma}{lemGapRedef}
\label{lem:lem-gap-redef}
	For any arm $i \in [K]$, 
\[\Delta_i = \left\{\begin{array}{ll}
	               \Delta_i^\star = \max_{j\in [K]}\m(i, j) & \text{ if } i\notin \cS^\star \\
	                \delta_i^\star & \text{ if } i\in \cS^\star    \\
	                    \end{array}\right.\;,
\]		
where $\delta_i^\star := \min_{j\neq i} [\M(i, j)\land (\M(j, i)^+ +(\Delta_j^\star)^+)].$
\end{restatable}

We introduce the empirical quantities 
\begin{eqnarray*}
	\mh(i,j; r) &:=& \min_{d}\left[\muh_j^d(r) -  \muh_i^d(r)\right], \\ 
	\Mh(i,j; r) &:=& \max_d \left[(\muh_i^d(r) - \muh_j^d(r))\right]. 
\end{eqnarray*}
and the empirical Pareto set
\begin{eqnarray*}S_r &:=& \{i \in A_r : \nexists j \in A_r : \vmuh_i(r) \prec \vmuh_j(r) \},\\
 & = &\{i \in A_r: \forall j  \in A_r \backslash \{i\}, \Mh(i,j;r) > 0 \}\;.
\end{eqnarray*}

Finally, we define for any arm $i \in A_r$
\begin{eqnarray*}
\widehat \Delta_{i, r}^\star &:=& \max_{j\in A_r\backslash\{ i\}}\mh(i,j;r), \\
\widehat \delta_{i,r}^\star &:=& \!\!\min_{j \in A_r \backslash\{i\}} [\Mh(i,j;r) \;\land\; (\Mh(j,i;r)^+ + (\widehat \Delta_{i, r}^\star)^+)]
\end{eqnarray*}
 and for any $i\in A_r$
\begin{eqnarray}
\label{eq:def-emp-gap}
	\widehat \Delta_{i,r} := \begin{cases}
		\widehat \Delta_{i, r}^\star & \text{ if } i \in A_r \backslash{S_r},\\
		\widehat\delta_{i,r}^\star & \text{ else. }
	\end{cases}
\end{eqnarray}

\begin{algorithm}[t]
\caption{Empirical Gap Elimination (EGE)}
\label{alg:three}
\KwResult{Pareto set}
\KwData{parameters $R, \bm t, \bm \lambda$}
\Input{Set $A_1 = \{1, \dots, K\}, B_1 := \emptyset$, $D_1 = \emptyset$} 
 \For{$r=1, 2, \dots, R$}{

  Collect $t_r$ samples from each arm $a \in A_r$
  
  Compute $S_r$ the empirical Pareto set
  
  Let $A_{r+1}$ be the set of $\lambda_{r+1}$ arms in $A_r$ with the smallest empirical gaps $\widehat\Delta_{a,r}$
  
 \tcp{ties broken in favor of arms in $S_r$}
 
 $ B_{r+1} \gets B_r \cup \{ S_r \cap (A_r \backslash A_{r+1})\}$ \; 
 
 $ D_{r+1} \gets D_r \cup \{  (A_r \backslash A_{r+1}) \backslash S_r \}$ \; 
 
 }
 \Output{$\widehat S_T = B_{R+1}  \cup A_{R+1}$}
\end{algorithm}

%\todoe{Now we need to explain to the reader that these gaps can indeed be interpreted as empirical counterpart or the gaps that are defined differently in the intro}
At the end of round $r$, the algorithm sorts the arms by increasing order of their empirical gaps: $\widehat{\Delta}_{(1),r} \leq \dots \leq  \widehat{\Delta}_{(\lambda_r),r}$ and we define $A_{r+1} = \left\{(1),\dots,(\lambda_{r+1})\right\}$. In case of ties, i.e. if $\widehat{\Delta}_{(\lambda_{r+1}),r} = \widehat{\Delta}_{(\lambda_{r+1} + m),r}$ for some $m$, we first add arms in $S_r$ to $A_{r+1}$. We emphasize that this tie-breaking rule is crucial in our analysis. %\todoem{say more?}. 
Arms in $A_{r+1}\backslash A_r$ are further classified as optimal (and added to the set $B_{r+1}$) or sub-optimal (and added to the set $D_{r+1}$) based on whether or not they belong to $S_r$. The output of the algorithm is the set $B_{R+1}\cup A_{R+1}$.  

\subsection{Particular Instances} 
Arm elimination algorithms have been proposed for different (mostly uni-dimensional) fixed-budget identification tasks \citep{audibert_best_2010,bubeck_multiple_2013,karnin_almost_2013,katz-samuels_feasible_2018} with different elimination rules, that could also be rewritten featuring some (simpler) gaps. In these works, two different types of arm schedule and sampling allocations have been mostly investigated: 
\begin{itemize}
 \item In Successive~Rejects (SR), one arm is de-activated in each round, ie. $R=K-1$ and $\lambda^{\text{SR}}_{r}=K-r+1$ for all $r \leq K$. The sampling allocation proposed by  \cite{audibert_best_2010} satisfies $t^{\text{SR}}_r = n^{\text{SR}}_{r} - n^{\text{SR}}_{r-1}$ with $n_r^{\text{SR}} = \left\lceil  \frac{1}{\overline{\log (K)}} \frac{T-K}{K+1-r}\right\rceil$ where $\overline{\log (K)} := 2^{-1} +\sum_{i=2}^{K}i^{-1}$ and $n_0^{\text{SR}}=0$.
 \item In Sequential~Halving (SH), one half of the active set is de-activated in each round, that is $R=\lceil\log_2(K)\rceil$ and for all $r\in \{1,\dots, \lceil\log_2(K)\rceil\}$, $ \lambda^{\text{SH}}_{r+1} := \lceil \lambda_r^{\text{SH}}/2\rceil$ (we easily very that $\lambda^{\text{SH}}_{R +1} = 1$). The sampling allocation proposed by \cite{karnin_almost_2013} is uniformly spread across rounds, that is $t_r^{\text{SH}} := \left\lfloor \frac{T}{\lvert A_r \lvert \lceil\log_2(K)\rceil } \right\rfloor$.
\end{itemize}
We refer to EGE-SR (resp. EGE-SH) as the instances of EGE using the same allocation as SR (resp. SH). In Appendix~\ref{sec:alt_alloc} we propose a third instantiation using the geometric allocation of \cite{karpov2022collaborative}.

%\todoem{Refer to other possible allocations? }
\begin{remark} For $D=1$, the PSI problem coincides with BAI and EGE-SR (resp. EGE-SH) coincides with SR (resp. SH\footnote{Besides the fact that the original SH algorithm for BAI discards samples collected in previous rounds to compute the empirical means.}). Indeed, in that case $S_r = \{\widehat a_r\}$ reduces to the empirical best arm, $\widehat \Delta_{i, r} = \muh_{\widehat a_r, r} - \muh_{i, r}$ for $i\neq \hat{a}_r$, $\widehat \Delta_{\widehat a_r, r} = \min_{i\in A_r\backslash\{\widehat a_r\}} \widehat \Delta_{i, r}.$
Then, our tie-breaking rule ensures that no arm is accepted before the last round and at each round $r$, $A_{r+1}$ is defined as the $\lambda_{r+1}$ arms in $A_r$ with the largest empirical means and the final survival arm is recommended as optimal. 
\end{remark}

\subsection{Alternative Approach}
Another idea to tackle fixed-budget PSI is to adapt an existing fixed-confidence algorithm to that setting. The APE algorithm of \cite{kone2023adaptive} takes  $\delta \in (0, 1)$ as input and uses a stopping time $\tau_\delta$ such that \[ \bP\left(\cS^\star =  S(\tau_\delta) \text{ and } \tau_\delta \leq C H(\nu)f(\delta)\right) \geq 1 - \delta, \]
where $S(\tau_\delta)$ is the empirical Pareto set at time $\tau_\delta$ and $f$ is a function of $\delta$ and $C$ is a constant. The idea is then to choose $\delta$ so as $C H(\nu)f(\delta) \leq T$, that is tune $\delta$ w.r.t $H(\nu)$ and $T$. This is roughly the approach used 
by UGapEb \citep{gabillon_best_2012} in BAI. In Appendix~\ref{sec:pucbe}, we analyze \apefb{}, a fixed-budget version of APE that takes as input a parameter $a\geq 0$ and we prove an upper-bound on its probability of error when $a\leq \frac{25}{36}\frac{T-K}{H(\nu)}$. However, such an ``oracle'' tuning is not very satisfying, as assuming $H(\nu)$ to be known is quite unrealistic in practice. This is why in the sequel we focus on presenting the analysis of \ege{} which does not require any prior knowledge about $\nu$.

\section{THEORETICAL GUARANTEES}
We first propose an analysis of \ege{} for a generic number of rounds $R$, arm schedule $\bm\lambda$ and sampling allocation $\bm t$ satisfying \eqref{eq:cond1} and \eqref{eq:cond2}. It features the quantity
%\[H_3^{R,\bm t, \bm \lambda}(\nu)  := \min_{r\in [R]} n_r\Delta_{(\lambda_{r+1} +1)}^2.\]
\[\widetilde T^{R,\bm t, \bm \lambda}(\nu)  := \min_{r\in [R]} \left(\sum_{s=1}^{r}t_s\right)\Delta_{(\lambda_{r+1} +1)}^2,\]
in which the dependency in $\nu$ is captured in the gaps. 

%\todoe{I rewrote the complexity to make more apparent the dependency in $t_r$.}
%ok
%\todoe{Maybe it is confusion to call that a complexity, as it is a $T/H$... Use another letter than $H$? Maybe $\widetilde{T}^{R,\bm t, \bm \lambda}(\nu)$ would be more appropriate, and also suggests the dependency in $T$} 
% ok [C]
\begin{restatable}{theorem}{mainRes}
\label{thm:main-res}
Let $\nu$ be a bandit with marginally $\sigma$-subgaussian arms. Then 
 Empirical Gap Elimination with parameters $R, \bm\lambda$ and $\bm t$ satisfies
 $$ e^{\text{EGE}}_T(\nu) \leq 2(K-1) \lvert \cS^\star\lvert R D  \exp\left( - \frac{\widetilde T^{R,\bm t, \bm\lambda}(\nu)}{144\sigma^2} \right).$$
 \end{restatable}
 This result shows that the probability of failure of \ege{} decreases exponentially fast with $\widetilde T^{R,\bm t, \bm \lambda}(\nu)$. 
% \todoe{Instead of writing \autoref{alg:three} everywhere, I would use algorithms names. We can put them in color and use hyper-ref on these names (I've done that in previous papers, it reads well)}
% ok 
In Appendix~\ref{subsec:cor-proof}, we further show that for both EGE-SR and EGE-SH, $\widetilde T^{R,\bm t, \bm \lambda}(\nu)$ is of order $T/(H_2(\nu)\log(K))$ with
\[   H_2(\nu) := \max_{i \in [K]} i \Delta_{(i)}^{-2},\]
where $(\cdot)$ is a permutation such that $\Delta_{(1)}\leq \dots\leq \Delta_{(K)}$. More precisely, we obtain the following. 
\begin{restatable}{corollary}{corsr}
\label{thm:sr_sh}
 	Let  $T\geq K$ and $\nu$ be a bandit with $\sigma$-subgaussian marginals. Then \emph{EGE-SR} satisfies 
 {\small $$ e_T^\text{SR}(\nu) \leq  2(K-1)^2 \lvert \cS^\star \lvert D\exp\left( - \frac{T-K}{144\sigma^2H_2(\nu)\overline\log(K)}\right),$$}
 and for \emph{EGE-SH}, $e_T^\text{SH}(\nu)$ is upper-bounded by  
 {\small $$2(K-1)\lceil\log_2(K) \rceil \lvert \cS^\star \lvert  D\exp\left( - \frac{T}{288\sigma^2 H_2(\nu) \lceil\log_2(K)\rceil}\right).$$}
\end{restatable}
%\begin{proof}
%	For EGE-SR, we have $R=K-1$, $\lambda_r^{\text{SR}}=K+1-r$ and $t^{\text{SR}}_r = n_r^{\text{SR}} - n_{r-1}^{\text{SR}}$ where  $n_r^{\text{SR}} = \left\lceil  \frac{1}{\overline{\log (K)}} \frac{T-K}{K+1-r}\right\rceil$, which yields 
% 	\begin{eqnarray*}
%		\widetilde T^{R,\bm t^{\textrm{SR}}, \bm\lambda^{\textrm{{SR}}}}(\nu)  &:=& \min_{r\in [K-1]} n_r^{\text{SR}}\Delta_{(\lambda_{r+1}^{\text{SR}} +1)}^2,\\
%		&\geq& \min_{r\in [K-1]} \frac{\Delta_{(K+1-r)}^2}{\overline{\log (K)}} \frac{T-K}{K+1-r},\\
%		&=& \frac{T-K}{\overline \log (K)} \frac{1}{\max_{r\in \{2,\dots K\}}r\Delta_{(r)}^{-2}}\\
%		&=& \frac{T-K}{\overline \log (K)} \frac{1}{H_2(\nu)}. 
%	\end{eqnarray*}
%\end{proof}
%
The complexity measure $H_2(\nu)$ featured in our error exponent satisfies $H_2(\nu) \leq H(\nu) \leq H_2(\nu)\log(2K)$ as proved by \cite{audibert_best_2010}. For BAI ($D=1$), we essentially recover the existing guarantees for SR and SH, whose error bounds also feature $H_2(\nu)$, up to constant factors inside the exponential and an extra multiplicative $K$ factor for Sequential~Halving. Still, to our knowledge this is the first analysis of the variant of SH that does not discard samples between rounds, which often performs (much) better in practice.  

In the general case ($D\geq 1$) we remark that the bounds obtained for EGE-SR and EGE-SH are hard to compare: the latter has an improved polynomial dependence ($K\log_2(K)$ instead of $K^2$) but a worse constant inside the exponential. As we shall see in the experiments, both algorithms have actually pretty close performance (like in BAI, see \cite{karnin_almost_2013}). Moreover, they both outperform a simple baseline using Uniform Allocation (UA) and recommending the Pareto set of the empirical means. We remark that Theorem~\ref{thm:main-res} yields an upper bound on the error probability of this strategy (by choosing $R=1$, $t_1 =T/K$ and $\lambda_2=0$, for which $\widetilde{T}^{1,\bm t,\bm\lambda}(\nu) = n_1\Delta_{(\lambda_2 +1)}^2 = T/(K\Delta_{(1)}^{-2})$), which we add to our summary in Table~\ref{tab:summary}. This bound can be much worse than that for EGE-SR/SH when the gaps are distinct. 

% \begin{equation}
% \label{eq:prob-ua}
% e_T^\text{UA}(\nu) \leq 2(K-1)\lvert \cS^\star \lvert D\exp\left( - \frac{T}{144\sigma^2 (K\Delta_{(1)}^{-2})}\right).	
% \end{equation}

% We remark that the decay factor in the exponential is larger for the Successive Rejects and Sequential Halving instanciations, in particular when $K$ is large. 
% However, when all the gaps are of the same order, i.e $\Delta_1 \approx \dots\approx \Delta_K$ we have $K\Delta_{(1)}^{-2}\approx H_2\approx H$  and Equation \eqref{eq:prob-ua} becomes, 
% $$ e_T^\text{UA}(\nu) \leq 2(K-1)\lvert \cS^\star \lvert D\exp\left( - \frac{T}{144\sigma^2 H(\nu) }\right),	 $$
% which is better than both Successive Rejects and Sequential Halving.  It is known in BAI \citep{audibert_best_2010}, the uniform allocation strategy is optimal in the fixed-budget when the sub-optimality gaps are all close.  

%% TRANSITION (removed for space)
%To assess the tightness of our upper bounds, we now present a worse case lower bound.

\begin{table}[t]
	\centering
	\begin{tabular}{|c|c|}
		\hline
Algorithm & Error probability\\
\hline\hline 
EGE-SR\!& $K^2\lvert\cS^\star\lvert D\!\exp\!\left(\!-T/(H_2(\!\nu\!)\log K)\right)$\\\hline \hline 
EGE-SH\!& $ K\!\log(\!K\!)\lvert\cS^{\!\star}\lvert D\!\exp\!\left(\!-T/(2H_2(\!\nu\!)\log\!K)\!\right)$\\\hline\hline 
APE-FB\!$^\star$ & $K\log(T)D\!\exp\!\left(\!-T/H(\!\nu\!)\right)$
\\\hline\hline 
UA & $K\lvert\cS^\star\lvert D\!\exp\!\left(\!-T/(K \Delta_{(1)}^{-2})\right)$\\\hline 
	\end{tabular}
	\caption{Upper bounds on $e_T(\nu)$ for different algorithms (up to constants). $^\star$\apefb{} tuned with $H(\nu)$.\label{tab:summary}}
\end{table}
% Recall that when $D=1$, PSI is exactly \emph{best arm identification}. Compared to the analysis Sequential~Halving for BAI (\cite{karnin_almost_2013}), we have an extra polynomial $K$ factor. This is due to the fact that our proof uses a good event which holds for any arm. Unlike \cite{karnin_almost_2013} we allow all the samples collected so far to be used (no re-starting). To our knowledge, this is the first analysis of {Sequential Halving} without \emph{re-starting} in BAI (although it is long known 
% that it performs very well without re-starting).    

\subsection{Lower Bound}
%\todoe{Insert the lower bound here (and comment on its tightness)}
We present in this section a lower-bound for some class of instances. We define $\cB$ to be the set of means $\bm\Theta\in \bR^{K\times D}$ such that each sub-optimal arm $i$ is only dominated by a single arm, denoted by $i^\star$ (that has to belong to $\cS^\star$) and that for each optimal arm $j$ there exists a unique sub-optimal arm which is dominated by $j$, denoted by $\underline j$.  %for any instance $\nu$ with means $\bm\Theta$: (1) for any $i\in \cD(\nu)$ 
We further assume that optimal arms are not too close to arms they don't dominate:
%$$ \forall i,j \in \cS^\star,\; \M(i, j)\geq 3\max(\Delta_{\underline i}^\star, \Delta_{\underline j}^\star),$$
%and 
for any sub-optimal arm $i$ and optimal arm $j$ such that $\vmu_i\nprec \vmu_j$, 
$$ \M(i, j)\geq 3\max(\Delta_i, \Delta_{\underline j}).$$
Let  $\nu:= (\nu_1, \dots, \nu_K)$ be an instance whose means $\bm\Theta \in \cB$ and such that $\nu_i \sim \cN(\vmu_i, \sigma^2 I)$. For every $i\in [K]$ we define the alternative instance $ \nu^{(i)} := (\nu_1, \dots \nu_i^{(i)}, \dots, \nu_K)$ in which only the mean of arm $i$ is modified to:
\begin{equation}
	\vmu_i^{(i)} := \begin{cases}
		\vmu_i - 2\Delta_i
		e_{d_{\underline i}} & \text{ if } i \in \cS^\star(\nu) ,\\
		\vmu_i + 2\Delta_i e_{d_i}&\text{ else},
	\end{cases}
\end{equation}
where $e_1,\dots, e_D$ denotes the canonical basis of $\bR^D$ and $d_i :=\argmin_d [\mu_{i^\star}^d - \mu_i^d]$. With $\nu^{(0)}:=\nu$, we prove the following.
\begin{restatable}{theorem}{thmLbdConsistent}
\label{thm:thm-lbd-consistent} 
Let $\bm\Theta:=(\vmu_1 \dots \vmu_K)^\T \in \cB$ and $\nu=(\nu_1, \dots, \nu_K)$ where $\nu_i \sim \cN(\vmu_i, \sigma^2I)$. For any algorithm $\cA$, there exists $i\in \{0, \dots, K\}$ such that %$ H(\nu^{(i)}) =  H(\nu)$ and 
$$ e_T^{\cA}(\nu^{(i)}) \geq \frac14 \exp\left( - \frac{2T}{\sigma^2 H(\nu^{(i)})}\right).$$ 
\end{restatable}
In particular, there exists some instances $\widetilde{\nu}\in \cB$ such that $e_T^{\cA}(\tilde{\nu}) \geq \frac14 \exp\left( - {2T}/({\sigma^2 H(\widetilde{\nu})})\right)$. On such instances, the decay rate of EGE-SR and EGE-SH is optimal up to constants and $\log(K)$ factors, and that of \apefb{} is optimal up to constant factors, when the complexity is known. 
In Appendix~\ref{sec:lowerBound} we prove a lower bound that holds for a larger class of instances. 
%$$ \forall i, j \in \cS^\star, \M(i,j)\geq 4 \max(\Delta_{\underline i}^\star, \Delta_{\underline j}^\star).$$ 

\subsection{Sketch of proof of Theorem~\ref{thm:main-res}}
%\todoe{More detail is we have space.}
%In this section we sketch the proof of Theorem~\ref{thm:main-res}.
We define for any arms $i, j$ and round $r$ the events
\begin{eqnarray*}
\xi_{i,j, r} \!\!\!\!&:=&\!\!\!\! \left\{\left \|(\vmuh_{i, n_r} \!- \vmuh_{j, n_r}) - (\vmu_i \!-\vmu_j)  \right\|_{\infty}\!\!\! \leq c \Delta_{(\lambda_{r+1} +1)}\right\} \\ 
\cE_c^1 \!&:=&\! \bigcap_{r\in [R]}\bigcap_{i \in \cS^\star}\bigcap_{j \in [K]}\;  \xi_{i, j, r} \text{ for any } c>0.
\end{eqnarray*}
We shall prove that there exists some $c>0$ such that \ege{} does not make any error on the event $\cE_c^1$. That is, no sub-optimal arm is added to $B_r$ and no  optimal arm is added to $D_r$, in any round $r$, and the possibly remaining arm in $A_{R+1}$ is an optimal arm.    
%In fixed-budget BAI, Successive~Rejects and Sequential~Halving rejects $K-1$ arms during a certain number of rounds and 
%eventually recommends the last survival as optimal. Such a strategy which only rejects arms is not possible in fixed-budget PSI as we do not know in advance the number of optimal arms. Thus we need accepts and rejects and this should be done carefully as some arms are not comparable.  
%\todoe{The paragraph above has nothing to do in the proof sketch, it rather provide intuition on the algorithm. Remove and add some elements earlier in the text?}
%% ok 

To do so, an important step is to justify that any sub-optimal arm should be de-activated before the optimal arm that dominates it the most. More formally, for any sub-optimal arm $i$, we let $i^\star \in \argmax_{j\in \cS^\star} \m(i, j)$, which by definition is such that $\Delta_i = \m(i, i^\star)$. For a sub-optimal arm $i$, we know that $i^\star \in \cS^\star$ always exists. More importantly, $i^\star$ could be the only arm dominating $i$.  
Therefore it is crucial to ensure that $i$ is no longer active before discarding $i^\star$, otherwise $i$ could appear as optimal w.r.t the remaining active arms. Another challenge is to properly estimate the gaps of the active arms. For this purpose we novelly prove Lemma~\ref{lem:lem-gap-redef} using geometrical properties of the Pareto set (cf Appendix~\ref{sec:tech_lemmas} for its full proof). Using this lemma and defining $$\cP_r := \left\{ \forall \; i \;\notin \cS_\star,\;  i \in A_r \Rightarrow i^\star \in A_r\right\},$$ we first prove the following concentration result. 

\begin{restatable}{lemma}{lemGapIneq}
\label{lem:lem-gap-ineq} Assume that $\cE_c^1$ holds. Let $r\in [R]$ such that $\cP_r$ holds. Then, for any sub-optimal arm $i\in A_r$,
	$$\lvert \widehat \Delta_{i, r}^\star - \Delta_i^\star \lvert \leq 2c\Delta_{(\lambda_{r+1}+1)} $$
	and for any optimal arm $i \in A_r$, $$ \widehat \delta_{i,r}^\star \geq \Delta_i -2c \Delta_{(\lambda_{r+1} +1)}.$$
\end{restatable}

This result then permits to prove by induction that $\cP_r$ holds in any round $r$, when $c$ is small enough. 
 
\begin{restatable}{lemma}{lemref}\label{lem:lemref} Let $c < 1/6$. 
	On the event $\cE_c^1$, for any $r\in [R+1]$, $\cP_r$ holds.  %and for any sub-optimal arm $i$, if $i\in A_r$ then $i^\star \in A_r$. 
	In particular, for any sub-optimal arm $i$, $i^\star$ cannot be deactivated before $i$.  
\end{restatable} 
A first consequence is that if $A_{R+1}$ contains one arm (i.e. $\lambda_{R+1}=1$) and $\cE_c^1$ holds, then it is an optimal arm.  

Note that in BAI (i.e PSI with $D=1$), for any sub-optimal arm $i$, $i^\star$ is the unique best arm and Lemma~\ref{lem:lemref} is sufficient to ensure the correctness of EGE. But in the general case we need to ensure that no optimal arm is rejected and no sub-optimal arm is accepted. This is done by using Lemma~\ref{lem:lem-gap-ineq} and Lemma~\ref{lem:lemref}. 
\begin{restatable}{lemma}{newlemma}\label{lem:newlemma} Let $c < 1/6$. 
	On the event $\cE_c^1$, the recommendation of \ege{} is correct  i.e $\widehat{S}_T  = \cS^\star$.  
\end{restatable} 
Theorem~\ref{thm:main-res} then follows by upper-bounding the probability of $\bar{\cE_c^1}$ using Hoeffding's inequality.

\section{RELAXING PSI}\label{sec:relax}
In this section we explain how EGE-SR can be slighlty adapted to tackle the ``at most $k$ optimal arms'' relaxation (or PSI-$k$) first introduced by \cite{kone2023adaptive} in the fixed confidence setting. 

In this problem the goal is to return a subset $\widehat{S}_T$ of $\cS^\star$ of size $k$ , or $\cS^\star$ itself if its size is smaller than $k$. In the fixed budget setting, we define the $k$-relaxed expected loss as $e_{T, k}(\nu) := \bE_\nu[\cL(\widehat S_T, k)]$ where   
%\section{RELAXED PARETO SET IDENTICIATION}
%In this section, we propose an algorithm for the 
%We study the ``at most $k$ optimal arms" or PSI-$k$ problem \citep{kone2023adaptive} in the fixed-budget setting. For this problem, given $k\in [K]$, the goal is to find $k$ (Pareto) optimal arms. In case the size of the Pareto set is less than $k$, we should find the exact Pareto set. This problem has applications in many domains including design optimization, A/B/n testing and clinical trials. In clinical trials for example, one might imagine a phase I trial with $K\gg1$ arms where due to resources limitations no more than $k$ arms could be investigated in further phase II/III. %% plus explicite, donner exemple. 
%We propose a modification of \ege for this problem and we show that the probability of error improves upon the results obtained so far for the un-relaxed problem. Essentially, we stop as soon $\lvert B_{r+1}\lvert = k$ and we recommend $B_{r+1}$ upon stopping (see the full description in \egesrk). 
%Given $k\in [K]$, an algorithm for PSI-$k$  should return a set $\widehat S$ and its loss is 
%>>>>>>> 70a52dc (merge)
\begin{eqnarray*}
	\cL(\widehat S, k) := \begin{cases}
		\ind\{\widehat S \subset \cS^\star\} & \text{if} \quad  \lvert \widehat S \lvert = k, \\
		\ind\{\widehat S = \cS^\star\} & \text{else}. 
	\end{cases}
\end{eqnarray*}
To minimize $e_{T, k}(\nu)$ for any parameter $k$ and budget $T$, we propose \egesrk{}, a variant of EGE-SR which may stop at some round $r < K-1$ if $\lvert B_{r+1}\lvert = k$ and recommend $B_{r+1}$ (see pseudocode in algorithm~\ref{alg:amp}). 

% This problem has applications in many domains including design optimization, A/B/n testing and clinical trials. In clinical trials for example, one might imagine a phase I trial with $K\gg1$ arms where due to resources limitations no more than $k$ arms could be investigated in further phase II/III. %% plus explicite, donner exemple. 

In Appendix~\ref{subsec:relax}, we prove an upper-bound on the expected loss $e_{T, k}(\nu)$. To introduce it, we define $\omega_{(k)}$ to be the $k$-th largest gap among the optimal arms : 
$ \omega_{(k)}:= \max_{i\in \cS^\star}^{k} \Delta_i$ with $\omega_{(k)} = 0$ if $|\cS^\star| < k$. Our bound features the complexity measure $H_2^{(k)}(\nu) :=\max_{i\in [K]} i(\Delta^{(k)}_{(i)})^{-2}$ with the $k$-relaxed gaps 
\begin{equation}
\Delta_i^{(k)} := 
	\begin{cases}
		\max(\Delta_i, \omega_{(k)}) & \text{if} \quad i \in \cS^\star \\
		\Delta_i &\text{else.} 
	\end{cases}
\end{equation}

%we can now state the main result of this section. 

\begin{algorithm}[t]
\caption{EGE-SR-$k$}\label{alg:amp}
\KwResult{At most $k$ Pareto optimal arms}
\Input{$k\in [K]$, Set $A_1= [K]$, $B_1 = \emptyset$, $D_1 = \emptyset$} 

 \For{$r=1, 2, \dots, K-1$}{
 
 Collect $t_r$ samples from each arm $a \in A_r$
  
 Compute $S_r$ the empirical Pareto set
  
 {Choose} $i_r \in \argmax_{i\in A_r} \widehat \Delta_{i, r}$  \tcp{ties broken in favor of arms in $S_r$}
  
  $A_{r+1} \gets A_r \backslash\{ i_r\}$\;
  
  $B_{r+1} \gets B_r \cup \{ S_r \cap \{i_r\}\}$ \; 
  
  $D_{r+1} \gets D_r \cup \{\{i_r\} \cap (A_r\backslash S_r) \}$ \; 
  
  \If{$\lvert B_{r+1} \lvert=k$}{
  \textbf{break and return } $\widehat S_T = B_{r+1}$  \; %\tcp{optional early-stopping}
}  
 }
 \Output{$\widehat S_T = B_{K}  \cup A_K$}
\end{algorithm}

\begin{restatable}{theorem}{mainAMP}
\label{thm:main-amp}
Let $k\in [K]$. \egesrk{} satisfies {%\small
$$e_{T, k}(\nu) \!\leq 2(K-1)^{2}\lvert \cS^\star \!\lvert D\exp\!\left(\!- \frac{T-K}{144\sigma^2H_2^{(k)}(\nu)\overline\log(K)}\!\right)\!.$$ }
\end{restatable}  

This result is particularly insightful when there are many optimal arms and some of them are easy to identify as such (large gaps). 
Indeed when $\lvert \cS^\star \lvert \approx K$ and $k\ll \lvert \cS^\star\lvert$, $H_2^{(k)}(\nu)$ can be an order of magnitude smaller than $H_2(\nu)$. 
We also note that when  $k>\lvert \cS^\star \lvert$ (then PSI-$k$ reduces to PSI), $H_2^{(k)}(\nu) = H_2(\nu)$ and we recover the result of Corollary~\ref{thm:sr_sh}.  

The stopping time of \egesrk{} is $$ \tau := \inf\left\{r: \lvert B_{r+1}\lvert  = k\right\} \land (K-1).$$ 
Letting $N_{\tau}$ denote the total number of samples used at termination, we upper-bound $\bE_\nu[\tau]$ and $\bE_\nu[N_\tau]$ in Appendix~\ref{subsec:relax}, showing that when the budget is large, we essentially have $\bE_\nu[\tau]\leq q$ and $\bE_\nu[N_\tau] \leq N_q$, with $q:=K-\lvert \cS^\star \lvert + k$. Intuitively it suggests that in the worst case the $(K-\lvert \cS^\star \lvert)$ sub-optimal arms are discarded before $k$ optimal arms are accepted as they might be needed to dominate some sub-optimal arms. Moreover, we remark that $q$ can be way smaller than $(K-1)$. For instance when $[K]= \cS^\star$, we have $q=k$. 

\section{EXPERIMENTAL STUDY}\label{sec:exp}
We evaluate our algorithms on synthetic and real-world tasks.
We compare to Uniform Allocation (UA) and \apefb{} for three parameters 
of the form $a_c = c \frac{25}{36}\frac{T -K}{H(\nu)}$, $c\in \{1/10, 1, 10\}$. Our guarantees on \apefb{} are only valid for $c\leq 1$ and are optimal for $c=1$. We consider a heuristic version of \apefb{} which adaptively estimates the hardness $H(\nu)$. We refer to this algorithm as APE-FB-ADAPT. We run the experiments 4000 times with different seeds and we report the $\log_{10}$ of the average mis-identification rate.  
%uses a parameter $0\leq a\leq \frac{25}{36}\frac{T}{H(\nu)}$ an its mis-identification error decays roughly in  $\exp(-T/a)$ (see Appendix~\ref{sec:pucbe} for the study of this algorithm). The major bottleneck of this algorithm is the parameter $a$ which depends on $H(\nu)$ so needs $H(\nu)$ be known to be properly tuned. 
 
\begin{figure*}[]
\centering
 \begin{minipage}{0.95\textwidth}
      \centering
  \begin{minipage}{0.3\linewidth}
          \begin{figure}[H]
              \includegraphics[width=\linewidth]{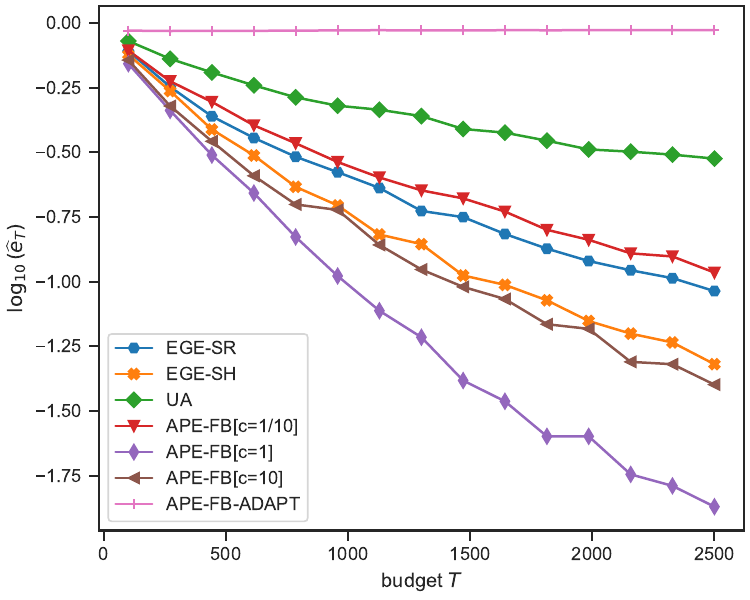}	
              \caption{Application 1: COV-BOOST trial}
          \end{figure}
      \end{minipage}
\hspace{0.5cm}
      \begin{minipage}{0.3\linewidth}
          \begin{figure}[H]
              \includegraphics[width=\linewidth]{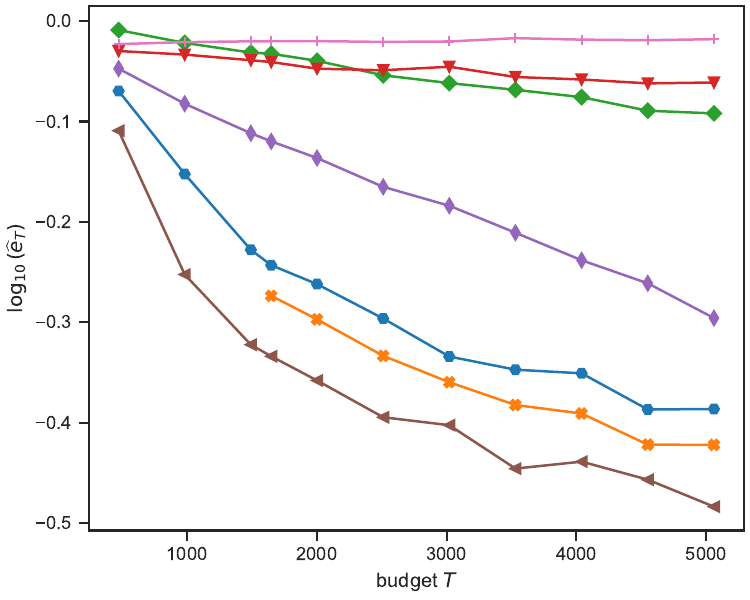}
              \caption{Application 2: Sorting Networks dataset.}
              \label{fig:sorting_net}
          \end{figure}
      \end{minipage}
      \hspace{0.5cm} 
      \begin{minipage}{0.3\linewidth}
          \begin{figure}[H]
              \includegraphics[width=\linewidth]{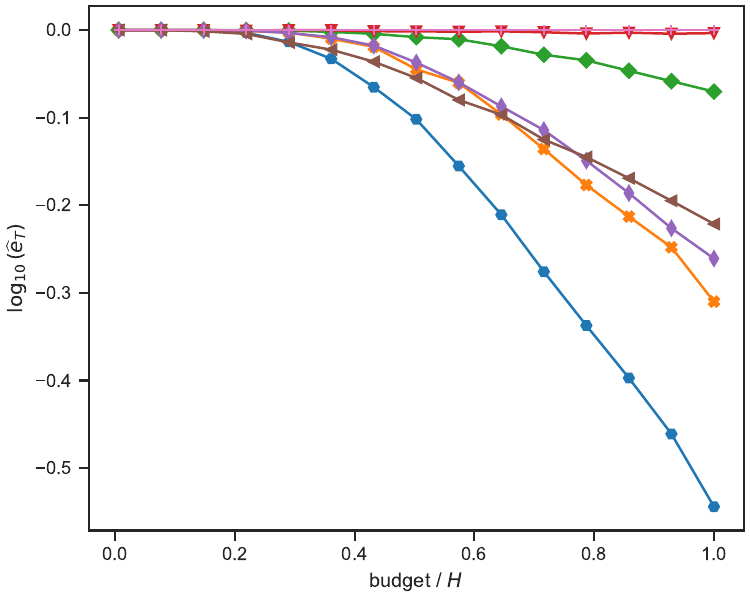}
              \caption{Arms on a convex Pareto set.}
          \end{figure}
      \end{minipage}
      \vfill  
     \begin{minipage}{0.3\linewidth}
          \begin{figure}[H]
        \includegraphics[width=\linewidth]{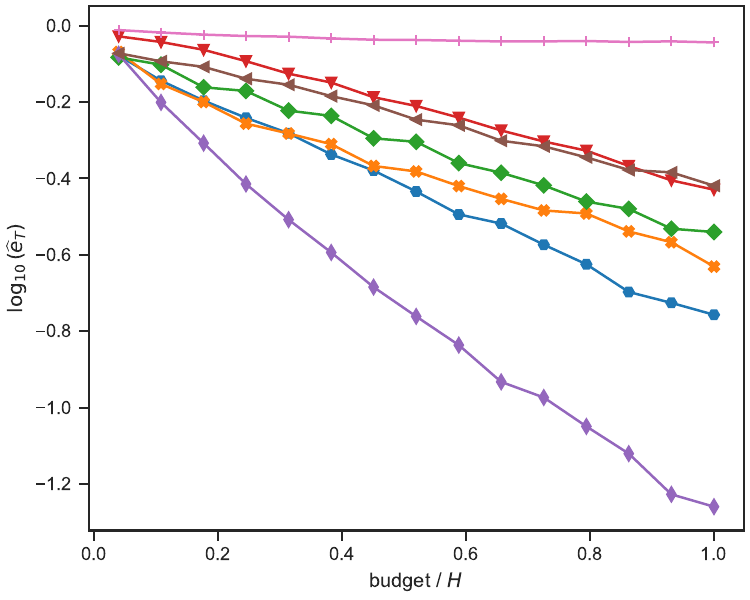}
              \caption{Each sub-optimal $i$ is only dominated by $i^\star$.}
          \end{figure}
      \end{minipage}
 \hspace{0.5cm}
            \begin{minipage}{0.3\linewidth}
          \begin{figure}[H]
              \includegraphics[width=\linewidth]{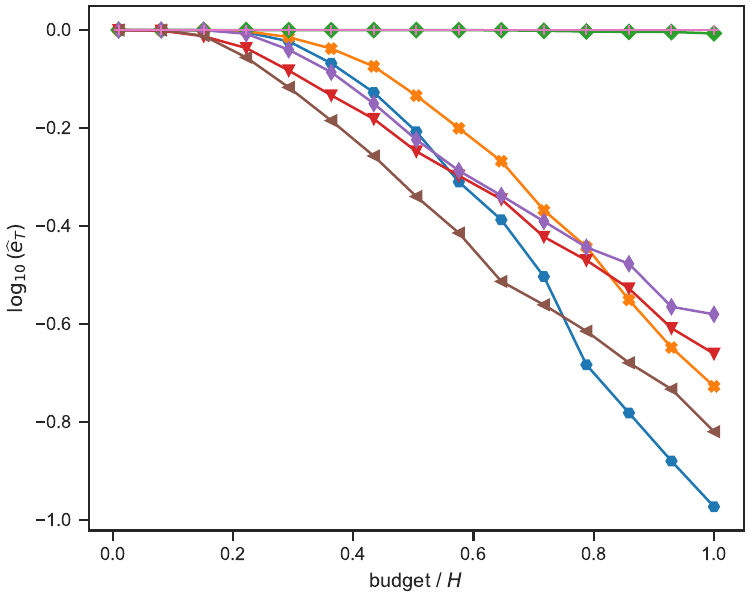}
              \caption{$K=200$ arms on the unit circle.}
          \end{figure}
      \end{minipage}
       \hspace{0.5cm} 
      \begin{minipage}{0.3\linewidth}
          \begin{figure}[H]
              \includegraphics[width=\linewidth]{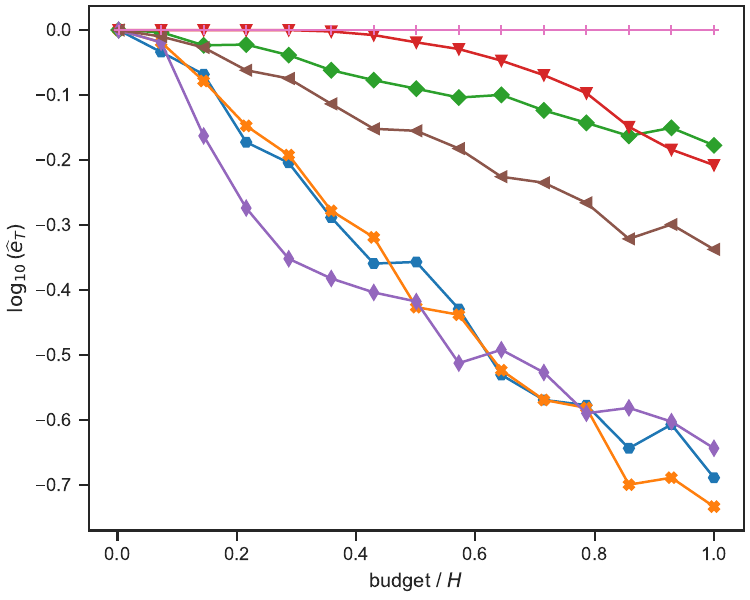}
              \caption{High dimension ($D=10$) with 2 group of arms. }
          \end{figure}
      \end{minipage} 
      \end{minipage}
   \end{figure*}

\subsection{Real-world Datasets} 

\paragraph{COV-BOOST} \!\!\citep{munro_safety_2021} is a phase II vaccine clinical trial conducted on 2883 participants to measure the immunogenicity of different Covid-19 vaccines as a booster (third dose). Combining the first two doses received and the third dose investigated in the trial, there were $K=20$ arms and the authors reported the sample mean response and confidence intervals (based on log-normal assumption of the data) of each arm to a bunch of immunogenicity indicators. \cite{kone2023adaptive} further extracted and processed the responses to 3 indicators (cellular response, anti-spike IgG and NT$_{50}$) to generate a multivariate normal bandit with $K=20, D=3$ and a diagonal covariance matrix. We use their dataset in our simulations with a total budget of $T=2500$. 
% The trial was conducted with 20 arms and many immunogenicity indicators were monitored simultaneously. \cite{kone2023adaptive} extracted and processed the average response of each arm to 3 immunogenicity indicators. The authors have also reported the sample covariance. We use this dataset to simulate a multivariate Gaussian bandit with $K=20, D=3$ and the means and covariance as reported in \cite{kone2023adaptive}. 
 %\begin{figure}[H]
 %   \begin{subfigure}[b]{0.49\linewidth}
 %        \centering
 %       \includegraphics[width=\linewidth]{code/covboost.png}	
 %        \captionof{figure}{COV-BOOST trial}
 %        \label{fig:covBoost}
 %    \end{subfigure}
%\hfil
%    \begin{subfigure}[b]{0.49\linewidth}
%         \centering
%        \includegraphics[width=\linewidth]{code/noc.png}	
%         \captionof{figure}{NoC hardware design}
%         \label{fig:noc}
%     \end{subfigure}
% \caption{Results on real-world tasks}
% \end{figure} 
\paragraph{Hardware Design}
 We use the SNW (Sorting Network Width) dataset \citep{sort_dataset} to generate a bandit model. The dataset is made of 206 different sorting networks (devices used to sort items). The authors reported  area and throughput of each network when synthetized on a FPGA (Field-Programmable Gate Array). The area is the number of FPGA slices (resources units) used during execution and the throughput is the number of samples treated per second. The goal is then to optimize both objectives simultaneously (reduced area and large throughput). The simulation process is costly and the measures may slightly vary for the same network due to randomness in the circuits. We simulate a Gaussian bandit with $K=206, D=2$ with the SNW dataset. We use a total budget of $T=5000$.  
%  for application-specific circuits (ASICs) and the objective is to optimize the energy consumption and the runtime. Each implementation is characterized by a 4-dimensional feature vector and for each of them both energy consumption and runtime are measured. We apply a KMeans clustering in the configurations to define $K=16$ clusters and for each cluster, we compute the average performance of the configurations in the cluster. We compute the sample covariance on the entire dataset. We simulate a multivariate normal bandit using the $K=16$ clusters' means and the sample covariance. 
 \subsection{Synthetic Benchmark}
 We run the algorithms on different synthetic instances. For each of them, we compute the complexity $H(\nu)$ and following \cite{audibert_best_2010, karnin_almost_2013} we set the total budget to $T= H(\nu)$. %We report the average error over independent trials with different seeds. 
% {\bfseries Experiment 1: 2 clusters of arms.}  We set $K=20$ and we choose $(\vmu_1, \dots, \vmu_{10}) \sim \mathcal{U}\left([0.2, 0.4]^2\right)^{\otimes 10}$ and $(\vmu_{11}, \dots, \vmu_{20})\sim \mathcal{U}\left([0.5, 0.7]^2\right)^{\otimes 10}$. 
 %\todoem{Do you choose it once and for all and run the algo 4000 times on this same instance? or 4000 different instances?}
% {\bfseries Experiment 1: 3 clusters of arms.} $K=30$ and we choose $(\vmu_1 \dots \vmu_{20})$ similarly to Instance 2 but $(\vmu_{21}, \dots, \vmu_{30}) \sim \mathcal{U}\left([0.2, 0.4]\times[0.6, 0.8]\right)^{\otimes 10}$.
 %\todoem{Keep only Exp 1 or Exp 2 in the main (maybe Exp 2), to add the results of NoC}
 
{\bfseries Experiment 1: Arms on a convex Pareto set.}
 $K=60, D=2$. We choose $x_1,\dots, x_{20}$ equally spaced in $[0.55, 0.95]$ and $i=1,\dots, 10$, $\vmu_i:=(x_i^2, 1/(4x_i^2))^\T$. $\vmu_{11}, \dots, \vmu_{60}$ are chosen from $\left\{(x, y)\in [0.1, 0.8]^2: xy\leq \frac15\right\}$.
 
{\bfseries Experiment 2: Group of sub-optimal arms where each sub-optimal arm $i$ has a unique $i^\star$.} $K=10, D=2, \lvert \cS^\star\lvert = 2$. For each sub-optimal arm $i$, there is a unique $j$ such that $\vmu_i \prec \vmu_j$. We choose $\vmu_1 := (0.4, 0.75)^\T, \vmu_2 := (0.75, 0.4)^\T$  and for $i=1, \dots, 4$, $\vmu_{2i+1} := (0.45 + 0.2^i, 0.35 - 0.2^i)^\T$, $\vmu_{2i+2}:=(0.10 + 0.20^i, 0.70 - 0.20^i)^\T$.

{\bfseries Experiment 3: Many arms on the unit circle.} $K=200, D=2$ and generate an isotropic multivariate normal instance with $\sigma= 1/4$.  We choose $\beta_1, \dots, \beta_{20}$ evenly spaced in $[\pi/12, \pi/2 - \pi/12]$ and $\beta_{21},\dots, \beta_{200}$ evenly spaced in $[\pi/2+\pi/6, 2\pi - \pi/6]$.  For each arm $i, \vmu_i:= (\cos(\beta_i), \sin(\beta_i))^\T$. 
 
 {\bfseries Experiment 4: High dimension ($D=10$) } We set $K=50$ and generate $(\vmu_1, \dots, \vmu_{30}) \sim \mathcal{U}\left([0.2, 0.45]^{10}\right)^{\otimes 30}$ and $(\vmu_{31}, \dots, \vmu_{50}) \sim \mathcal{U}\left([0.55, 0.75]^{10}\right)^{\otimes 20}$ 
\subsection{Results}
In all experiments, the uniform allocation (UA) baseline is largely outperformed by both EGE-SR and EGE-SH (with no clear ordering between the two algorithms). This is particularly the case when there are many arms and for complex Pareto sets. By estimating the hardness to classify each arm, \ege{} eventually allocates more samples to arms that are difficult to classify, leading to a smaller error probability. %while uniform sampling allocates much more samples to arms that are far away from the Pareto set. 
\apefb{} is a good competitor to our \ege{} algorithms however it is not robust to the hyper-parameter $a$, which requires the knowledge of $H(\nu)$ to be properly tuned. Our proposed heuristic that estimates the complexity online fails dramatically.

In Appendix~\ref{sec:implementation} we discuss the computational complexity of \ege{} and we detail the implementation setup. In Appendix~\ref{sec:more_exp} we report additional experimental results on synthetic datasets. We illustrate the PSI-$k$ relaxation on the SNW dataset, showing that \egesrk{} can efficiently identify a subset of the Pareto set. 
%The results have shown that EGE-SH/SR both outperform the uniform allocation. This is particularly the case where the number of the arms is large. We have noted that the performance of APE-B is 
%highly sensitive the parameter $a$. When $a$ is set to its optimal theoretical value (assuming $H(\nu)$ is known), APE-B is a good competitor to EGE-SR/SH. 
%  \begin{figure}[H]
 %   \begin{subfigure}[b]{0.49\linewidth}
 %        \centering
 %       \includegraphics[width=\linewidth]{code/test3.pdf}	
 %        \captionof{figure}{$\cI_1$ to $\cI_5$}
 %        \label{fig:covBoost}
 %    \end{subfigure}
%\hfil
%    \begin{subfigure}[b]{0.49\linewidth}
%         \centering
%        \includegraphics[width=\linewidth]{code/test3.pdf}	
%         \captionof{figure}{$\cI_6$ to $\cI_{10}$}
%         \label{fig:noc}
%     \end{subfigure}
% \caption{Results on synthetic datasets.}
% \end{figure} 
%\todoe{Raccourcir les légendes des figures (2 lignes max) pour gagner de la place. Mettre CovBoost et NoC en premier.}
%\todoe{Si expés supplémentaire en appendix ou détails expérimentaux, il faut les pointer ici}

\section{CONCLUSION}
We proposed the first algorithms for Pareto set identification in the fixed-budget setting. 

Our generic algorithm \ege{} can be coupled with different allocation and arm elimination schemes. We proved that for two instantiations, EGE-SR and EGE-SH, the probability of error decays exponentially fast with the budget, with an exponent that is un-improvable for some bandit instances. We conducted experiments showing that these algorithms consistently outperform a uniform sampling baseline and are competitive with an oracle algorithm that knows the complexity of the underlying instance. 

Extensions of this work could consider variance-aware estimates of the gaps or a Bayesian setting where prior information on the means could be used to improve performance \citep{atsidakou2023bayesian}. Finally, just like Successive~Rejects or Sequential~Halving, EGE algorithms heavily relies on the knowledge of $T$ for their tuning. In future work, we are interested in proposing anytime algorithms, that can have a small error probability for any budget $T$. We may take inspiration from some anytime algorithms recently proposed for simpler pure exploration tasks \citep{katz-samuels_feasible_2018,jourdan2023epsilonbestarm}.   

%^Following the recent results of \cite{jourdan2023epsilonbestarm} on anytime BAI algorithms, we would like to extend their algorithm to anytime PSI.  
 %Following the current trend in fixed-budget pure exploration \citep{degenne23a, wang2023uniformly}, one could intend to study more precisely the complexity of Pareto set identification in the asymptotic fixed-budget setting.
%Our algorithms's mis-identification errors decays exponentially fast with the budget and the decay rate depends on a hardness term that has been used in the fixed-confidence setting \cite{auer_pareto_2016}.
 %We proved n information theoretic lower-bound showing that for some instances our algorithms are optimal up to constants and $\log(K)$ terms. 

\subsubsection*{Acknowledgements}
Cyrille Kone is funded by an Inria/Inserm PhD grant. Emilie Kaufmann acknowledges the support of the French National Research Agency under the projects BOLD (ANR-19-CE23-0026-04) and FATE (ANR22-CE23-0016-01).
%To preserve the anonymity, please include acknowledgments \emph{only} in the camera-ready papers.
\bibliographystyle{apalike}
\bibliography{main}
\clearpage
\section*{Checklist}

% %%% BEGIN INSTRUCTIONS %%%
%The checklist follows the references. For each question, choose your answer from the three possible options: Yes, No, Not Applicable.  You are encouraged to include a justification to your answer, either by referencing the appropriate section of your paper or providing a brief inline description (1-2 sentences). 
%Please do not modify the questions.  
%\textbf{In your paper, please delete this instructions block and only keep the Checklist section heading above along with the questions/answers below.}
% %%% END INSTRUCTIONS %%%
 \begin{enumerate}
 \item For all models and algorithms presented, check if you include:
 \begin{enumerate}
   \item A clear description of the mathematical setting, assumptions, algorithm, and/or model. [Yes]: Section 2, 3, 4, 5
   \item An analysis of the properties and complexity (time, space, sample size) of any algorithm. [Yes]: Section 3, 4, 5 and supplemental material 
   \item (Optional) Anonymized source code, with specification of all dependencies, including external libraries. [Not Applicable]
 \end{enumerate}

 \item For any theoretical claim, check if you include:
 \begin{enumerate}
   \item Statements of the full set of assumptions of all theoretical results. [Yes]: Section 3, 4, 5 and supplemental material 
   \item Complete proofs of all theoretical results. [Yes]: Supplemental material
   \item Clear explanations of any assumptions. [Yes]     
 \end{enumerate}

 \item For all figures and tables that present empirical results, check if you include:
 \begin{enumerate}
   \item The code, data, and instructions needed to reproduce the main experimental results (either in the supplemental material or as a URL). [Yes]
   \item All the training details (e.g., data splits, hyperparameters, how they were chosen). [Not Applicable]
         \item A clear definition of the specific measure or statistics and error bars (e.g., with respect to the random seed after running experiments multiple times). [Yes]
         \item A description of the computing infrastructure used. (e.g., type of GPUs, internal cluster, or cloud provider). [Not Applicable]
 \end{enumerate}

 \item If you are using existing assets (e.g., code, data, models) or curating/releasing new assets, check if you include:
 \begin{enumerate}
   \item Citations of the creator If your work uses existing assets. [Yes]
   \item The license information of the assets, if applicable. [Not Applicable]
   \item New assets either in the supplemental material or as a URL, if applicable. [Not Applicable]
   \item Information about consent from data providers/curators. [Not Applicable]
   \item Discussion of sensible content if applicable, e.g., personally identifiable information or offensive content. [Not Applicable]
 \end{enumerate}

 \item If you used crowdsourcing or conducted research with human subjects, check if you include:
 \begin{enumerate}
   \item The full text of instructions given to participants and screenshots. [Not Applicable]
   \item Descriptions of potential participant risks, with links to Institutional Review Board (IRB) approvals if applicable. [Not Applicable]
   \item The estimated hourly wage paid to participants and the total amount spent on participant compensation. [Not Applicable]
 \end{enumerate}

 \end{enumerate}

\appendix
\thispagestyle{empty}
% For one-column format, uncomment the following:
\onecolumn
%\aistatstitle{Instructions for Paper Submissions to AISTATS 2024: \\
%Supplementary Materials}
% For two-column format, uncomment the following:
%\twocolumn[ \makesupplementtitle ]
%\vfill 
%
%\clearpage
\section{ANALYSIS OF EMPIRICAL GAP ELIMINATION} 
\label{sec:analysis}
\subsection{Proof of Theorem~\ref{thm:main-res}}
We state below an important lemma used in the proof of Theorem~\ref{thm:main-res}. For each sub-optimal arm $i$, we define $i^\star$ to be an arbitrary element in $\argmax_{k\in \cS^\star}\m(i, k)$, which by definition yields $\Delta_i^\star = \m(i,i^\star)$.  
Introducing the property
\[\cP_{r} := \left\{\forall i \notin\cS^\star, i \in A_r \Rightarrow i^\star \in A_r \right\},\]
the first step of the proof consists in proving that $\cP_r$ holds for all $r$. To do so, we introduce several intermediate results whose proofs are gathered in Appendix~\ref{sec:concentration}.

The first one controls the deviation of the empirical quantities $\Mh(i,j;r)$ and $\mh(i,j;r)$ from their actual values. 
\begin{restatable}{lemma}{lemIneq}
\label{lem:lem-ineq}
	On the event $\cE_c^1$, for all $r\in [R]$ and $i,j \in A_r$,  if  $i\in \cS^\star $ or $j\in \cS^\star$ then 
	\begin{align*}
	&\lvert \Mh(i,j;r) - \M(i,j)\lvert \leq c\Delta_{(\lambda_{r+1}+1)} \; \text{and} \\
	 &\lvert \mh(i,j;r) - \m(i,j)\lvert \leq c\Delta_{(\lambda_{r+1}+1)}.
	\end{align*}
\end{restatable}

The second one builds on it  and uses Lemma~\ref{lem:lem-gap-redef} to prove that the empirical gaps cannot be too far from their true gaps, if $\cP_r$ holds. It is complementary to Lemma~\ref{lem:lem-gap-ineq} stated in the main paper.

\begin{restatable}{lemma}{lemGapBound}
\label{lem:lem-gap-bound}
	Assume that $\cE_1^c$ holds and let $r\in [R]$ such that $\cP_r$ holds.
	 Then for any $i\in A_r$
	 \[ \widehat \Delta_{i, r} -\Delta_i \geq \begin{cases}
	 	-2c\Delta_{(\lambda_{r+1} +1)} & \text{ if } i \in \cS^\star,\\
	 	-c\Delta_{(\lambda_{r+1} +1)} &\text{ else.}
	 \end{cases} \]
	 
%	 $i \notin \cS^\star$,  
%	$$  \widehat \Delta_{i, r} - \Delta_i  \geq -c\Delta_{(\lambda_{r+1} +1)}$$
%	and for $i\in \cS^\star$
%	$$ \Delta_i - \widehat\Delta_{i, r}\geq -2c\Delta_{(\lambda_{r+1} +1)}.$$
\end{restatable}

\begin{restatable}{lemma}{lemDom}
 \label{lem:lem_dom}
	Let $c>0$ and assume $\cE_c^1$ holds. At round $r\in [R]$, for any sub-optimal arm $i\in A_r$, if $i^\star \in A_r$ and $i^\star$ does not empirically dominate $i$ then $\Delta_i^\star \leq c\Delta_{(\lambda_{r+1} +1)}$.
\end{restatable}

We are now ready to prove the following key result.

\lemref*
\begin{proof} We assume that the event $\cE_c^1$ holds and we prove the result by induction on $r$. 

$\cP_r$ trivially holds for $r=1$ as all arms are active. 
%Let $k \in \{1,\dots,K-1\}$ be such that $(i\in  A_k\cap(\cS^\star)^\complement$ and $i^\star \in A_k)$. 
Let $r\geq 1$ such that $\cap_{s=1}^r\cP_{s}$ holds. We shall prove that for any $i\in A_{r+1} \cap (\cS^\star) ^\complement $, $i^\star$ cannot be deactivated at the end of round $r$. 
	A first observation is that the empirical gap $\widehat\Delta_{i,r}$ of each arm $i$ satisfies 
	\begin{equation}
 \widehat{\Delta}_{i,r} = \max\left(\widehat\Delta_{i,r}^\star,\widehat{\delta}_{i,r}^\star\right)\label{eq:trick1}
\end{equation}
which follows from the fact that when $i \in S_r$, $\widehat\Delta_{i,r}^\star < 0 \leq \widehat{\delta}_{i,r}^\star$ and when $i \in A_r\backslash S_r$, $ \widehat{\delta}_{i,r}^\star < 0 \leq \widehat\Delta_{i,r}^\star$. Using Lemma~\ref{lem:lem-gap-bound} and the inductive hypothesis permits to prove that
\begin{equation}
\forall i \in A_r, \ \ \widehat{\Delta}_{i,r} \geq \Delta_i - 2c\Delta_{(\lambda_{r+1}+1)}. 
 \label{eq:trick2}
\end{equation}
%by distinguishing two cases and using \eqref{eq:trick1}: either $i \notin \cS^\star$ and 
%\[\widehat{\Delta}_{i,r} = \widehat\Delta_{i,r}^\star \geq \left(\Delta_i - 2c\Delta_{(\lambda_{r+1}+1)}\right) > \Delta_i -2c\Delta_{(\lambda_{r+1} +1)} \]
%or $i \in \cS^\star$ and 
%\[\widehat{\Delta}_{i,r} =  \widehat\delta_{i,r}^\star \geq \Delta_i - 2c\Delta_{(\lambda_{r+1}+1)}\]
%where in both cases the last inequality follows from the corresponding statement in Lemma~\ref{lem:lem-gap-ineq}.

\medskip

To prove that $\cP_{r+1}$ holds, we proceed by contradiction and assume that there exists a sub-optimal arm $i \in A_{r}$ (and therefore $i^\star \in A_r$ by $\cP_r$) such that $i \in A_{r+1}$ and $i^\star \in  A_{r} \backslash A_{r+1}$.  

A first observation is that as there are $\lambda_{r+1}$ arms in $A_{r+1}$ and $i^\star$ is deactivated at the end of round $r$, there exists an arm $a_r\in A_{r+1} \cup \{ i^\star\}$ such that $\Delta_{a_r} \geq \Delta_{(\lambda_{r+1} +1)}$ and $\widehat\Delta_{i^\star} \geq \widehat\Delta_{a_r}$. We now consider two cases depending on whether $i^\star$ is empirically optimal or empirically sub-optimal.

\paragraph{Case 1: arm $i^\star \notin S_r$ i.e.  $i^\star$ is empirically sub-optimal.} 
%Let $B_k:= \{ (K+1-k), \dots (K)\}$ be the $k$ worst arms. 
We have 
\[ \max_{j\in A_r\backslash \{i^\star\}} \mh(i^\star,j;r) := \widehat{\Delta}_{i^\star,r}^\star = \widehat{\Delta}_{i^\star,r} \geq \widehat{\Delta}_{a_r,r} \]
Using Lemma~\ref{lem:lem-ineq} (on the LHS) and Equation \eqref{eq:trick2} (on the RHS) yields
\begin{eqnarray*}
\label{eq:eq1}
	\max_{j\in A_r\backslash \{i^\star\}} \m(i^\star,j) &\geq& \Delta_{a_r} - 3c\Delta_{(\lambda_{r+1} +1)},\\
	&\geq&(1- 3c) \Delta_{(\lambda_{r+1} +1)}\;,
\end{eqnarray*}
where the last inequality follows since $\Delta_{a_r}\geq \Delta_{(\lambda_{r+1} +1)}$. As $i^\star$ is an optimal arm, the LHS of the previous inequality is negative. So it follows that $0 \geq (1- 3c) \Delta_{(\lambda_{r+1} +1)}$ which yields a contradiction if {$3c < 1$}. 

\paragraph{Case 2: arm $i^\star \in S_r$ i.e. $i^\star$ is empirically optimal.} 
%We recall that $i\in A_{r+1}$. 
We first prove that $i^\star$ does not empirically dominate $i$. 
Indeed, if $i$ were dominated by $i^\star$, we would have 
$i\notin S_r$, so $\widehat \Delta_{i, r} = \widehat \Delta_{i, r}^\star>0$ and since $i^\star \in S_r$, 
\begin{eqnarray}
	\widehat \Delta_{i^\star, r} = \widehat\delta^\star_{i^\star,r} &\leq& (\Mh(i, i^\star; r))^+ + (\widehat\Delta_{i, r}^\star)^+, \\ \label{eq:lemref1}
	&=& 0 + \widehat\Delta_{i, r}.
\end{eqnarray}
 Recalling that $i\in A_{r+1}$ we also have $\widehat \Delta_{i^\star, r}\geq \widehat \Delta_{i, r}$, which combined with Equation~\eqref{eq:lemref1} yields 
 \begin{equation}
 	\label{eq:lemref2}
 	\widehat \Delta_{i^\star, r} = \widehat \Delta_{i, r}. 
 \end{equation}
However, the tie-breaking rule of \ege{} ensures that Equation~\eqref{eq:lemref2} is not possible since $i^\star \in S_r$ and it is deactivated while $i\in A_{r+1} \cap (S_r)^\complement$ (in case of an equality in the gaps, empirically sub-optimal arms are removed). Therefore $i$ is not empirically dominated by $i^\star$. Hence, by Lemma~\ref{lem:lem_dom}
\begin{equation}
	\label{eq:eq3}
	\Delta_{i}^\star \leq c\Delta_{(\lambda_{r+1} +1)}.
\end{equation}

%Using that $i^\star \in S_r$ has a larger  empirical gap than the arms in $A_{r+1}$, recalling that $i \in A_{r+1}$ (by assumption), we have on one side
%\[
%	 \Mh(i, i^\star; r)^+ + (\widehat\Delta_{i, r}^\star)^+ \geq \widehat\delta_{i^\star,r}^\star = \widehat{\Delta}_{i^\star,r}, \] 
%	 and on the other side 
	 
%	 $$ \widehat \Delta_{i^\star, r}\geq \widehat \Delta_{i, k} \geq \widehat \Delta_{i, r}^\star  $$
%	 where the last step uses \eqref{eq:trick1}. Using Lemma~\ref{lem:lem-ineq} and Lemma~\ref{lem:lem-gap-ineq} applied to the sub-optimal arm $i$, it follows that 
%	 \[\M(i, i^\star)^+ +  \Delta_{i}^\star \geq \Delta_i^\star - 6c\Delta_{(\lambda_{r+1} +1)}\;.\]
%	 Therefore, 
%\begin{eqnarray}
%\label{eq:eq3}
%	\Delta_i^\star &\leq& {8c}\Delta_{(\lambda_{r+1}+1)}. 
%\end{eqnarray}
Moreover, since $a_r\in A_{r+1}\cup\{i^\star\}$, using the definition of $\widehat \Delta_{i^\star, r}$  yields 
\[
	 \Mh(i, i^\star; r)^+ + (\widehat\Delta_{i, r}^\star)^+ \geq \widehat{\Delta}_{a_r,r}\]
Using Lemma~\ref{lem:lem-ineq}, Lemma \ref{lem:lem-gap-bound}  applied to $i$ and Equation~\eqref{eq:trick2} applied to $a_r$, we obtain
\begin{eqnarray*}
	\M(i, i^\star)^+ +  \Delta_i^\star &\geq& \Delta_{a_r} -5c\Delta_{(\lambda_{r+1} +1)}\\
	&\geq& (1-5c)\Delta_{(\lambda_{r+1} +1)}%\\
	%&\geq& (1-6c)\Delta_{i}^\star 
\end{eqnarray*}
%% on peut en fait garder 1/5 
so 
\begin{equation}
\label{eq:eq4}
\Delta_i^\star \geq (1-5c)\Delta_{(\lambda_{r+1} +1)}.
\end{equation}
Combining this inequality with Equation~\eqref{eq:eq3} yields a contradiction if {$ 6c < 1$}. Therefore $\cP_{r+1}$ holds and we have proved the claimed result by induction on $r$. 
 \end{proof}
 Before proving Theorem~\ref{thm:main-res} we need  Lemma~\ref{lem:hoeff-good-event} to upper-bound $\bP(\overline{(\cE_c^1)})$ for any $c>0$.
 
 \begin{restatable}{lemma}{HoeffGoodEvent}
 \label{lem:hoeff-good-event}
 	For any $c>0$, 
 	$$ \bP(\overline{(\cE_c^1)}) \leq 2\lvert \cS^\star\lvert(K-1)D R \exp\left( - \frac{c^2 \widetilde T^{R, \bm t, \bm\lambda}(\nu)}{4\sigma^2}\right).$$
 \end{restatable}

 We can now prove Theorem~\ref{thm:main-res}. 
 \mainRes*
 
\begin{proof}
 We show that on $\cE_c^1$, every arm is well classified so the recommended set is the true Pareto optimal set. First by Lemma~\ref{lem:lemref} we know   that on $\cE_c^1$, the property $\left(i\in A_r \cap (\cS^\star)^{\complement} \Rightarrow i^\star \in A_r\right)$ holds for every $r \in [R+1]$. As in the proof of Lemma~\ref{lem:lemref} (see Equation~\eqref{eq:trick2}), this permits to prove in Lemma~\ref{lem:lem-gap-bound} 
 \begin{equation}
\forall r \in [R], \forall i \in A_r, \ \ \widehat{\Delta}_{i,r} \geq \Delta_i - 2c\Delta_{(\lambda_{r+1} +1)}
 \label{eq:re-trick2}
\end{equation}

We now establish that at the end of every round $r\in [R]$ no mis-classification can occur. That is for every arm $i\in A_r\backslash A_{r+1}$:
\begin{enumerate}[a)]
	\item if $i\notin S_r$ (that is,  $i$ is added to $D_r$) then $i\notin \cS\star$,
	\item if $i\in S_r$ (that is, $i$ is added to $B_r$) then $i\in \cS^\star$ .
\end{enumerate} 
 Let $i\in A_r\backslash A_{r+1}$. Since  $\lambda_{r+1} = \lvert A_{r+1} \lvert$ should remain active at the end of round $r$, 
if $i$ is removed at the end of the round, there exists $a_r\in A_{r+1} \cup \{ i\}$ such that $\Delta_{a_r} \geq \Delta_{(\lambda_{r+1} +1)}$ and $a_r$ has a smaller empirical gap than $i$.  
 	\paragraph{Case 1: $i\notin S_r$ i.e. $i$ is empirically sub-optimal.}  	As $a_r \in A_{r+1}\cup\{i\}$ has a smaller empirical gap than $i$, we have 
 	\begin{eqnarray*}
 	\widehat \Delta_{i,r}^\star &\geq& \widehat \Delta_{a_r, r},
  	\end{eqnarray*}
  	Assume by contradiction $i\in \cS^\star$. Using Lemma~\ref{lem:lem-ineq} and Equation~\eqref{eq:re-trick2} yields  
  	\begin{eqnarray*}
  	\label{eq:lemref4}
  		\max_{j\in A_r\backslash \{i\}} \m(i,j) &\geq& \Delta_{a_r} - 3c\Delta_{(\lambda_{r+1} +1)} \geq (1-3c)\Delta_{(\lambda_{r+1} +1)}, 
  	\end{eqnarray*}
  	When $3c<1$, the RHS of Equation~\ref{eq:lemref4} is positive, so this inequality yields 
  	\begin{eqnarray*}
  		\max_{j\in A_r\backslash \{i\}} \m(i,j) >0, 
  	\end{eqnarray*}
  	that is there exists $j\in A_r$ such that 
  	$$ \m(i,j)>0$$
  	so there exists $j$ such that $\vmu_i \prec \vmu_j$, 
  which contradicts the assumption $i\in \cS^\star$. Therefore, $i\notin \cS^\star$: $i$ is a sub-optimal arm. 
   
 	\paragraph{Case 2: $i\in S_r$ i.e. $i$ is empirically optimal.} 
Since $a_r$ has a larger empirical gap than $i$ and $i\in S_r$, we have
\begin{equation}
\label{eq:thm-main-eq0}
\min_{j\in A_r\backslash\{i\} }\M(i,j; r) \geq \widehat{\delta}_i^\star \geq \widehat\Delta_{a_r,r}	
\end{equation}
 Assume by contradiction that $i$ is sub-optimal. By Lemma~\ref{lem:lemref} (for $c<1/6$), $i^\star \in A_r$. 
 Combining with \eqref{eq:thm-main-eq0} yields 
 $$ \M(i, i^\star; r) \geq \widehat\Delta_{a_r,r}.$$
 
 Further using Lemma~\ref{lem:lem-ineq} and Equation~\eqref{eq:re-trick2} yields 
 	\begin{eqnarray*} 
 		%\min_{j\in A_r\backslash\{i\}} 
 		\M(i, i^\star)  \geq (1-3c)\Delta_{(\lambda_{r+1} +1)}.
 	\end{eqnarray*}
 	Then by taking $6c<1$, the RHS of the inequality is positive so 
 	\begin{equation}
 	\label{eq:eq-mainres-0}
 		%\min_{j\in A_r\backslash\{i\}} 
 		\M(i, i^\star) > 0, 
 	\end{equation}
 which is not possible as $\vmu_i \prec \vmu_{i^\star}.$ Therefore, $i\in \cS^\star$: it is an optimal arm. 
 %that is no active arm dominates $i$. Recall that from Lemma~\ref{lem:lemref} is $i\in A_r \cap (\cS^\star)^\complement$ and $\cE^1_c$ holds for $c<1/6$ then $i^\star \in A_r$ and as $\vmu_i\prec \vmu_{i^\star}$,  $\M(i, i^\star)<0$ . So if $i$ were a sub-optimal arm, Equation~\eqref{eq:eq-mainres-0} would not be possible as  long as $\cE^1_c$ holds for $c<1/6$. Thus $i$ is an optimal arm : $i\in \cS^\star$. 
% 	Which (using \autoref{lem:lemref}) implies that $i$ is an optimal arm: $i\in \cS^\star$.  
 	
 	\medskip
 	
 	This proves that on the event $\cE_c^1$, $B_{R} \subseteq \cS^\star$ and $D_{R} \subseteq (\cS^\star)^{\complement}$. Moreover, if there is a remaining active arm in $A_{R+1}$ (which happens if and only if $\lambda_{R+1} = 1$), this arm has to be optimal by Lemma \ref{lem:lemref}. Therefore, on  $\cE_c^1$, the set recommended by \ege{} is $\widehat{S}_T = B_{R} \cup A_{R+1} = \cS^\star$.
 	
As a consequence, for any $0<x<1/6$, letting $c_x = 1/6 -x$ we have 
 $$ e^\text{EGE}_T(\nu) \leq \bP(\overline{(\cE_{c_x}^1)}),$$
 which by Lemma~\ref{lem:hoeff-good-event} yields 
 $$ e^\text{EGE}_T(\nu) \leq 2\lvert \cS^\star\lvert(K-1)D R \exp\left( - \frac{c_x^2 \widetilde T^{R, \bm t, \bm\lambda}}{4\sigma^2}\right).$$ 	
 	We conclude by letting $x\rightarrow0$.
 \end{proof}

\subsection{Proof of Corollary~\ref{thm:sr_sh}} \label{subsec:cor-proof}

\corsr*

\begin{proof} The result follows from Theorem~\ref{thm:main-res} and the expression of the arm schedule and allocation vector. For EGE-SR, we have $R=K-1$, $\lambda_r^{\text{SR}}=K+1-r$ and $t^{\text{SR}}_r = n_r^{\text{SR}} - n_{r-1}^{\text{SR}}$ where  $n_r^{\text{SR}} = \left\lceil  \frac{1}{\overline{\log (K)}} \frac{T-K}{K+1-r}\right\rceil$, which yields 
 	\begin{eqnarray*}
		\widetilde T^{R,\bm t^{\textrm{SR}}, \bm\lambda^{\textrm{{SR}}}}(\nu)  &:=& \min_{r\in [K-1]} n_r^{\text{SR}}\Delta_{(\lambda_{r+1}^{\text{SR}} +1)}^2,\\
		&\geq& \min_{r\in [K-1]} \frac{\Delta_{(K+1-r)}^2}{\overline{\log (K)}} \frac{T-K}{K+1-r},\\
		&=& \frac{T-K}{\overline \log (K)} \frac{1}{\max_{r\in \{2,\dots K\}}r\Delta_{(r)}^{-2}}\\
		&=& \frac{T-K}{\overline \log (K)} \frac{1}{H_2(\nu)}. 
	\end{eqnarray*}
For EGE-SH we have $R= \lceil\log_2(K)\rceil$, $n_r^\textrm{SH} := \sum_{s=1}^r t_s^\textrm{SH} \geq  \frac{T}{\lambda_r^\textrm{SH} \lceil\log_2(K)\rceil}$ and $\lambda_r^\textrm{SH}\leq  2(\lambda_{r+1}^\textrm{SH} +1)$. 
	Then 
	\begin{align*}
		\widetilde T ^{R,\bm t^{\textrm{SH}}, \bm\lambda^{\textrm{SH}}}(\nu)  &:= \min_{r\in \{1, \dots, \lceil\log_2(K)\rceil\}} n_r^\textrm{SH}\Delta_{(\lambda_{r+1}^\textrm{SH} +1)}^2, \\
	&\geq \min_{r\in \{1, \dots, \lceil\log_2(K)\rceil\}} \frac{\Delta_{(\lambda_{r+1}^\textrm{SH}+1)}^2 }{\lambda_{r+1}^\textrm{SH} +1 } \times \frac{T}{2\lceil\log_2(K)\rceil},\\
	&= \frac{T/(2\lceil\log_2(K)\rceil)}{\max_{r\in \lvert\lceil\log_2(K)\rceil\lvert} (\lambda_{r+1}^{\textrm{SH}}+1) \Delta_{(\lambda_{r+1}^{\textrm{SH}}+1)}^{-2}}\\
	&\geq \frac{T}{2H_2(\nu) \lceil\log_2(K)\rceil}. 
	\end{align*}
\end{proof}

\subsection{Analysis of \egesrk{}} \label{subsec:relax}
In this section, we analyze \egesrk{} and bound its stopping time. 

We define for any $c>0$ and  $t\in [K-1]$: 
$$ \cE_c^2(t) := \bigcap_{r \in [t]}\bigcap_{i\in \cS^\star} \bigcap_{j \in [K]}\;\left\{\left \|(\vmuh_{i, n_r} - \vmuh_{j, n_r}) - (\vmu_i -\vmu_j)  \right\|_{\infty}
  \leq c  \Delta_{(K+1-r)}^{(k)} \right\}$$
and in particular we abuse notation and define $$ \cE_c^2 := \cE_c^2(K-1).$$ 
We define $\tilde k = \min(k, \lvert \cS^\star\lvert)$. 

For any round $r$, let $\alpha(r) = \lvert B_r \lvert$ denote the number of arms so far identified as optimal at the beginning of round $r$. We denote these arms by $a_1, \dots, a_{\alpha(r)}$. We say that the algorithm has made an error before round $r$ if $B_r \cap (\cS^\star)^\complement \neq \emptyset$ or $D_r \cap \cS^\star \neq \emptyset$.  
We will show that on $\cE_c^2(t)$, the algorithm does not make any error until the end of round $\tau_k^t$ such that $\alpha(\tau_k^t+1) = \tilde k$.  

\begin{restatable}{lemma}{fundAMP}
\label{lem:fund-amp}
Let $c<1/6, t\in [R]$ and $k\leq \lvert \cS^\star\lvert$. Let $\tau_k^t := \min\{r \in [t]: \alpha(r+1) = \tilde k\}\land t$. 
On the event $\cE_c^2(t)$, the algorithm makes no error until the end round $\tau_k^t$ and for any $r\leq \tau_k^t + 1$, if $j\in A_r$ is a sub-optimal arm then $j^\star \in A_r$.  In particular, on the event $\cE_c^2(t)$, $\{ a_1, \dots, a_{\alpha(\tau_k^t +1)}\} \subset \cS^\star$. 
\end{restatable}

Said otherwise, this lemma states that the first arms that will be declared as optimal by \egesrk{} will be actually optimal if $\cE_c^2(t)$ holds for $c$ small enough. 

\begin{proof}[Proof of Lemma~\ref{lem:fund-amp}]
In the sequel we assume that $\cE_c^2(t)$ holds. We prove the correctness by induction on the round $r$. 
Let $r\in [t]$ and let $B_r = \left\{a_1, \dots, a_{\alpha(r)}\right\}$ denote the arms so far identified as optimal. 
Let $\cH_r$ be the property ``for any sub-optimal arm $i\in A_r$, $i^\star \in A_r$ and no error occurred so far". 

$\cH_r$ trivially holds for $r=1$ and $\alpha(1) = 0$. We now assume that it holds until the beginning of round $r$ and that $\alpha(r)<\tilde k$. We will show that the arm $i_r$ de-activated at the end of round $r$ is well classified and that for any sub-optimal arm $j \in A_{r+1}$, $j^\star \in A_{r+1}$.

\begin{lemma}
\label{lem:induc-exist}
	If $\cH_r$  holds at round $r$ and $\alpha(r)<\tilde k$, there exists $a\in A_r$ such that $\Delta_a^{(k)}= \Delta_a $ and $\Delta_a^{(k)}\geq \Delta_{(K+1-r)}^{(k)}.$
\end{lemma}

\begin{proof}[Proof of Lemma~\ref{lem:induc-exist}]
At the beginning of round $r$, it remains $K-r+1$ active arms so there exists $a \in A_r$ such that $\Delta_a^{(k)} \geq \Delta_{(K+1-r)}^{(k)}$. If $a$ is sub-optimal then $\Delta_a^{(k)} = \Delta_a$. Otherwise, if $a$ is an optimal arm, since $\alpha(r)<\tilde k$ and no error has occurred so for (by assumption) then there exists one of the optimal arms $a'\in A_r$ (one of those with the $\tilde k$ largest gaps) such that $\Delta_{a'}^{(k)} = \Delta_{a'}$ and $\Delta_{(K+1-r)}^{(k)} \leq \Delta_a^{(k)} \leq \Delta_{a'}^{(k)}$. 
%$\tilde k$ best arm $a' \in A_r$ such that $\Delta_{(K+1-r)}^{(k)} \leq \Delta_a^{(k)} \leq \Delta_{a'}^{(k)} = \Delta_{a'} $. 
%Which proves the lemma. 
\end{proof}

Lemma \ref{lem:lem-ineq} still holds for the event $\cE_c^2(t)$ with the modified gaps introduced earlier. We state the following lemma which is similar to Lemma~\ref{lem:lem-gap-ineq}. 
\begin{restatable}{lemma}{lemGapIneqAmp} 
\label{lem:lem-gap-ineq-amp} Assume that the event $\cE_c^2(t)$ holds. Let $r\in [t]$ and assume that for \emph{any} sub-optimal arm $j \in A_r$, $j^\star \in A_r$. Then, for any sub-optimal arm $i\in A_r$,
	$$\lvert \widehat \Delta_{i, r}^\star - \Delta_i^\star \lvert \leq 2c \Delta_{({K+1-r})}^{(k)} $$
	and for any optimal arm $i \in A_r$, $$ \widehat \delta_{i,r}^\star \geq \Delta_i -2c \Delta_{({K+1-r})}^{(k)}.$$
\end{restatable}

As already noted in the proof of Lemma~\ref{lem:lemref} and Lemma~\ref{lem:lem-gap-bound}, it is simple to see that the empirical gap $\widehat\Delta_{i,r}$ of each arm $i$ satisfies 
	\begin{equation}
 \widehat{\Delta}_{i,r} = \max\left(\widehat\Delta_{i,r}^\star,\widehat{\delta}_{i,r}^\star\right)\label{eq:trick1-amp}.
\end{equation}
It follows from the fact that when $i \in S_r$, $\widehat\Delta_{i,r}^\star < 0 \leq \widehat{\delta}_{i,r}^\star$ and when $i \in A_r\backslash S_r$, $ \widehat{\delta}_{i,r}^\star < 0 \leq \widehat\Delta_{i,r}^\star$. 

Using Lemma~\ref{lem:lem-gap-ineq-amp} and the inductive hypothesis $\cH_r$ permits to prove that
\begin{equation}
\forall i \in A_r, \ \ \widehat{\Delta}_{i,r} \geq \Delta_i - 2c\Delta_{({K+1-k})}^{(k)}
 \label{eq:trick2-amp}
\end{equation}
We split the proof in three steps to prove that $\cH_{r+1}$ holds.

\paragraph{Step 1:} $i_r$ is well classified. 
Let $i_r$ be empirically sub-optimal ($i_r\notin S_r$) and assume $i_r \in \cS^\star$. 
%If $i_r \in S_r$, 
Since $i_r$ is removed at the end of round $r$, by the inductive assumption and using Equation~\ref{eq:trick2-amp} and Lemma~\ref{lem:lem-ineq} (adapted) we have 
\begin{equation}
\label{eq:amp-x1}
\max_{j\in A_r\backslash\{i_r\}}\mh(i_r,j)\geq \Delta_i - 3c\Delta_{(K+1-r)}^{(k)},\; \forall i \in A_r.	
\end{equation}
By Lemma \ref{lem:induc-exist}, there exists $a\in A_r$ such that $\Delta_a = \Delta_a^{(k)}\geq \Delta_{(K+1-r)}^{(k)}$. Applying \eqref{eq:amp-x1} to this arm $a$ yields 
\begin{equation}
\label{eq:amp-x2}
	\max_{j\in A_r\backslash\{i_r\}}\mh(i_r,j)\geq (1-3c) \Delta_{(K+1-r)}^{(k)}. 
\end{equation}
Recalling that $c<1/6$ we see that the LHS of \eqref{eq:amp-x2} is positive, so there exists $j$ such that $$ \vmu_{i_r} \prec \vmu_j$$
which contradicts the assumption $i_r\in \cS^\star$, so $i_r\in [K]\backslash\cS^\star$. 
%so $i_r\notin \cS^\star$. 
Now if $i_r$ is empirically optimal, i.e $i_r \in S_r$ assume $i_r$ is a sub-optimal arm that is $i_r\notin \cS^\star$. By the hypothesis
$\cH_r$, $i_r^\star \in A_r$ and 
by definition of $\widehat\Delta_{i_r, r}$ and since $i_r$ is removed, we have (Lemma \ref{lem:lem-ineq} and Equation~\ref{eq:trick2-amp})
\begin{equation}
\label{eq:amp-x3}
	%\min_{j\in A_r\backslash\{i_r\}}
	\M(i_r, i_r^\star) \geq \Delta_i -3c \Delta_{(K+1-r)}^{(k)}\quad  \forall i\in A_r.
\end{equation}
Using Lemma \ref{lem:induc-exist} as before, there exists $a\in A_r$ such that $\Delta_{(K+1-r)}^{(k)} \leq \Delta_a^{(k)} = \Delta_a$. 
So the LHS of \eqref{eq:amp-x3} is positive for $c<1/6$. That is %no arm in $A_r$ actually dominates $i_r$: 
%$$ \forall i\in A_r, \vmu_{i_r} \nprec \vmu_i. $$ 
$$ \M(i_r, i_r^\star) >0, $$
which is impossible as $\vmu_{i_r} \prec \vmu_{i_r^\star}$.   
This concludes the proof of this part: $i_r$ is well classified. 
%We recall that by induction, if a sub-optimal $j\in A_r$, then $j^\star \in A_r$. Thus, $i_r$ is not a sub-optimal arm. 

\paragraph{Step 2:} If $i_r\in \cS^\star$ is an optimal arm, then no active arm is ``optimally" dominated by $i_r$: 
$$ \forall i \in A_{r+1}\cap (\cS^\star)^\complement, i^\star \neq i_r$$ 
To prove this, assume $i_r\in \cS^\star$. Note that we can deduce from ``Step 1" that $i_r \in S_r$. 
Then by contradiction, let $i\in A_{r+1}\cap (\cS^\star)^\complement$ such that $i^\star = i_r$. We claim that $i$ is not empirically dominated by $i_r$. 
Indeed, if $i$ were dominated by $i_r$, we would have $i\notin S_r$ (i.e empirically sub-optimal) so 
$$ \widehat \Delta_{i, r} = \widehat \Delta_{i, r}^\star >0$$ and since $i_r\in S_r$, by defintion of $\widehat\delta_{i_r,r}^\star$ \begin{eqnarray*}
	\widehat\Delta_{i_r, r} &\leq& (\Mh(i, i_r; r))^+ +(\widehat \Delta_{i}^\star)\\
	&=& 0 + \widehat\Delta_{i, r},
\end{eqnarray*}
and further noting that $\widehat\Delta_{i_r, r}\geq\widehat\Delta_{i, r}$, we would have $\widehat\Delta_{i_r, r} = \widehat\Delta_{i, r}$. However, our tie-breaking ensures that this inequality is impossible: since $i_r \in S_r, i\notin S_r$ and $i_r$ is deactivated we have $\widehat\Delta_{i_r, r} > \widehat\Delta_{i, r}$. 
Therefore $i$ is not empirically dominated by $i_r$. Moreover, since $i_r=i^\star$ (by assumption), Lemma~\ref{lem:lem_dom} (it trivially holds with the modified gaps) yields 
\begin{equation}
\label{eq:amp-x7}
	\Delta^\star_i \leq c\Delta_{(K+1-r)}^{(k)}.
\end{equation}

Recalling that $i_r$ is deactivated, we have for any arm $j\in A_r$, 
\begin{equation}
\label{eq:amp-x5}
	(\Mh(i, i_r; r))^+  + (\widehat \Delta_{i,r}^\star)^+ \geq \widehat\Delta_{j, r}
\end{equation}
On the other side, there exists $a\in A_r$ such that $\Delta_{(K+1-r)}^{(k)}\leq\Delta_{a}^{(k)} = \Delta_a$ (Lemma \ref{lem:induc-exist}). Applying Equation~\eqref{eq:trick2-amp} %Lemma~\ref{lem:lem-gap-ineq-amp} 
and Lemma~\ref{lem:lem-ineq} and taking $j=a$ in Equation~\ref{eq:amp-x5} yields
$$ (\M(i, i_r))^+  + (\Delta_i^\star)^+ \geq (1-5c)\Delta_{(K+1-r)}^{(k)}$$
then as $i_r = i^\star$, so $i\prec i_r$, the latter yields 
\begin{equation}
\label{eq:amp-x6}
	\Delta_i^\star \geq (1-5c)\Delta_{(K+1-r)}^{(k)}.
\end{equation}
Both \eqref{eq:amp-x7} and \eqref{eq:amp-x6} cannot hold when $c<1/6$.  So for any $i\in A_{r+1}$ $$ i^\star \neq i_r.$$ 
Which achieves the proof of "Step 2''. 
Put together, "Step 1\&2'' prove that $\cH_{r+1}$ holds. 
\medskip

\paragraph{Step 3:} Conclusion

We have proved that if $\cH_r$ holds and $\alpha(r) < \tilde k$ then 
$\cH_{r+1}$ holds. Note that if $i_r\in S_r$ then $\alpha(r+1) = \alpha(r) + 1$ otherwise $\alpha(r+1)=\alpha(r)$. 
Therefore, no error occurs until the end of round $\min(r',t)$  where $r'$ is such that $\alpha(r'+1) = \tilde k$. As a consequence $B_{\tau_k^t +1} = \{a_1, \dots , a_{\alpha(\tau_k^t +1)}\} \subset \cS^\star$. And by the proof of "Step 2'', if $j\in A_{\tau_p^t}$ is a sub-optimal arm, then $j^\star \in A_{\tau_p^t + 1}$.  
\end{proof}

We can now prove the main theorem of this section. 
\mainAMP* 
 \begin{proof}
 	The proof is a direct consequence of Lemma~\ref{lem:fund-amp} and Hoeffding's inequality. Indeed, Lemma~\ref{lem:fund-amp} yields that \egesrk{} is correct on the event $\cE_c^2$ for $c<1/6$. Therefore, for any $0<c<1/6$
 	$$ e_{T, k}(\nu) \leq \bP(\overline{(\cE_c^2)}).$$ 
 Letting $0<c<1/6$ fixed, by union bound and  Hoeffing's inequality (as in the proof of Lemma~\ref{lem:hoeff-good-event}) it follows that 
 $$ \bP(\overline{(\cE_c^2)}) \leq 2(K-1)^2\lvert \cS^\star \lvert D\exp\left(- c^2\frac{T-K}{4\sigma^2 H_2^{(k)} \overline{\log}(K)(\nu)}\right)$$
therefore,
 $$e_{T, k}(\nu) \leq 2(K-1)^2\lvert \cS^\star \lvert D\exp\left(- c^2\frac{T-K}{4\sigma^2  H_2^{(k)}(\nu)\overline{\log}(K)}\right)$$ and taking the limit to $1/6$ (as it holds for any $0<c<1/6$) yields the expected result. 
 \end{proof}
 
\medskip 

The theorem below bounds the expected stopping time and the number of samples used at stopping. 
\begin{restatable}{theorem}{earlyStopping}
\label{thm:early-stopping}
Fix $k< \lvert \cS^\star \lvert$ and let $q:=K-\lvert \cS^\star \lvert + k$. Then
	\begin{eqnarray*}
		\bE[\tau] &\leq& q + 2(K-1)\lvert \cS^\star\lvert (K-q-1) q D \exp\left( - \frac{T-K}{144\sigma^2 H_2^{(k)}(\nu) \overline\log(K)}\right) \text{ and }\\
		\bE[N_\tau] &\leq& N_q 	+ 2(K-1)\lvert \cS^\star\lvert (K-q-1) q DT\exp\left(- \frac{T-K}{144\sigma^2 H_2^{(k)}(\nu) \overline\log(K)}\right).
	\end{eqnarray*}
\end{restatable}

This result suggests that for this relaxed problem, the algorithm might not need to use the whole budget, in particular when $T$ is large. For example, consider a setting $[K]=\cS^\star$ then $q=k$ and we roughly use $N_k$ samples which can be way smaller than $N_{K-1}$. 
%Furthermore,  \autoref{thm:main-amp} still holds with the early-stopping. 

%\earlyStopping* 

\begin{proof}
The idea is to show that if the algorithm has not stopped after round $q$ then some high probability event must not hold. 
Let $c<1/6$  be fixed. %and define for any $r\in [K-1]$ $$\cE_c^2(r) := \left\{ \forall i \in \bA,\;  d \in \{1, \dots D\}, \;  s \in \{1, \dots, r\},\; \left\lvert n_s^{-1}\sum_{t=1}^{n_s}X_{i, t}^d - \mu_i^d\right\lvert \leq c \widetilde \Delta_{(K+1-s)} \right\}$$
We have 
\begin{eqnarray*}
	\bE[\tau] &\leq& q + \bE[\tau \ind\{\tau > q\}],\\
	&\leq& q + \sum_{s=q+1}^{K-1} \bP(\tau\geq s).
\end{eqnarray*}

We claim that for any $s>q$, $$\{\tau\geq s\} \subset \overline{(\cE_c^2(q))}.$$ Indeed,

\begin{eqnarray*}
	\tau \geq s &\impl& \tau > q \\
	&\impl& \alpha(q+1) < k,
%	&\impl& \alpha(\lvert (\cS^\star)^\complement\lvert + k +1) < k, 
\end{eqnarray*} 
However, by Lemma~\ref{lem:fund-amp}, if $\cE_c^2(q)$ holds and $\alpha(q+1) <k$ then 
no error has occured until the end of round $q$, therefore  $D_q = (\cS^\star)^\complement$ and $\lvert B_q\lvert = k$, which is not possible as $\alpha(q+1) < k$. So $\tau \geq s \impl (\cE_c^2(q)$ does not hold). 
Then
\begin{eqnarray}
\bE[\tau] &\leq& q + \sum_{s=q+1}^{K-1} \bP(\overline{(\cE_c(q))}),\\
&\leq& q + \sum_{s=q+1}^{K-1}	2(K-1)\lvert \cS^\star\lvert q D \exp\left( - c^2\frac{T-K}{4\sigma^2  H_2^{(k)}(\nu)\overline\log(K)}\right), \\
&\leq& q + 2(K-1)\lvert \cS^\star\lvert (K-q-1) q D \exp\left( - c^2\frac{T-K}{4\sigma^2 H_2^{(k)}(\nu) \overline\log(K)}\right).\label{eq:eq-sto-tau}
%&=& q +  K(K+q)(\lvert \cS^\star\lvert - p - 1) D\exp\left( - 2c^2\frac{T-K}{H_2\overline\log(K)}\right). 
\end{eqnarray}
Similarly, we have 
$$ \bE[N_\tau] \leq N_q + (N_\tau \ind_{\tau>q}).$$ 
Then, 
\begin{eqnarray*}
	\bE[N_\tau] &\leq& N_q + \sum_{s=q+1}^{K-1} N_s \bP(\tau \geq s) \\
	&\leq& N_q + \sum_{s=q+1}^{K-1} N_s \bP(\overline{(\cE_c^2(q))})\\
	&\leq& N_q + 2(K-1)\lvert \cS^\star\lvert q D\exp\left(- c^2\frac{T-K}{4\sigma^2 H_2^{(k)}(\nu) \overline\log(K)}\right)\sum_{s=q+1}^{K-1} N_s.
\end{eqnarray*}
Simple algebra yields for any $r\in [K-1],$
\begin{eqnarray*}
N_r &=& (K-r)n_r + \sum_{s=1}^{r} n_s,\\
&\leq& 	\frac{T-K}{\overline\log(K)} + (K-r) + r + \frac{T -K}{\overline\log(K)}\sum_{s=1}^{r}\frac{1}{K+1-s}\\
&=& \frac{T-K}{\overline\log(K)}\left( 1 + \sum_{s=K+1-r}^{K} s^{-1}\right) + K \\&\leq& T . 
\end{eqnarray*}
Therefore, 
\begin{eqnarray}
\label{eq:eq-sto-N}
\bE[N_\tau] &\leq& N_q 	+ 2(K-1)\lvert \cS^\star\lvert (K-q-1) q DT\exp\left(- c^2\frac{T-K}{4\sigma^2 H_2^{(k)}(\nu) \overline\log(K)}\right). 
\end{eqnarray}
As \eqref{eq:eq-sto-tau} and \eqref{eq:eq-sto-N} holds for any $c<1/6$ taking the limit for a sequence $c_x\rightarrow 1/6$ yields the expected constants in the exponent, which achieves the proof. 
\end{proof}

\clearpage 
 \section{CONCENTRATION RESULTS} \label{sec:concentration}

 In this section we prove the concentration results used in Appendix~\ref{sec:analysis}. 
 We recall the definition of the good event: 
\begin{eqnarray*}
\cE_c^1 &:=& \bigcap_{r\in [R]}\bigcap_{i \in \cS^\star}\bigcap_{j \in [K]}   \left\{\left \|(\vmuh_{i, n_r} - \vmuh_{j, n_r}) - (\vmu_i -\vmu_j)  \right\|_{\infty} \leq c \Delta_{(\lambda_{r+1} +1)}\right\}\;.
\end{eqnarray*}

\subsection{Proof of Lemma~\ref{lem:lem-ineq}}

\lemIneq* 

 \begin{proof}
 Assume $i\in\cS^\star$ or $j\in \cS^\star$. 
	We have 
	\begin{eqnarray*}
		\lvert \Mh(i,j;r) - \M(i,j) \lvert  &=&\left \lvert \max_{d}\left[\muh_{i, n_r}^d - \muh_{j,n_r}^d \right] - \max_d \left[\mu_i^d - \mu_j^d\right]\right\lvert, \\
		&\overset{(a)}{\leq}& \max_d \left\lvert (\muh_{i, n_r}^d - \muh_{j,n_r}^d) - (\mu_i^d - \mu_j^d)
		 \right\lvert, \\
		 &=& \left\| (\vmuh_{i, n_r} - \vmuh_{j,n_r}) - (\vmu_i - \vmu_j)\right\|_\infty, \\
		&\overset{(b)}{\leq}& c\Delta_{(\lambda_{r+1} +1)}.
	\end{eqnarray*}
where $(a)$ follows by reverse triangle inequality and $(b)$ holds on the event $\cE_c^1$. The second part of the lemma follows from 
\begin{eqnarray*}
	\lvert \mh(i,j;r) - \m(i,j) \lvert &=& \lvert -\Mh(i,j;r) + \M(i,j)\lvert, 
\end{eqnarray*}
as $\M(i,j) = -\m(i,j)$ and $\Mh(i,j;r)= -\mh(i,j;r)$. 
\end{proof}

\subsection{Proof of Lemma~\ref{lem:lem-gap-bound}}

We recall the definition of the property
\[\cP_{r} = \left\{\forall i \notin\cS^\star, i \in A_r \Rightarrow i^\star \in A_r \right\}.\]
and start by proving a first intermediate result. 

\begin{lemma}
\label{lem:lem-delta-star}
Assume that $\cE_1^c$ holds and let $r\in [R]$ such that $\cP_r$ holds. Then for any $i\in A_r$,  
$$(\widehat{\Delta}_{i,r}^\star)^+ - (\Delta_i^\star)^+ \leq 2c\Delta_{(\lambda_{r+1}+1)} \quad\text{ and } \quad (\widehat{\Delta}_{i,r}^\star)^+ - (\Delta_i^\star)^+ \geq -c\Delta_{(\lambda_{r+1}+1)}.$$ 
\end{lemma}

\begin{proof}
We define the $\max$ of an empty set to be $-\infty$. 
We first analyze the case $i\in \cS^\star$. When $i\in \cS^\star$, we have 
$$  (\Delta_i^\star)^+ = 0 = \left(\max_{j\in A_r\backslash\{i\}} \m(i, j)\right)^+,$$
which yields 
\begin{eqnarray*}
\lvert (\widehat{\Delta}_{i,r}^\star)^+ - (\Delta_i^\star)^+ \lvert &=& \left\lvert \left(\max_{j\in A_r \backslash\{i\}} \mh(i, j; r)\right)^+ -  \left(\max_{j\in A_r \backslash\{i\}} \mh(i, j)\right)^+\right\lvert, \\
&\leq& \left\lvert \left(\max_{j\in A_r \backslash\{i\}} \mh(i, j; r)\right) -  \left(\max_{j\in A_r \backslash\{i\}} \mh(i, j)\right)\right\lvert \quad \text{(since $\lvert x^+ - y^+\lvert \leq \lvert x-y\lvert)$}, \\
&\leq& \max_{j\in A_r\backslash\{ i\}} \left\lvert \mh(i, j; r) -  \mh(i, j)\right\lvert, \\
&\leq& c\Delta_{(\lambda_{r+1}+1)}, 
\end{eqnarray*} 
where the last inequality follows from Lemma~\ref{lem:lem-ineq}. 
We now assume that $i$ is a sub-optimal arm. Since $\cP_r$ holds at round $r$,  $i^\star \in A_r$ and 
$$\Delta_{i}^\star = \max_{j\in A_r\backslash\{i\}} \m(i,j).$$
Let $\hat i \in \argmax_{j\in A_r\backslash\{i\}} \mh(i,j; r)$ then 
\begin{eqnarray}
	(\widehat{\Delta}_{i,r}^\star)^+ - (\Delta_i^\star)^+ &=& (\mh(i,\hat i; r))^+ - (\m(i, i^\star))^+,\\
	&\geq& (\mh(i,i^\star; r))^+ - (\m(i,i^\star))^+ \quad \text{(since $i^\star \in A_r$)}. \label{eq:delta-star-1}
%	&\geq& -c\Delta_{(\lambda_{r+1}+1)}, 
\end{eqnarray}
We further note that 
\begin{eqnarray*}
	\left\lvert (\mh(i,i^\star; r))^+ - (\m(i,i^\star))^+ \right\lvert &\leq& \left\lvert \mh(i,i^\star; r) -  \m(i,i^\star)\right\lvert,\\
	&\leq& c\Delta_{(\lambda_{r+1}+1)}, \end{eqnarray*}
which follows from Lemma~\ref{lem:lem-ineq}. Combining the latter inequality with \eqref{eq:delta-star-1} yields 
\begin{equation}
\label{eq:delta-star-2}
(\widehat{\Delta}_{i,r}^\star)^+ - (\Delta_i^\star)^+ \geq -c\Delta_{(\lambda_{r+1}+1)}.	
\end{equation}
We also have 
\begin{eqnarray*}
	(\widehat{\Delta}_{i,r}^\star)^+ - (\Delta_i^\star)^+ &\leq&  (\mh(i,\hat i; r))^+ - (\m(i, \hat i))^+,\\
	&\leq& \left\lvert \mh(i,\hat i; r) - \m(i, \hat i)\right\lvert,\\
	&=&  \left\lvert \M(i, \hat i) - \Mh(i,\hat i; r)\right\lvert, \\
	&\leq& \left\| (\vmu_i - \vmu_{\hat i}) - (\vmuh_{i, n_r} - \vmuh_{\hat i, n_r})\right\|_\infty, \\
	&=& \left\| \left(\left(\vmu_i - \vmu_{i^\star}\right) - \left(\vmuh_{i, n_r} - \vmuh_{i^\star, n_r}\right)\right) + \left( \left(\vmu_{i^\star} - \vmu_{\hat i}\right) - \left(\vmuh_{i^\star, n_r} - \vmuh_{\hat i, n_r}\right) \right) \right\|_\infty,\\
	&\leq& 2c\Delta_{(\lambda_{r+1}+1)},
\end{eqnarray*}
where the last inequality follows from the triangle inequality and Lemma~\ref{lem:lem-ineq}. This proves the lemma. 
\end{proof}

We then prove Lemma~\ref{lem:lem-gap-ineq} which can be viewed as a ``symmetric'' version of Lemma~\ref{lem:lem-gap-bound}. 

\lemGapIneq*

 \begin{proof}
Let $i\in A_r \cap (\cS^\star)^\complement$. {By assumption, $i^\star \in A_r$}, so 
$$  \Delta_i := \Delta^\star_i = \max_{j\in A_r \backslash\{i\}} \m(i,j), $$
then we have 
\begin{eqnarray*}
	\lvert \widehat \Delta_{i, r}^\star - \Delta_i\lvert &=& \left\lvert \left(\max_{j\in A_r \backslash\{i\}} \mh(i, j; r)\right) -  \left(\max_{j\in A_r \backslash\{i\}} \mh(i, j)\right)\right\lvert,\\
	&\overset{(a)}{\leq}& \max_{j\in A_r \backslash\{i\}} \lvert \mh(i, j; r) - \m(i, j)\lvert, \\
	&\leq& \max_{j\in A_r \backslash\{i\}}  \lvert \M(i, j) - \Mh(i,j;r)\lvert\\, \\
	&\leq& \max_{j\in A_r \backslash\{i\}} \left\|(\vmu_i - \vmu_j) - (\vmuh_{i, n_r} - \vmuh_{j,n_r}) \right\|_\infty\\
	&=& \max_{j\in A_r \backslash\{i\}} \left\| \left(\left(\vmu_i - \vmu_{i^\star}\right) - \left(\vmuh_{i, n_r} - \vmuh_{i^\star, n_r}\right)\right) + \left( \left(\vmu_{i^\star} - \vmu_{j}\right) - \left(\vmuh_{i^\star, n_r} - \vmuh_{j, n_r}\right) \right) \right\|_\infty, \\
	&\overset{(b)}{\leq}& 2c\Delta_{(\lambda_{r+1}+1)},
\end{eqnarray*}
where $(a)$ follows by reverse triangle inequality and $(b)$ follows by triangle inequality and Lemma~\ref{lem:lem-ineq}. So we have proved the first statement of the lemma. 

Before proving the second statement of the lemma, recall that $$\widehat \delta_{i,r}^\star=\min_{j \in A_r \backslash\{i\}} [\Mh(i,j;r) \;\land\; (\Mh(j,i;r)^+ + (\widehat \Delta_{j, r}^\star)^+)].$$ Let $i\in A_r \cap \cS^\star$. For any $j\in A_r$ we have 
\begin{eqnarray*}
	\lvert \Mh(j,i;r)^+ - \M(j,i)^+\lvert &\leq& \lvert \Mh(j,i;r) - \Mh(j,i)\lvert, \\
	&\leq& c\Delta_{(\lambda_{r+1} +1)}\quad \text{(Lemma }\ref{lem:lem-ineq}),
\end{eqnarray*}
and similarly 
$$ \lvert \Mh(i,j;r) - \M(i,j)\lvert \leq c\Delta_{(\lambda_{r+1} +1)}$$ 
which put together yields 
\begin{eqnarray}
\label{eq:eq-lem-ineq-gap-a}
\Mh(j,i;r)^+ &\geq& \M(j,i)^+ -c\Delta_{(\lambda_{r+1} +1)} \quad \text{and}\\
\label{eq:eq-lem-ineq-gap-b}
	\Mh(i,j;r)&\geq& \M(i,j) - c\Delta_{(\lambda_{r+1} +1)}. %\quad (\text{Lemma }\ref{lem:lem-ineq}).
	\end{eqnarray} 
Letting $j\in A_r$, we have by Lemma~\ref{lem:lem-delta-star}
\begin{equation}
\label{eq:eq-lem-ineq-gap0}
%	(\widehat{\Delta}_{i,r}^\star)^+ - (\Delta_i^\star)^+ \leq 2c\Delta_{(\lambda_{r+1}+1)} \text{ and } 
	(\widehat{\Delta}_{j,r}^\star)^+ - (\Delta_j^\star)^+ \geq -c\Delta_{(\lambda_{r+1}+1)}
\end{equation}

Put together, \eqref{eq:eq-lem-ineq-gap0} and \eqref{eq:eq-lem-ineq-gap-a} yields 
\begin{eqnarray*}
\Mh(j,i;r)^+ + (\widehat \Delta_{j,r}^\star)^+ &\geq&  \M(j,i)^+ + (\Delta_{j}^\star)^+ -2c\Delta_{(\lambda_{r+1} +1)}% \quad %\text{and}\\
%	\Mh(i,j;r)&\geq& \M(i,j) - 2c\Delta_{(\lambda_{r+1} +1)}.
\end{eqnarray*}
for any $j\in A_r$. Which combined with \eqref{eq:eq-lem-ineq-gap-b} yields 
$$ [\Mh(i,j;r) \land (\Mh(j,i;r)^+ + (\widehat \Delta_{j,r}^\star)^+)] \geq [\M(i,j) \land (\M(j,i)^+ + (\Delta_{j}^\star)^+)] - 2c\Delta_{(\lambda_{r+1}+1)}$$
for any $j\in A_r$. Therefore,
\begin{eqnarray*}
	\widehat \delta_{i,r}^\star &\geq& \left(\min_{j\in A_r\backslash\{i\}} [\M(i,j) \land (\M(j,i)^+ + (\Delta_{j}^\star)^+)]  \right) - 2c\Delta_{(\lambda_{r+1} +1)}, \\
	&\geq& \left(\min_{j\in [K]\backslash\{i\}} [\M(i,j) \land (\M(j,i)^+ + (\Delta_{j}^\star)^+)]  \right) - 2c\Delta_{(\lambda_{r+1} +1)},\\
	&=& \delta_i^\star - 2c\Delta_{(\lambda_{r+1} +1)}. 
\end{eqnarray*}
\end{proof}

Finally, we prove Lemma~\ref{lem:lem-gap-bound}.

\lemGapBound*
\begin{proof}
	Let $i\in A_r \cap (\cS^\star)^\complement$. Since $i\in (\cS^\star)^\complement$, $\Delta_i = \Delta_i^\star$. 
	If $i\notin S_r$ then $ \widehat \Delta_{i, r} = \widehat \Delta_{i, r}^\star$. If $i \in S_r$ then 
	$\widehat\Delta_{i, r} = \widehat\delta_{i,r}^\star$ and 
	 $$ \widehat\Delta_{i, r}^\star \leq 0 \leq \widehat\delta_{i,r}^\star $$
which follows since $i\in S_r$. Therefore in both cases 
\begin{eqnarray*}
 \widehat\Delta_{i, r}  - \Delta_i &\geq& \widehat\Delta_{i, r}^\star - \Delta_i^\star\\
 & = & \max_{j} \m(i,j ; r) - \m(i,i_\star) \\
 & \geq& \m(i,i_\star ; r) - \m(i,i_\star) \\
 & \geq& - c \Delta_{(\lambda_{r+1}+1)}
\end{eqnarray*}
where the last inequality uses Lemma~\ref{lem:lem-ineq}. 

Similarly, let $i\in A_r \cap \cS^\star$. We have $\Delta_i = \delta_i^\star$ (Lemma~\ref{lem:lem-gap-redef}) and $\widehat\Delta_{i, r} = \widehat\delta_{i, r}^\star$ if $i\in S_r$. If $i\in A_r \backslash S_r$ (empirically sub-optimal) then 
$$ \widehat \Delta_{i, r} = \widehat\Delta_{i, r}^\star \geq 0$$ and there exists $j\in A_r$ such that $\vmuh_{i, n_r} \prec \vmuh_{j, n_r}$ so for this arm $j$, $\Mh(i,j;r)<0$. Therefore $\widehat \delta_{i, r}^\star<0$ and 
$$\widehat \delta_{i, r}^\star < 0 \leq \widehat \Delta_{i, r} = \widehat\Delta_{i, r}^\star.$$ Put together we have in all cases 
$$ \widehat \Delta_{i, r}  - \Delta_i \geq \widehat \delta_{i, r}^\star - \Delta_i\geq -2c \Delta_{(\lambda_{r+1} +1)},$$
where the last inequality follows from Lemma~\ref{lem:lem-gap-ineq}.
\end{proof}

\subsection{Proof of Lemma~\ref{lem:lem_dom}}

 \lemDom*

\begin{proof}
We have by definition 
	\begin{eqnarray*}
		\vmuh_{i, n_r}  \nprec \vmuh_{i^\star, n_r} &\impl& \exists\; d\;: \muh_{i, n_r}^d > \muh_{i^\star, n_r}^d\\
		&\impl& \exists\; d\;: (\muh_{i, n_r}^d - \mu_i^d) -(\muh_{i^\star, n_r}^d - \mu_{i^\star}^d) > \mu_{i^\star}^d - \mu_{i}^d\geq \m(i, i^\star),
	\end{eqnarray*}
	so% there exists $d$ such that 
	$$\left\| (\vmuh_{i, n_r}- \vmuh_{i^\star, n_r}) - (\vmu_i - \vmu_{i^\star})\right\|_\infty \geq \Delta_i^\star,$$
	%$$\lvert \muh_{i, n_r}^d - \mu_i^d\lvert \geq \frac12 \Delta_{i}^\star  \quad \textrm{or}\quad \lvert \muh_{i^\star, n_r}^d - \mu_{i^\star}^d\lvert \geq \frac12 \Delta_{i}^\star, $$
	which, on the event $\cE_c^1$ is only possible if $$\Delta_i^\star \leq c\Delta_{(\lambda_{r+1} +1)}.$$
\end{proof}

\subsection{Proof of Lemma~\ref{lem:hoeff-good-event}}

\HoeffGoodEvent*
 \begin{proof}
 	We have 
 	\begin{eqnarray*}
 		\bP(\overline{(\cE_c^1)}) &\leq& \sum_{r=1}^R \sum_{i\in \cS^\star} \sum_{j\neq i} \sum_{d=1}^D \bP\left(\lvert (\muh_{i, n_r}^d - \muh_{j,n_r}^d)-(\mu_i^d - \mu_j^d)\lvert > c\Delta_{\lambda_{r+1}+1} \right),\\
 		&\leq& 2\lvert \cS^\star\lvert(K-1)D \sum_{r} \exp\left( - \frac{c^2n_r\Delta_{\lambda_{r+1}+1}^2}{4\sigma^2}\right) \quad \text{(Hoedding's inequality)},\\
 		&\leq& 2\lvert \cS^\star\lvert(K-1)D R \exp\left( - \frac{c^2 \widetilde T^{R, \bm t, \bm\lambda}(\nu)}{4\sigma^2}\right). 
 	\end{eqnarray*}
 \end{proof}

\section{LOWER BOUND} 
\label{sec:lowerBound}
In this section, we state our lower bounds for fixed-budget PSI that hold for some class of instances. 
\paragraph{Illustration of the sub-optimality gaps}
Before presenting of our lower bounds, we give additional details on the sub-optimality gaps. Recall that for a sub-optimal arm we have 
\begin{equation}
\label{eq:def-gap-sub}
	\Delta_i := \Delta_i^\star:= \max_{j \in \cS^\star} \m(i,j), 
\end{equation}
%which is the smallest quantity that should be added component-wise to $\vmu_i$ to make $i$ appear Pareto optimal w.r.t $\{\vmu_k: k\in [K]\setminus \{i\} \}$.
and for an optimal arm $i \in \cS^\star$,
\begin{equation}
\label{eq:def-gap-opt}
    \Delta_i := 
    \min (\delta_i^+, \delta_i^-)
\end{equation}
where 
\begin{eqnarray*}
	\delta_i^+ &:=& \min_{j\in \cS^\star \setminus \{i\}} \min(\M(i,j), \M(j,i))\;, \\
	\delta_i^- &:=& \min_{j\in [K]\setminus \cS^\star}[(\M(j,i))^+ + \Delta_{j}] 
\end{eqnarray*} with the convention $\min_\emptyset = +\infty$. 
First note that when $D=1$ we recover the classical definition of sub-optimal gaps in single-objective bandit \citep{kaufmann_complexity_2014, audibert_best_2010}. Indeed, in this case there is a single optimal arm denoted $a_\star$ and $\cS^\star = \{a_\star\}$ then as any arm $a\neq a_\star$ is sub-optimal it comes  
\begin{eqnarray*}
\Delta_a &=& \Delta_{a}^\star = \m(a, a_\star)	\\
&=& \theta_{a_\star} - \theta_a 
\end{eqnarray*}
and as $\delta_{a_\star}^+ = \infty$ we have 
\begin{eqnarray*}
\Delta_{a_\star} &=& \delta_{a_\star}^- =   	\min_{j\neq a_\star}[(\M(j,a_\star))^+ + \Delta_{j}]\\
&=&  	\min_{j\neq a_\star} \Delta_{j}
\end{eqnarray*}
where the last line follows since every $j\neq a_\star$ is dominated by $a_\star$. Now we illustrate these gaps in higher-dimension. Following Appendix~A of \cite{auer_pareto_2016} we illustrate below the sub-optimality gaps in dimension $D=2$. 

\medskip
\definecolor{yellow2}{rgb}{0.8862745 , 0.84313726, 0.}
\definecolor{red2}{rgb}{0.9607843 , 0.3137255 , 0.07450981}
\definecolor{green3}{rgb}{0.        , 0.43529412, 0.52156866}
\definecolor{blue}{rgb}{0.08627451211214066, 0.125490203499794, 0.23529411852359772}
\begin{minipage}{\linewidth}
\centering
\resizebox{0.3\linewidth}{0.3\linewidth}{
\begin{tikzpicture}
%\draw[step=0.5cm,gray,very thin] (-3.5,-3.5) grid (3.5 , 3.5);
   \draw [ultra thick, draw=gray, ->] (-6.2,-6) -- (6,-6) node[right, black] {};
  \draw [ultra thick, draw=gray, ->] (-6,-6.2) -- (-6,6) node[above, black] {};
  
\node[circle, fill=yellow2, label=below:{$i$},inner sep=0pt, minimum size=10pt,  ultra thick] (i) at (-4.5, -1) {};
\node[circle, fill=green3, label=below:{$j$},inner sep=0pt, minimum size=10pt, ultra thick] (j) at (-2.0, 2) {};
\node[circle, fill=red2, label=below:{$k$},inner sep=0pt, minimum size=10pt,  ultra thick] (k) at (3., 0.5) {};
\node[circle, fill=blue, label=below:{$\ell$},inner sep=0pt, minimum size=10pt,  ultra thick] (l) at (1.5, -3){};

\draw[->,ultra thick, dashed](i) -- (-2, -1) node[below] {$\m(i,j)$};
\draw[->,ultra thick, dashed](i) -- (-4.5, 0.5) node[above] {$\m(i,k)$};
\draw[->,ultra thick, dashed](l) -- (3, -3) node[above] {$\m(\ell,k)$};
\draw[->,ultra thick, draw = purple](k) -- (3, 2) node[anchor=north west] {$\M(j,k)$};
\draw[->,ultra thick, draw = purple](j) -- (3, 2) node[anchor=south east] {$\M(k,j)$};
\draw[->,ultra thick, dotted, draw = purple](-2.0, -4.5) -- (1.5, -4.5) node[above] {$\M(\ell, j)$};

%\draw[mark=+, ultra thick,  mark size=4pt] plot coordinates {(-2.5, 3)} node[anchor=north west]{\large $1$};
%\draw[mark=+, ultra thick, mark size=4pt] plot coordinates {(-3.0, 1)} node[anchor=south west]{\large $2$};
%\draw[mark=+, ultra thick, mark size=4pt] plot coordinates {(2., -0.5)} node[anchor=south west]{\large $3$};
%\draw[mark=+, ultra thick, mark size=4pt] plot coordinates {(1.5, -3)} node[anchor=south west]{\large $4$};
%\draw[->, red, dashed, ultra thick](-3,1) -- (-2, 1);
%\draw[->, red, dashed, very thick](-2.5,3) -- (-3.5, 3);
%\draw[->, red, dashed, very thick](2,-0.5) -- (1, -0.5);
%\draw[->, red, dashed, very thick](1.5,-3) -- (2.5, -3);
%\draw[->, red, dashed](1,1) -- (2, 1);
\end{tikzpicture}}
\captionof{figure}{Illustration of the sub-optimality gaps. Plain lines represent "distances'' between Pareto optimal arms and dashed lines are for margins from sub-optimal to optimal arms.}
\label{fig:lbd-illusxx}	
\end{minipage}

In this instance $\Delta_i = \m(i, j)$ and it is easy to see that if $i$ is increased by $\Delta_i$ in both $x$ and $y$ axes it will become non-dominated. We also have $\Delta_\ell = \m(\ell, j)$. As $\ell$ is only dominated by $j$, if is it translated by $\m(\ell, j)$ on the $x$-axis it will become Pareto optimal. For Pareto-optimal arms $k,j$, $\delta^+_{k} = \delta^+_{j} = \M(j, k)$. As $k$ dominates both $i$ and $\ell$ its margin to sub-optimal arms is $\delta_k^- = \min(\Delta_i, \Delta_\ell)$ and we have 
$\delta_j^- = \min(\M(\ell, j) + \Delta_\ell, \Delta_i)$. Finally, observe that for both $j, k$, $\Delta_j = \Delta_k = \M(j, k)$. If $k$ is translated by $\M(j, k)$ on the $y$-axis it will dominate $j$. Similarly, if $j$ is translated by $-\M(j, k)$ on the $y$-axis, it will be dominated by $k$. 

%$\delta_i^+$ accounts for how much $i$ is close to dominate (or to be dominated by) another optimal arm while $\delta_i^-$ translates in a sense the smallest ``margin" from an optimal arm $i$ to the sub-optimal arms.

%\todoc{paragraph sur les gaps}
We will now prove the lower bound. The idea is to build alternative instances where the Pareto set is changed and the complexity is less than that of the base instance. To obtain a different Pareto set we can either make optimal an arm that was sub-optimal or make sub-optimal an arm that was optimal.
However in doing so, we need to take care not to decrease the gap of any arm, in particular that of the arm that is shifted, otherwise the alternative instance could be harder.  
This is challenging as the definition of the gaps are completely different for optimal and sub-optimal arms. In the sequel we restrict to some class of instances. 
\newcommand{\undi}{\underline i}
 
 We define $\cB$ to be the set of means $\bm\Theta\in \bR^{K\times D}$ such that each sub-optimal arm $i$ is only dominated by a single arm, denoted by $i^\star$ (that has to belong to $\cS^\star$) and that for each optimal arm $j$ there exists a unique sub-optimal arm which is dominated by $j$, denoted by $\underline j$.  %for any instance $\nu$ with means $\bm\Theta$: (1) for any $i\in \cD(\nu)$ 
We further assume that optimal arms are not too close to arms they don't dominate:
%$$ \forall i,j \in \cS^\star,\; \M(i, j)\geq 3\max(\Delta_{\underline i}^\star, \Delta_{\underline j}^\star),$$
%and 
for any sub-optimal arm $i$ and optimal arm $j$ such that $\vmu_i\nprec \vmu_j$, 
$$ \M(i, j)\geq 3\max(\Delta_i, \Delta_{\underline j}).$$
An instance $\nu$ belongs to $\cB$ if its means $\bm\Theta = (\vmu_1 \dots \vmu_K)^\T$ belongs to $\cB$. 
\medskip 
\thmLbdConsistent*
\begin{proof}
$\nu_1, \dots, \nu_K$ are distributions over $\bR^D$ parameterized by their means (resp.) $\vmu_1, \dots, \vmu_K$. We let $\nu:= (\nu_1, \dots, \nu_K)$ and $\bm\Theta^{(0)} := (\vmu_1 \dots \vmu_K)^{\T}\in \bR^{K\times D}$. We denote by $\cS^\star(\nu)$ the Pareto optimal set of the bandit instance $\nu$. For all $i\in [K]$, we define  $ \bm\Theta^{(i)} := (\bm\theta_1^{(i)}\;\dots\;\bm\theta_K^{(i)})^\T \in \bR^{K\times D}$ (with $\vmu^{(i)}$ to be defined later) and we let $\nu^{(i)}$ denote the bandit with parameter $\bm\Theta^{(i)}$. Let $\cA$ denote an algorithm for Pareto set identification in the fixed-budget setting. %Let $\cF$ denote the natural $\sigma$-algebra associated to the stochastic process.
We further assume that for any $i\in [K]$, $\bm\Theta_{(i)}$ differs from $\bm\Theta^{(0)}$ only in line $i$:  
\begin{equation}
	\forall a \neq i, \; \bm\theta_a^{(i)}= \bm\theta_a, 
\end{equation}
so $\nu^{(i)}$ and $\nu$ only differs in arm $i$.

We state the following lemma which is a rewriting of Lemma~16 of \cite{kaufmann_complexity_2014}. 
\begin{lemma}[\cite{kaufmann_complexity_2014}]
\label{lem:infth}
	Let $\nu$ and $\nu'$ be two bandit instances such that $\cS^\star(\nu)\neq \cS^\star(\nu')$. Then 
	$$ \max\left(e_T^{\cA}(\nu), e_T^{\cA}(\nu')\right) \geq \frac14 \exp\left( -\sum_{a=1}^K \bE_\nu[T_a(T)] \KL(\nu_i, \nu_i')\right)$$ 
\end{lemma}
%We introduce the sets
%$$ \cI_1 := \{ i \in \cS^\star(\nu): \exists j \in [K]: \delta_i^+ = \M(i, j) \} $$ and similarly 
%$$ \cI_2 := \{ i \in \cS^\star(\nu): \exists j \in [K]: \delta_i^+ = \M(j, i) \}.$$ 
For any $i\in [K]$, \begin{equation}
\label{eq:alt}
	\vmu_i^{(i)} := \begin{cases}
		\vmu_i + 2\Delta_i e_{d_i} & \text{ if } i \not\in \cS^\star(\nu) ,\\
		\vmu_i -2\Delta_i e_{d_{\underline i}} &\text{ else.}
	\end{cases}
\end{equation}
Let for any $i$, $\Omega(i):=\{ j: \vmu_i \prec \vmu_j\}$ the set of arms that dominate $i$. We recall that for a sub-optimal arm $i$, 
$$ d_i \in \argmin_{d} \left[ \mu^d_{i^\star} - \mu^d_i\right].$$
We now justify that for any $i\in [K], \cS^\star(\nu^{(i)}) \neq \cS^\star(\nu)$, by considering two cases:
\begin{itemize}
 \item If $i \notin  \cS^\star(\nu)$,
 as $\Omega(i) = \{i^\star\}$ and $i$ has been increased to be larger than $i^\star$ on $d_i$, we have that $i \in \cS^\star(\nu^{(i)})$. 
 \item If $i \in \cS^\star(\nu)$, since there exists $\underline i$ such that $\Omega(\underline i) = \{ i\}$ and $\Delta_i = \Delta_{\underline i}^\star$, the shifting ensures that $\vmu^{(i)}_{\undi}\nprec \vmu_i^{(i)}$ it comes that $\underline i \in \cS^\star(\nu^{(i)})$.
% \item let $j\in \cS^\star(\nu)$ be such that $\delta_i^+ = M(i,j)$. Then in $\nu^{(i)}$, we can check that arm $j$ dominates arm $i$, hence $i \notin \cS^\star(\nu^{(i)})$, whereas $i \in \cS^\star(\nu)$
 %\item If $i \in \cI_2$, let $j \in \cS^\star(\nu)$ be such that $\delta_j^{+} = M(j,i)$. Then in $\nu^{(i)}$, we can check that arm $i$ dominates arm $j$, hence $j\notin \cS^\star(\nu^{(i)})$, whereas $j \in \cS^\star(\nu)$
\end{itemize}

We remark that in both cases, we have either $\vmu_i^{(i)} = \vmu_i + 2\Delta_{i}e_{d_i}$ or $\vmu_i^{(i)} = \vmu_i - 2\Delta_{i}e_{d_{\underline i}}$. Recalling that for multivariate Gaussian distribution with the same covariance matrix, $\KL\left(\cN(\bm\mu,\Sigma), \cN(\bm\mu',\Sigma)\right) = \frac{1}{2} (\bm\mu-\bm\mu')^\T \Sigma^{-1} (\bm\mu-\bm\mu')$, we get for all $i$
\begin{eqnarray*}
	\KL(\nu_i, \nu_i^{(i)}) &=& 2 \Delta_i^2 / \sigma^2. 
\end{eqnarray*}
By applying  Lemma~\ref{lem:infth}, we have for any $i\in [K]$, 
\begin{eqnarray*}
	\max(e_T(\nu), e_T(\nu^{(i)})) &\geq&  \frac 1 4 \exp\left( -\bE_\nu[T_i(T)] \KL(\nu_i, \nu_i^i)\right) \\
	& \geq &  \frac 1 4 \exp\left( -\frac{2\bE_\nu[T_i(T)]\Delta_i^2}{\sigma^2} \right)
\end{eqnarray*}
We conclude by using the pigeonhole principle: there must exist $i \in [K]$ such that $\bE_{\nu}[T_i(T)]\leq \frac{T}{{H}(\nu){\Delta}_i^2}$. Therefore, there exists $i\in [K]$ such that $$\max(e_T(\nu), e_T(\nu^{(i)})) \geq  \frac14 \exp\left( -\frac{2 T}{\sigma^2H(\nu)}\right).  $$
 It remains to show that for any $i\in [K], H(\nu^{(i)}) = H(\nu)$. 
 We introduce the following notation for any arms $p,q$
 \begin{eqnarray*}
 	\M^i(p,q) &:=& \max_d [(\mu_p ^i)^d - (\mu_q ^i)^d ], \\
 	\m^i(p, q) &:=& -\M^i(p, q), 
 \end{eqnarray*}
 which are the quantities $\M$ and $\m$ computed in the bandit $\nu^{(i)}$. We denote by $\cD(\nu):=(\cS^\star(\nu))^\complement$ the set of non-optimal arms in the bandit $\nu$.
Letting $i$ be fixed we denote by $\Delta_k^{(i)}$ the sub-optimality gap of arm $k$ in the bandit $\nu^{(i)}$.
$\delta_k^{-, (i)}$, $\delta_k^{+, (i)}$ and $\Delta_k^{\star, (i)}$ are respectively the equivalent of $\delta_k^-$, $\delta_k^+$ and $\Delta_k^\star$ computed in the bandit $\nu^{(i)}$.
 We will prove that 
\begin{equation}
	\label{eq:lbd-fund}
	\forall\; k\in [K],\; \Delta_k^{(i)} = \Delta_k.
\end{equation}
We recall the assumptions on the class $\cB$: 
\begin{eqnarray}
	&\text{}& \forall i \in \cD(\nu),\;  \exists!\text{ arm }\;  i^\star \text{ such that } \vmu_i \prec \vmu_{i^\star}, \label{eq:lbd-assump1}\\
	&\text{}& \forall i \in \cS^\star(\nu),\; \exists! \; \text{ arm } \underline i \text{ such that } \vmu_{\underline i} \prec \vmu_i, \label{eq:lbd-assump2}\\
	&\text{}&
	\forall i \in \cD(\nu),\; \forall j \in \cS^{\star}(\nu) \text{ such that } \vmu_i\nprec \vmu_j, \; \M(i, j)\geq 3\max(\Delta_{\underline j}, \Delta_i). \label{eq:lbd-assump4}
\end{eqnarray}
As for any optimal arm $i$ there exists $\undi$ which is only dominated by $i$, 
Assumption \eqref{eq:lbd-assump2} and \eqref{eq:lbd-assump4} yield  
\begin{equation}
\label{eq:lbd-assump3}
	 \forall i, j \in \cS^\star(\nu),\; i\neq j,\;  \M(i, j) \geq 3\max(\Delta_{\undi}, \Delta_{\underline j})
\end{equation}
 In the sequel, we fix $i\in [K]$ and we prove \eqref{eq:lbd-fund} by analyzing first the case $i\notin \cS^\star(\nu)$ then the case $i\in \cS^\star(\nu)$. 
 
 \paragraph{\bfseries Step 1:} Assume $i\notin \cS^\star(\nu)$.
 
 We recall that $\vmu_i^{(i)} = \vmu_i + 2\Delta_i e_{d_i}$ and for any $j\neq i$, $\vmu_i^{(i)} = \vmu_i$. Then 
 \begin{equation}
 	\label{eq:lbd-mM}
 	\forall p\neq i,\;  \forall q\neq i,\;  \M^i(p, q) = \M(p, q) \text{ and } \m^i(p, q) = \m(p, q). 
 \end{equation}
 Since only arm $i$ has changed (and increased in one coordinate) from $\nu$ to $\nu^{(i)}$, if $j\in \cD(\nu)\setminus\{i\}$ then $j\in \cD(\nu^{(i)})$. We will prove that for any $j\in \cS^\star(\nu)$, 
 \begin{equation}
 \label{eq:lbd-w1}
  \vmu_j=\vmu_j^{(j)} \nprec \vmu_i^{(i)}.	
 \end{equation}
 By contradiction, assume an arm $j$ exists such that \eqref{eq:lbd-w1} holds. Then $\M(j, i) \leq 2\Delta_i$ and 
 $$ \M(j,i^\star) \leq \M(j, i)\leq 2\Delta_i,$$
 which is not possible by \eqref{eq:lbd-assump3} as $j, i^\star \in \cS^\star(\nu)$. Therefore,  \eqref{eq:lbd-w1} does not hold and we deduce:
 $$ \cD(\nu^{(i)}) = \cD(\nu) \setminus\{i\} \quad \text{ and }\quad \cS^\star(\nu^{(i)}) = \cS^\star(\nu) \cup \{ i\}.$$
 Moreover, we will show that 
 $$ \forall j \in \cD(\nu^{(i)}), \Delta_j^{\star,(i)} = \Delta_j^\star.$$
 Indeed, as for any $j\in \cD(\nu^{(i)})$, $j, j^\star$ have not been shifted, we have 
 $$ \Delta_{j}^{\star, (i)} \geq \Delta_j^\star$$
 and if the inequality is strict then, since only $i$ has changed,
  $$ \Delta_{j}^{\star, (i)} = \m^i(j, i)>\Delta^\star_j,$$ which yields (as only $d_i$ has been increased)
 $$ \m(j, i)> -\Delta^\star_j,$$ so 
 $$ \M(j, i) < \Delta_j^\star,$$
 which combined with $\M(j, i^\star) \leq \M(j, i)$ yields 
 $$ \M(j, i^\star) < \Delta_j^\star$$
 which is not possible by assumption \eqref{eq:lbd-assump4}. 
 In short, we have proved that for any $j\in \cD(\nu^{(i)})$, 
 $$ \Delta_{j}^{\star, (i)} = \Delta^\star_j.$$
Let $j\in \cS^\star(\nu^{(i)})\backslash\{i\}$. For any $k\neq i$
  $$ \M^i(j, k) = \M(j, k) \text{ and } \M^i(k, j) = \M(k, j).$$
  
  We have 
  \begin{eqnarray}
 	\M^i(i^\star, i) &=& \left(\max_{d\neq d_{i}} \left[ \mu_{i^\star}^{d} - \mu_{i}^d\right]\right)  \lor \left( \mu_{i^\star}^{d_{i}} - \mu_{i}^{d_{i}} - 2\left( \mu_{i^\star}^{d_{i}} - \mu_{i}^{d_{i}} \right)\right), \nonumber\\\nonumber
 	&=& \left(\max_{d\neq d_{i}} \left[ \mu_{i^\star}^{d} - \mu_{i}^d\right]\right) \lor \left(\mu_{i}^{d_{i}} - \mu_{i^\star}^{d_{i}}  \right),\nonumber\\\nonumber
 	&\overset{(a)}{=}& \max_{d\neq d_{i}} \left[ \mu_{i^\star}^{d} - \mu_{i}^d\right], \\
 	&\overset{(b)}{=}& \M(i^\star, i),\nonumber\\
 	&\overset{(c)}{\geq}& \Delta_{i}^\star = \Delta_{i^\star} \label{eq:lbd1-eq-00} 
 \end{eqnarray}
 where $(a)$ follows since $(\mu_{i}^{d_{i}} - \mu_{i^\star}^{d_{i}})<0$ while $\M^i(i^\star, i)>0$, $(b)$ follows from the definition of $d_{i}$ as the argmin and $(c)$ follows from the inequality $\M(p, q)\geq \m(q, p)$ for any $p,q$. 

Additionally for $j\in \cS^\star(\nu^{(i)})\backslash\{i, i^\star\}$,  
  \begin{eqnarray}
  	\M^i(j, i) &\geq& \M(j, i) - 2\Delta_i,\\
  	&\geq& \M(j, i^\star) - 2\Delta_i, \\
  	&\geq& \frac13 \M(j, i^\star), \quad (\text{since } \M(j, i^\star)\geq 3\Delta_i \text{ by } \eqref{eq:lbd-assump3})\\
  	&\geq&\max(\Delta_{\underline j}, \Delta_i) \geq \Delta_j. \label{eq:lbd-h1} 
  \end{eqnarray}
  Combining \eqref{eq:lbd-h1} with \eqref{eq:lbd1-eq-00} yields for $j\in \cS^\star(\nu^{(i)})\backslash\{i\}$, \begin{equation}
  \label{eq:lbd1-eq-01}
  	\M^i(j, i) \geq \Delta_j.
  \end{equation}
We now compute $\M^i(i,j)$. Direct calculation yields
$$ \M^i(i, i^\star) = \Delta_i = \Delta_{i ^\star},$$
 and if $j\neq i^\star$, 
 \begin{eqnarray} 
 	\M^i(i, j)&\geq& \M(i, j),\\
 	%&\geq& \M(i, j^\star),\\% \quad \text{(by convention $j^\star =j$ if $j\in \cS^\star(\nu)$)}\\
 	&\geq& \max(\Delta_j, \Delta_i) \quad \text{(from \eqref{eq:lbd-assump4}}). \label{eq:lbd-w2}
 \end{eqnarray}
With what precedes we state the following for any $j\in \cS^\star(\nu^{(i)})\backslash\{i\}$: 
 \begin{enumerate}[a)]
 	\item $
 	\min_{k\in \cS^\star(\nu^{(i)})\backslash \{j\}} \M^i(j, k) = \left(\min_{k\in \cS^\star(\nu)\backslash \{j\}}\M(j, k)\right) \land \M^i(j, i)$,
 	\item $\min_{k\in \cS^\star(\nu^{(i)})\backslash \{j\}} \M^i(k, j) = \left(\min_{k\in \cS^\star(\nu)\backslash \{j\}} \M(k, j)\right) \land \M^i(i, j)$, \text{ and } 
 	\item $\min_{k\in \cD(\nu^{(i)})} \left[(\M^i(k, j))^+ +\Delta_k^{\star, (i)}\right] =  \left(\min_{k\in \cD(\nu)\backslash\{ i\}} \left[\M(k, j)^+ +\Delta_k^{\star}\right]\right)$. 
 \end{enumerate}
 Which combined with $\delta_j^+ \geq \Delta_j$ and \eqref{eq:lbd-w2}  and \eqref{eq:lbd1-eq-01} yields 
 $$ \delta_j^{+, (i)} = \min(\delta_j^+, \min(\M^i(i,j), \M^i(j, i)))\geq \Delta_j ,$$
 and for $j\neq i^\star$, 
 $$ \delta_j^{-, (i)} = \Delta_j $$ 
% $$ \Delta_j^i \geq \Delta_j.$$
 As a consequence, for $j\neq i^\star$, $\Delta_j^{(i)} = \Delta_j$. For $j=i^\star$, $\M^i(i, i^\star) = \Delta_i=\Delta_{i^\star}$  
 and by \eqref{eq:lbd1-eq-00}
 \begin{eqnarray*}
 	\M^i(i^\star, i) &\geq& \Delta_{i}^\star = \Delta_{i^\star}.
 \end{eqnarray*}
 Therefore, $\min(\M^i(i,i^\star), \M^i(i^\star, i)) = \Delta_{i^\star}$ and as $\delta_{i^\star}^+ \geq \Delta_{i^\star}$, we have 
  $$ \delta_{i^\star}^{+, (i)} = \min(\delta_{i^\star}^+, \min(\M^i(i,i^\star), \M^i(i^\star, i))) =  \Delta_{i^\star},$$
   and $$\delta_{i^\star}^{-, (i)} >\Delta_{i^\star}, \quad \text{(from \eqref{eq:lbd-assump4})}$$
   so $$ \Delta_{i^\star}^{(i)} = \Delta_{i ^\star}$$
 To sum up, so far we have proved that for any $j\neq i$, 
 $$ \Delta_j^{(i)} =  \Delta_j.$$
 Proving that $\Delta_i^{(i)} = \Delta_i$ will conclude the proof for Step 1.  
 
 Recall that 
 $$ \Delta_i^{(i)} := \left(\min_{j\in \cS^\star(\nu)} \min(\M^i(i, j), \M^i(j, i))\right) \land \left( \min_{j\in \cD(\nu)\setminus\{i\}} [(\M^i(j, i))^+ + \Delta_j^{\star, (i)}]\right), $$
with $\min(\M^i(i, i^\star), \M^i(i^\star, i)) = \Delta_i$ and by \eqref{eq:lbd-w2} and \eqref{eq:lbd-h1}:
\begin{eqnarray*}
	\forall j \in \cS^\star(\nu) \setminus \{i^\star\}, \min(\M^i(i, j), \M^i(j, i)) \geq \Delta_i. 
\end{eqnarray*}

Moreover, for any $j\in \cD(\nu) \backslash\{i\}$, 
$$ \M^i(j, i)\geq \M(j, i^\star) - 2\Delta_{i} \overset{\eqref{eq:lbd-assump4}}{\geq}\Delta_i$$
%\todoc{condition }
%$$ \M^i(j, i) \geq \Delta_i.$$
Therefore, $\delta_i^{-, (i)} \geq \Delta_i$ and $\delta_i^{+, (i)} = \Delta_i$. We conclude that
$$\Delta_i^{(i)} = \Delta_i.$$
 which achieves the proof of Step 1: If $i\notin \cS^\star(\nu)$ 
 $$ H(\nu^{(i)}) = H(\nu).$$  
 
 \paragraph{\bfseries Step 2:} Assume $i\in \cS^\star(\nu)$. 
 
 We have $\vmu_i^{(i)} = \vmu_i - 2\Delta_{i}e_{d_{\underline i}}$ and for any $j\neq i$, $\vmu_j^{(i)} = \vmu_j$. 
 First, note that $\underline i \in \cS^\star(\nu^{(i)})$. This is immediate as $i$ is the only arm such that $\vmu_{\underline i} \prec \vmu_i$ and due to the shifting, $ \vmu_{\underline i}^{(i)} \nprec \vmu^{(i)}_i $. We claim that $i \in \cS^\star(\nu^{(i)})$. Otherwise, there exists $j\in \cS^\star(\nu)$ such that $\M(i, j)\leq 2\Delta_i$ which is not possible by \eqref{eq:lbd-assump3}. For any $j\in \cD(\nu) \setminus\{\underline i\}$, $j$ and $j^\star$ have not been shifted, so $j\in \cD(\nu^{(i)})$ and similarly to Step 1, we have 
 
 $$ \cS^\star(\nu^{(i)}) = \cS^\star(\nu) \cup \{ \underline i\} \text{ and } \cD(\nu^{(i)}) = \cD(\nu) \backslash \{ \underline i\},$$
 and for any $j\in \cD(\nu^{(i)})$, 
 $$ \Delta_j^{\star, (i)} = \Delta_j^\star.$$
 
 Let $j \in \cS^\star(\nu^{(i)})\backslash\{ i, \underline i\}$.
 We have, 
 $$ \delta_j^{+, (i)} = \left( \min_{k\in \cS^\star(\nu) \backslash\{j, i\}} \left[ \min(\M(k, j), \M(j, k))\right]\right) \land \underbrace{\left(\min(\M(j, \underline i), \M(\underline i, j))\right)}_{\alpha_j} \land \underbrace{\left(\min(\M^i(j, i), \M^i(i, j))\right)}_{\beta_j}.$$

% \todoc{cond}
 We have 
 \begin{eqnarray*}
 	\M^i(j, i) &\geq& \M(j, i) \\
 	&\geq& 3\max(\Delta_i, \Delta_j)\quad \text{(from \eqref{eq:lbd-assump3})}, 
 \end{eqnarray*}
 and 
  \begin{eqnarray*}
 	\M^i(i, j) &\geq& \M(i, j) -2\Delta_i\\
 	&\geq& \max(\Delta_j, \Delta_i) \quad \text{(from \eqref{eq:lbd-assump3})}. 
 \end{eqnarray*}
 It follows that
 \begin{equation}
 \label{eq:lbd-w3}
 	\beta_j \geq \max(\Delta_{j}, \Delta_i).
 \end{equation}
 On the other side, 
 \begin{eqnarray*}
 	\M^i(j, \underline i) &=& \M(j, \underline i), \\
 	&\geq& \M(j, i) \quad \text{(as $\vmu_{\underline i} \prec \vmu_i$)},\\
 	&\geq& 3\max(\Delta_j, \Delta_i) \quad \text{(from \eqref{eq:lbd-assump3})} 
 \end{eqnarray*}
and 
 \begin{eqnarray*}
 	\M^i(\underline i, j) &=& \M(\underline i, j) \geq \max(\Delta_j, \Delta_{\underline i})  \quad \text{(from \eqref{eq:lbd-assump4}})
 \end{eqnarray*}
 %so $\M^i(\underline i, i) = \Delta_i$ and for $j\neq i$, $\M^i(\underline i, j)\geq \Delta_j$ (from \eqref{eq:lbd-assump4}). 
 Therefore, 
 \begin{equation}
 \label{eq:lbd-z3}
 \alpha_j \geq \max(\Delta_j, \Delta_{\undi}), 
 \end{equation} which yields 
 $$ \delta_j^{+, (i)} \geq \Delta_j.$$
 As $\underline j \in \cD(\nu^{(i)})$, we check that 
$$ \delta_{j}^{-, (i)} = \min_{k\in \cD(\nu) \backslash\{\underline i\}}\left[\M(k, j)^+ + \Delta_k^\star\right] = \Delta_{\underline j}^\star = \Delta_j.$$
Put together, 
we have for any $j \in \cS^\star(\nu^{(i)})\backslash\{i, \underline i\}$, 
$$ \Delta_j^{(i)} := \min(\delta_{j}^{+, (i)}, \delta_{j}^{-, (i)}) = \Delta_j.$$ 

 We also easily check that 
 \begin{equation}
 \label{eq:lbd-z1}
  \delta_{i}^{-, (i)} = \min_{k\in \cD(\nu) \backslash\{\underline i\}}\left[\M(k, i)^+ + \Delta_k^\star\right] > \Delta_{\underline i} = \Delta_i	
 \end{equation}
 and 
 \begin{eqnarray}
  \label{eq:lbd-z2}
 	\delta_i^{+, (i)} &=& \left( \min_{k\in \cS^\star(\nu) \backslash\{i\}} \left[ \min(\M^i(k, i), \M^i(i, k))\right]\right) \land \underbrace{\left(\min(\M^i(i, \underline i), \M^i(\underline i, i))\right)}_{\beta_{\underline i}}.%,\\
% 	&=& \left(\min_{k\in \cS^\star(\nu) \backslash\{i\}} \beta_k \right) \land \Delta_i,\\
% 	&=& \Delta_i.
 \end{eqnarray}
 %\todoc{cond check}
 We further note that $\M^i(\underline i, i) = \Delta_i$ and similarly to \eqref{eq:lbd1-eq-00}, 
 \begin{eqnarray*}
 	\M^i(i, \underline i) &=& \left(\max_{d\neq d_{\underline i}} \left[ \mu_{i}^{d} - \mu_{\underline i}^d\right]\right)  \lor \left( \mu_i^{d_{\undi}} - \mu_{\undi}^{d_{\undi}} - 2\left( \mu_i^{d_{\undi}} - \mu_{\undi}^{d_{\undi}} \right)\right), \\
 	&=& \left(\max_{d\neq d_{\underline i}} \left[ \mu_{i}^{d} - \mu_{\underline i}^d\right]\right) \lor \left(\mu_{\undi}^{d_{\undi}} - \mu_{i}^{d_{\undi}}  \right),\\
 	&\overset{(a)}{=}& \max_{d\neq d_{\underline i}} \left[ \mu_{i}^{d} - \mu_{\underline i}^d\right], \\
 	&\overset{(b)}{=}& \M(i, \undi),\\
 	&\overset{(c)}{\geq}& \Delta_{\underline i} = \Delta_i
 \end{eqnarray*}
 where $(a)$ follows since $(\mu_{\undi}^{d_{\undi}} - \mu_{i}^{d_{\undi}})<0$ while $\M^i(i, \undi)>0$, $(b)$ follows from the definition of $d_{\undi}$ as the $\argmin$ and $(c)$ follows from the inequality $\M(p, q)\geq \m(q, p)$ for any $p,q$. 
 Therefore we have $\beta_{\undi} = \Delta_i$  and from \eqref{eq:lbd-w3}, \eqref{eq:lbd-z3} and \eqref{eq:lbd-z1}, %\eqref{eq:lbd-z2}
 $$ \delta_{i}^{+, (i)} = \Delta_i \text{ and } \delta_i^{-, (i)}>\Delta_i$$
 so that finally 
 $$ \Delta_{i}^{(i)} = \Delta_i.$$
 It remains to compute $\Delta_{\underline i}^{(i)}$. 
 We have 
 \begin{eqnarray*}
 	\delta_{\underline i}^{+, (i)} &=& \left( \min_{k\in \cS^\star(\nu) \backslash\{\undi, i\}} \left[ \min(\M(k, \undi), \M(\undi, k))\right]\right)  \land \underbrace{\left(\min(\M^i(\undi, i), \M^i(i, \undi))\right)}_{\beta_{\undi}},\\
 	&=& \left(\min_{k\in \cS^\star(\nu) \backslash\{\undi, i\}} \alpha_k\right) \land \beta_{\undi},\\
 	&=& \Delta_{\undi},
 \end{eqnarray*}
 which follows from \eqref{eq:lbd-z3} and the fact that $\beta_{\undi} = \Delta_{\undi}$.
 On the other side, 
 \begin{eqnarray*}
 	\delta_{\undi}^{-, (i)} &=& \min_{k\in \cD(\nu) \backslash\{\underline i\}}\left[\M(k, \undi)^+ + \Delta_k^\star\right] \geq 3\Delta_{\underline i} \quad \text{(from \eqref{eq:lbd-assump4}}). \end{eqnarray*}
 Therefore, 
 $$ \Delta_{\undi}^{(i)} = \Delta_{\undi},$$
 which concludes the proof. 
 \end{proof}
 
 \medskip
 
The conditions of Theorem~\ref{thm:thm-lbd-consistent} include a large class of instances. In particular, for any value of $D$, on can find instances in $\cB$. We illustrate in Fig.\ref{fig:lbd-illus1} to Fig.\ref{fig:lbd-illus5} below an instance $\nu \in \cB$ with 4 arms in dimension 2, and the 4 corresponding alternative instances $\nu^{(i)}$. Similar examples could be given for any dimension and any number of arms.

\medskip 

\begin{minipage}{0.18\linewidth}
\resizebox{\linewidth}{\linewidth}{
\begin{tikzpicture}
\draw[step=0.5cm,gray,very thin] (-3.5,-3.5) grid (3.5 , 3.5);
\draw[mark=+, ultra thick,  mark size=4pt] plot coordinates {(-2.5, 3)} node[anchor=north west]{\large $1$};
\draw[mark=+, ultra thick, mark size=4pt] plot coordinates {(-3.0, 1)} node[anchor=south west]{\large $2$};
\draw[mark=+, ultra thick, mark size=4pt] plot coordinates {(2., -0.5)} node[anchor=south west]{\large $3$};
\draw[mark=+, ultra thick, mark size=4pt] plot coordinates {(1.5, -3)} node[anchor=south west]{\large $4$};
%\draw[->, red, dashed, very thick](-3,1) -- (-2, 1);
%\draw[->, red, dashed, very thick](-2.5,3) -- (-3.5, 3);
%\draw[->, red, dashed, very thick](2,-0.5) -- (1, -0.5);
%\draw[->, red, dashed, very thick](1.5,-3) -- (2.5, -3);
\end{tikzpicture}}
\captionof{figure}{Instance $\nu \in \cB$}
\label{fig:lbd-illus1}
\end{minipage}
\hspace{0.1cm}
\begin{minipage}{0.18\linewidth}
\resizebox{\linewidth}{\linewidth}{
\begin{tikzpicture}
\draw[step=0.5cm,gray,very thin] (-3.5,-3.5) grid (3.5 , 3.5);
\draw[mark=+, ultra thick,  mark size=4pt] plot coordinates {(-2.5, 3)} node[anchor=north west]{\large $1$};
\draw[mark=+, ultra thick, mark size=4pt] plot coordinates {(-3.0, 1)} node[anchor=south west]{\large $2$};
\draw[mark=+, ultra thick, mark size=4pt] plot coordinates {(2., -0.5)} node[anchor=south west]{\large $3$};
\draw[mark=+, ultra thick, mark size=4pt] plot coordinates {(1.5, -3)} node[anchor=south west]{\large $4$};
%\draw[->, red, dashed, very thick](-3,1) -- (-2, 1);
\draw[->, red, dashed, ultra thick](-2.5,3) -- (-3.5, 3);
%\draw[->, red, dashed, very thick](2,-0.5) -- (1, -0.5);
%\draw[->, red, dashed, very thick](1.5,-3) -- (2.5, -3);
%\draw[->, red, dashed](1,1) -- (2, 1);
\end{tikzpicture}}
\captionof{figure}{Instance $\nu^{(1)}$}
\label{fig:lbd-illus2}
\end{minipage}
\hspace{0.05cm}
\begin{minipage}{0.18\linewidth}
\resizebox{\linewidth}{\linewidth}{
\begin{tikzpicture}
\draw[step=0.5cm,gray,very thin] (-3.5,-3.5) grid (3.5 , 3.5);
\draw[mark=+, ultra thick,  mark size=4pt] plot coordinates {(-2.5, 3)} node[anchor=north west]{\large $1$};
\draw[mark=+, ultra thick, mark size=4pt] plot coordinates {(-3.0, 1)} node[anchor=south west]{\large $2$};
\draw[mark=+, ultra thick, mark size=4pt] plot coordinates {(2., -0.5)} node[anchor=south west]{\large $3$};
\draw[mark=+, ultra thick, mark size=4pt] plot coordinates {(1.5, -3)} node[anchor=south west]{\large $4$};
\draw[->, red, dashed, ultra thick](-3,1) -- (-2, 1);
%\draw[->, red, dashed, very thick](-2.5,3) -- (-3.5, 3);
%\draw[->, red, dashed, very thick](2,-0.5) -- (1, -0.5);
%\draw[->, red, dashed, very thick](1.5,-3) -- (2.5, -3);
%\draw[->, red, dashed](1,1) -- (2, 1);
\end{tikzpicture}}
\captionof{figure}{Instance $\nu^{(2)}$}
\label{fig:lbd-illus3}
\end{minipage}
\hspace{0.05cm}
\begin{minipage}{0.18\linewidth}
\resizebox{\linewidth}{\linewidth}{
\begin{tikzpicture}
\draw[step=0.5cm,gray,very thin] (-3.5,-3.5) grid (3.5 , 3.5);
\draw[mark=+, ultra thick,  mark size=4pt] plot coordinates {(-2.5, 3)} node[anchor=north west]{\large $1$};
\draw[mark=+, ultra thick, mark size=4pt] plot coordinates {(-3.0, 1)} node[anchor=south west]{\large $2$};
\draw[mark=+, ultra thick, mark size=4pt] plot coordinates {(2., -0.5)} node[anchor=south west]{\large $3$};
\draw[mark=+, ultra thick, mark size=4pt] plot coordinates {(1.5, -3)} node[anchor=south west]{\large $4$};
%\draw[->, red, dashed, very thick](-3,1) -- (-2, 1);
%\draw[->, red, dashed, very thick](-2.5,3) -- (-3.5, 3);
\draw[->, red, dashed, ultra thick](2,-0.5) -- (1, -0.5);
%\draw[->, red, dashed, very thick](1.5,-3) -- (2.5, -3);
\end{tikzpicture}}
\captionof{figure}{Instance $\nu^{(3)}$}
\label{fig:lbd-illus4}
\end{minipage}
\hspace{0.05cm}
\begin{minipage}{0.18\linewidth}
\resizebox{\linewidth}{\linewidth}{
\begin{tikzpicture}
\draw[step=0.5cm,gray,very thin] (-3.5,-3.5) grid (3.5 , 3.5);
\draw[mark=+, ultra thick,  mark size=4pt] plot coordinates {(-2.5, 3)} node[anchor=north west]{\large $1$};
\draw[mark=+, ultra thick, mark size=4pt] plot coordinates {(-3.0, 1)} node[anchor=south west]{\large $2$};
\draw[mark=+, ultra thick, mark size=4pt] plot coordinates {(2., -0.5)} node[anchor=south west]{\large $3$};
\draw[mark=+, ultra thick, mark size=4pt] plot coordinates {(1.5, -3)} node[anchor=south west]{\large $4$};
%\draw[->, red, dashed, very thick](-3,1) -- (-2, 1);
%\draw[->, red, dashed, very thick](-2.5,3) -- (-3.5, 3);
%\draw[->, red, dashed, very thick](2,-0.5) -- (1, -0.5);
\draw[->, red, dashed, ultra thick](1.5,-3) -- (2.5, -3);
\end{tikzpicture}}
\captionof{figure}{Instance $\nu^{(4)}$}
\label{fig:lbd-illus5}
\end{minipage}

\medskip

We propose below a larger class $\cB'$. Before stating our result we recall some notation. We recall that $e_1,\dots, e_D$ is the canonical basis of $\bR^D$. We identify a bandit $\nu$ with its means $\bm\Theta := (\vmu_1 \dots \vmu_K)^\T \in \bR^{K\times D}$. For any arm $i$, let $\Omega(i):= \{ j\in [K]: \vmu_i \prec \vmu_j\}$ the set of arms that dominate $i$ and $\Omega^\star(i):= \{j\in [K]: \m(i,j) = \Delta_i^\star\}$ the arms that ``maximally'' dominate $i$. Let $\Pi(i) := \{ i \in \cD(\nu) : \Omega^\star(i) = \{i\}\}$ the set of arms that are ``only maximally dominated by $i$''.
 For a bandit $\nu$ we let $\cD(\nu)$ or simply $\cD$ denote the set of sub-optimal arms. Let $\cB'$ be the set of instances such that : 
 \begin{eqnarray}
 	&\text{(1)}& \quad \forall i \in \cD(\nu) \text{ there exists } d_i \text{ such that } \vmu_i + \Delta_i e_{d_i} \text{ is Pareto optimal w.r.t } \{ \vmu_j: j\in [K]\setminus\{i\}\}\label{eq:lbd2-cond1},\\
 	&\text{(2)}& \quad \forall i \in \cS^\star(\nu) \text{ there exists } \undi \in \cD \text{ such that } \Pi(i) = \{ \undi \} \text{ and } \Omega(\undi) = \{ i\} \label{eq:lbd2-cond2},\\
 	&\text{(3)}& \quad \forall i, j \in \cS^\star(\nu), \M(i,j)\geq 3\max(\Delta_{\undi}, \Delta_{\underline j}), %\forall k: i\in \Omega^\star(k) \text{ or } j\in \Omega^\star(k).
 	\label{eq:lbd2-cond3}\\
 	&\text{(4)}& \quad \forall i, j, \text{ if } \vmu_i \nprec \vmu_j\text{ then } \M(i,j)\geq 3\max(\Delta_i, \Delta_j) %\forall k: i\in \Omega^\star(k) \text{ or } j\in \Omega^\star(k).
 	\label{eq:lbd2-cond4}
 \end{eqnarray}
 
% \begin{enumerate}
% 	\item $\forall i \in \cD(\nu)$ there exists $d_i$ such that $\vmu_i + \Delta_i e_{d_i}$ is Pareto optimal w.r.t $\{ \vmu_j: j\in [K]\setminus\{i\}\}$
% 	\item $\forall i \in \cS^\star(\nu)$ there exists $\undi \in \cD$ such that $\Pi(i) = \{ \undi \}$  and $\Omega(\undi) = \{ i\}$ 
% 	\item $\forall i, j \in \cS^\star(\nu), \M(i,j)\geq 3\Delta_k, \forall k: i\in \Omega^\star(k) \text{ or } j\in \Omega^\star(k) $
% \end{enumerate}
 Let  $\nu:= (\nu_1, \dots, \nu_K)$ be an instance whose means $\bm\Theta \in \cB$ and such that $\nu_i \sim \cN(\vmu_i, \sigma^2 I)$. For every $i\in [K]$ we define the alternative instance $ \nu^{(i)} := (\nu_1, \dots \nu_i^{(i)}, \dots, \nu_K)$ in which only the mean of arm $i$ is modified to:
\begin{equation}
	\vmu_i^{(i)} := \begin{cases}
		\vmu_i - 2\Delta_i
		e_{d_{\underline i}} & \text{ if } i \in \cS^\star(\nu) ,\\
		\vmu_i + 2\Delta_i e_{d_i}&\text{ else},
	\end{cases}
\end{equation}
where $e_1,\dots, e_D$ denotes the canonical basis of $\bR^D$ and $d_i :=\argmin_d [\mu_{i^\star}^d - \mu_i^d]$. With $\nu^{(0)}:=\nu$, we prove the following.

\begin{theorem}
Let $\bm\Theta:=(\vmu_1 \dots \vmu_K)^\T \in \cB'$ and $\nu=(\nu_1, \dots, \nu_K)$ where $\nu_i \sim \cN(\vmu_i, \sigma^2I)$. For any algorithm $\cA$, there exists $i\in \{0, \dots, K\}$ such that $ H(\nu^{(i)}) \leq  H(\nu)$ and 
$$ e_T^{\cA}(\nu^{(i)}) \geq \frac14 \exp\left( - \frac{2T}{\sigma^2 H(\nu^{(i)})}\right).$$ 
\end{theorem}

\begin{proof}
	The proof is similar to the proof of Theorem~\ref{thm:thm-lbd-consistent}. In particular we use the notation introduced therein: the superscript $i$ means that the quantity is computed on the means of bandit $\nu^{(i)}$ i.e $\{\vmu_i^{(1)}, \dots, \vmu^{(K)}\}$.
Proceeding similarly to to the proof of Theorem~\ref{thm:thm-lbd-consistent} we have :  there exists $i\in [K]$ such that $$\max(e_T(\nu), e_T(\nu^{(i)})) \geq  \frac14 \exp\left( -\frac{2 T}{\sigma^2H(\nu)}\right).$$
	
It remains to show that for each alternative bandit $\nu^{(i)}$, $\cS^\star(\nu^{(i)}) \neq \cS^\star(\nu)$ and $H(\nu^{(i)}) \leq H(\nu)$. To do so we analyze first an instance $\nu^{(i)}$ for $i\in \cD(\nu)$.

\paragraph{Step 1:} Assume $i\notin \cS^\star(\nu)$.

As  only arm $i$ has changed in $\nu^{(i)}$ w.r.t $\nu$ and $\vmu_i^{(i)} = \vmu_i +2\Delta_ie_{d_i}$, letting  $j\in \cD(\nu)\backslash\{i\}$ 
$j$ and $j^\star$ has not changed, thus $j \in \cD(\nu^{(i)})$ and
\begin{equation}
	\Delta_j^{(i)} \geq \Delta_j
\end{equation}

Let $j\in \cS^\star(\nu)$, we claim that $j\in \cS^\star(\nu^{(i)})$. Indeed, if $j\notin \cS^\star(\nu^{(i)})$ then 
$$  \vmu_j \prec  \vmu_i + 2\Delta_i e_{d_i}, $$ 
so $\M(j,i) \leq 2\Delta_i$. However $\M(j,i^\star) \leq \M(j,i)$ which yields 
$$ \M(j,i^\star) \leq 2\Delta_i,$$ 
%\todoc 
which is not possible by \eqref{eq:lbd2-cond3}. Therefore $j\in \cS^\star(\nu^{(i)})$ and we will show that $\Delta_j^{(i)} \geq \Delta_j$. 
Direct calculation yields  
\begin{equation}
\label{eq:lbd2-0}
	\delta_j^{+, (i)} = \min(\delta_j^+, \min(\M^i(i,j), \M^i(j,i)))
\end{equation}
and we have 
\begin{eqnarray}
\label{eq:lbd2-1}
	\M^i(i,j) = \left(\max_{d\neq d_i} [\mu_i^{d} - \mu_j^d]\right) \lor \left(\mu_i^{d_i} - \mu_j^{d_i} + 2\Delta_i \right).
\end{eqnarray}
If $\vmu_i \prec \vmu_j$ then since $\vmu_i + \Delta_ie_{d_i}$ is non-dominated w.r.t $\{ \vmu_k: k\in [K]\backslash\{i\}\}$ it follows that 
$$ \mu_i^{d_i} + \Delta_i \geq \mu_j^{d_i},$$
so $$ \left(\mu_i^{d_i} - \mu_j^{d_i} + 2\Delta_i \right) \geq \Delta_i,$$
which put back in \eqref{eq:lbd2-1} yields $\M^i(i, j)\geq \Delta_i$. And, as $\vmu_i \prec \vmu_j$, $\Delta_j \leq \Delta_i$ so 
$$ \M^i(i, j) \geq \Delta_j.$$
 In case $\vmu_i \nprec \vmu_j$, since $\M^i(i,j)\geq \M(i,j)$ and 
$$ \M(i,j)\geq \Delta_i$$
%\todoc {}
by condition \eqref{eq:lbd2-cond4} (as $\vmu_i \nprec \vmu_j$ and $j\in \cS^\star(\nu)$). Therefore in all cases 
\begin{equation}
\label{eq:lbd2-2}
	\M^i(i, j) \geq \max(\Delta_j, \Delta_i). 
\end{equation}
Similarly to \eqref{eq:lbd1-eq-00} $\M^i(i^\star, i)\geq \Delta_{i^\star}$ and for $j\neq i^\star$
\begin{eqnarray}
	\M^i(j, i) &\geq& \M(j, i) - 2\Delta_i,\\
	&\geq& \M(j, i^\star) - 2\Delta_i, \\
	&\geq& \frac13 \M(j, i^\star) \quad (\ref{eq:lbd2-cond3}) \\
	&\geq& \max(\Delta_{\underline j}, \Delta_i) \geq \Delta_j. \label{eq:lbd2-3}
\end{eqnarray}

Combining \eqref{eq:lbd2-0}, \eqref{eq:lbd2-2} and \eqref{eq:lbd2-3} yields 
$$ \delta_j^{+, (i)} \geq \Delta_j $$

On the other side, 

\begin{eqnarray*}
	\delta_j^{-, (i)} &:=& \min_{k\in \cD(\nu^{(i)})} \left[ \M^i(k, j)^+ + \Delta_k^{(i)} \right], \\
	&=& \min_{k\in \cD(\nu)\backslash\{i\}}  \left[ \M(k, j)^+ + \Delta_k^{(i)} \right], \\
	&\geq& \min_{k\in \cD(\nu)\backslash\{i\}}  \left[ \M(k, j)^+ + \Delta_k \right],\\
	&\geq& \delta_j^- \geq \Delta_j
\end{eqnarray*}

Therefore, for any $j\in \cS^\star(\nu)$, 
\begin{equation}
\label{eq:lbd2-4}
	\Delta_j^{(i)} = \min(\delta_j^{+, (i)}, \delta_j^{-, (i)}) \geq \Delta_j.
\end{equation}
So far we have proved that for any $j\neq i$, 
$$ \Delta_j^{(i)} \geq \Delta_j.$$
It remains to compute $\Delta_i^{(i)}$. Recall that $i \in \cD(\nu)$ but $i\in \cS^\star(\nu^{(i)})$ so the expression of its gap has completely changed and we have to check that it does not decrease.
We have 
$$ \Delta_i^{(i)} = \min(\delta_i^{+, (i)}, \delta_i^{-, (i)})$$
where 
\begin{eqnarray*}
	\delta_i^{+, (i)} &:=& \min_{k\in \cS^\star(\nu^{(i)})\backslash\{i\}} \min(\M^i(i,j), \M^i(j,i)),\\
	&=& \min_{k\in \cS^\star(\nu)} \min(\M^i(i,j), \M^i(j,i)), 
\end{eqnarray*}  
and by \eqref{eq:lbd2-2} and \eqref{eq:lbd2-3} $\min(\M^i(i,k), \M^i(k,i))\geq \Delta_i$ for any $k\in \cS^\star(\nu)$ so 
\begin{equation}
\label{eq:lbd2-5}
	\delta_i^{+, (i)} \geq \Delta_i.
\end{equation}
Let $k\in \cD(\nu)\backslash\{i\}$. If $\vmu_k \prec \vmu_i$ then $\Delta_i \leq \Delta_k$ and 
$$ (\M^i(k, i))^+ + \Delta_k^{(i)} \geq \Delta_{k}^{(i)} \geq \Delta_k$$
so $(\M^i(k, i))^+ + \Delta_k^{(i)} \geq \Delta_i.$  If $\vmu_k \nprec \vmu_i$
we have by \eqref{eq:lbd2-cond4} $\M(k, i)\geq 3\Delta_i$ and $\M^i(k,i)\geq \M(k,i)-2\Delta_i$
 so 
 $$ (\M^i(k, i))^+ + \Delta_k^{(i)} \geq \Delta_i.$$
  
 Therefore
 \begin{eqnarray}
 \delta_i^{-, (i)} &:=& \min_{k\in \cD(\nu^{(i)})} \left[(\M^i(k, i))^+ + \Delta_k^{(i)}\right],\\
 	\delta_i^{-, (i)} &=& \min_{k\in \cD(\nu) \backslash\{i\}} \left[(\M^i(k, i))^+ + \Delta_k^{(i)}\right],\\
 	&\geq& \Delta_i. \label{eq:lbd2-6}
 \end{eqnarray}
 Combining \eqref{eq:lbd2-5} and \eqref{eq:lbd2-6} yields 
 $$ \Delta_i^{(i)} \geq \Delta_i.$$
 Put together with what precedes yields for any arm $k$
 $$ \Delta_k^{(i)} \geq \Delta_k,$$
 therefore 
 $$ H(\nu^{(i)}) \leq H(\nu).$$
 
\paragraph{Step 2:} Assume $i\in \cS^\star(\nu)$.

 We have $\undi \in \Pi(i)$. If $j \in \cD(\nu) \cap (\Pi(i))^\complement$, as $j$ and $j^\star$ have not been modified from $\nu$ to $\nu^{(i)}$ and the only change has decreased $i$ in one coordinate, we have $j\in \cD(\nu^{(i)})$ and 
$$ \Delta_j^{(i)} = \Delta_i.$$  Direct calculation shows that : 
$\cS^\star(\nu^{(i)}) = \cS^\star(\nu) \cup\{\undi\}$
 and $\cD(\nu^{(i)}) = \cD(\nu)\backslash\{\undi\}$. Since $\Pi(i) = \{ \undi\}$, %and $\Delta_i = \Delta_{\underline i}$, 
 for any $k\in \cD(\nu) \backslash\{\undi\}$, 
$$ \Delta_k^{(i)} = \Delta_k.$$ 
The rest of the proof is similar to "Step 2'' in the proof of Theorem~\ref{thm:thm-lbd-consistent} but we recall the main arguments to make this part self-contained. 
For any $j \in \cS^\star(\nu^{(i)})\backslash\{ i, \underline i\}$.
 We have, 
 $$ \delta_j^{+, (i)} = \left( \min_{k\in \cS^\star(\nu) \backslash\{j, i\}} \left[ \min(\M(k, j), \M(j, k))\right]\right) \land \underbrace{\left(\min(\M(j, \underline i), \M(\underline i, j))\right)}_{\alpha_j} \land \underbrace{\left(\min(\M^i(j, i), \M^i(i, j))\right)}_{\beta_j}.$$  
 As $\vmu_{\undi} \nprec \vmu_j$ and $\vmu_j \nprec \vmu_{\undi}$, 
 $$ \alpha_j \geq \Delta_j. \quad (\text{by } \eqref{eq:lbd2-cond4})$$
 Similarly 
 \begin{equation}
  \label{eq:lbd2-9}
 \M^i(j, i)\geq \M(j,i) \geq \max(\Delta_j, \Delta_i)
 \end{equation}
where the last inequality follows since $i,j\in \cS^\star(\nu)$ and from \eqref{eq:lbd2-cond3}. 
 On the other side , 
 \begin{eqnarray}
 	\M^i(i, j) &\geq& \M(i,j) - 2\Delta_i,\\
 	&\geq& \frac13 \M(i, j) \quad (\eqref{eq:lbd2-cond3}),\\
 	&\geq& \max(\Delta_i, \Delta_j). \label{eq:lbd2-10}
 \end{eqnarray}
 Further noting that 
 $$ \left( \min_{k\in \cS^\star(\nu) \backslash\{j, i\}} \left[ \min(\M(k, j), \M(j, k))\right]\right) \geq \Delta_j $$ yields 
 \begin{equation}
 	\delta_j^{+, (i)} \geq \Delta_j. 
 \end{equation}
 Moreover, for any $k\in \cD(\nu^{(i)}) = \cD(\nu)\backslash\{ \undi\}$, 
 $$ (\M(k,j))^+ + \Delta_k^{(i)} = \M(k,j)^+ + \Delta_{k}\geq \Delta_j.$$
 Therefore, for any $j\in \cS^\star(\nu^{(i)})\backslash\{ i, \underline i\} $
 \begin{equation}
 	\Delta_j^{(i)} :=\min(\delta_j^{+, (i)}, \delta_j^{-, (i)}) \geq \Delta_j. 
 \end{equation}
 To summarize, we have proved that for any $j\in [K]\backslash\{ i, \undi\}$, 
 $$ \Delta_j^{(i)} \geq \Delta_j.$$
 We now analyze $\Delta_{i}^{(i)}.$
 
We have \begin{eqnarray}
 	\delta_i^{+, (i)} &=& \left( \min_{k\in \cS^\star(\nu) \backslash\{i\}} \left[ \min(\M^i(k, i), \M^i(i, k))\right]\right) \land \underbrace{\left(\min(\M(i, \underline i), \M(\underline i, i))\right)}_{\beta_{\underline i}}.%,\\
% 	&=& \left(\min_{k\in \cS^\star(\nu) \backslash\{i\}} \beta_k \right) \land \Delta_i,\\
% 	&=& \Delta_i.
 \end{eqnarray}
% \todoc{ref preuve}
 We further note by direct calculation $\M^i(\underline i, i) = \Delta_i$ and $\M^i(i, \underline i) \geq \Delta_i$. 
 Combining with \eqref{eq:lbd2-9} and \eqref{eq:lbd2-10} yields 
 $$ \delta_i^{+, (i)}  \geq \Delta_i.$$
 Let $k \in \cD(\nu) \backslash\{ \undi \}$. If $\vmu_k \prec \vmu_i$ then $\Delta_i \leq \Delta_k$ and 
 $$ (\M^i(k, i))^+ + \Delta_k^{(i)} \geq \Delta_k^{(i)} \geq \Delta_k$$
 so $$ (\M^i(k, i))^+ + \Delta_k^{(i)} \geq \Delta_i.$$ Similarly, if $\vmu_k \nprec \vmu_i$ as in "Step 1'' we prove 
 $$ (\M^i(k, i))^+ + \Delta_k^{(i)} \geq \Delta_i.$$
Thus $\delta_i^{-, (i)} \geq \Delta_i$, hence 
\begin{equation}
	\Delta_i^{(i)} \geq \Delta_i. 
\end{equation}
 The computation of $\Delta_{\undi}^{(i)}$ proceeds identically to that of $\Delta_i^{(i)}$ to prove that 
 $$ \Delta_{\undi}^{(i)} \geq \Delta_{\undi}.$$
 Therefore, for any $k\in [K]$, 
 $$ \Delta_k^{(i)} \geq \Delta_k,$$
 which yields $H(\nu^{(i)}) \leq H(\nu).$ 
  
%\todoc{plto example d'instance de B'}
%\todoc{prove that no optimal is dominated by $i$}
\end{proof}

\section{WHEN THE COMPLEXITY TERM $H$ IS KNOWN}
\label{sec:pucbe}
In this section we analyze \apefb{} which has been used in the experiments reported in the main paper. This algorithm is similar to UCB-E \citep{audibert_best_2010} and UGapEb \citep{gabillon_best_2012}. The idea is to use the sampling rule APE of \cite{kone2023adaptive} (designed for fixed-confidence PSI) for $T$ rounds and analyze the theoretical guarantees of the resulting algorithm. We recall that at each round $t$, $T_i(t)$ is the number of times arm $i$ has been pulled up to time $t$ and $\vmuh_i(t)$ is the empirical estimate of $\vmu_i$ based on the $T_i(t)$ samples collected so far and $S(t)$ is the empirical Pareto set at time $t$. For any arms $i,j$,
\begin{eqnarray*}
	\Mh(i, j;t) &:=& \max_d \left[\muh_i^d(t) - \muh_j^d(t)\right] \text{ and }\\
	\mh(i,j; t) &:=& \min_d \left[ \muh_j^d(t) - \muh_i^d(t)\right] = -\Mh(i,j;t). 
\end{eqnarray*}
 Let $a\geq 0$ and define for any arm $i$, $$\beta_i(t) := \frac{2}{5}\sqrt{\frac{a}{T_i(t)}}.$$
Let 
\begin{eqnarray*}
	Z_1(t) &:=& \min_{i \in S(t) } \min_{j\in S(t) \backslash \{i\}}\left[\M(i,j; t) - \beta_i(t) - \beta_j(t)\right], \\
	Z_2(t) &:=& \min_{i \in S(t)^\complement} \max_{j\neq i} \left[ \m(i,j; t) - \beta_i(t) - \beta_j(t)\right],
\end{eqnarray*} 
and we use by convention $\min_\emptyset = \infty$.  Letting $$\OPT(t) := \left\{ i \in [K]: \forall j\neq i, \Mh(i,j; t) - \beta_i(t) - \beta_j(t) >0 \right\}, $$
we recall that  APE\citep{kone2023adaptive} defines two arms $b_t$ and $c_t$ as follows: 
 \begin{equation}
 \label{eq:def_bt}
 	b_t := \argmax_{i \notin \OPT(t)} \min_{j\neq i}\left[ \Mh(i,j; t) + \beta_i(t) + \beta_j(t)\right], 
 \end{equation}
 and when $\OPT(t) = [K]$ (which could often occur when $\cS^\star = [K]$), define 
 $$ b_t := \argmin_{i \in [K]}\min_{j\neq i}\left[ \Mh(i,j;t) -\beta_i(t) - \beta_j(t)\right].$$
 
Finally, we define \begin{equation}
 \label{eq:def_ct}
 	c_t := \argmin_{j\neq b_t} \left[ \Mh(b_t,j; t) - \beta_j(t) \right], 
 \end{equation}
 and $a_t$, the least explored among $b_t$ and $c_t$. Then at each round we pull $a_t$. 
 % (see the full description of \apefb).  
\medskip 

\begin{algorithm}[H]
\caption{APE-FB}
\label{alg:pucbe}
\KwResult{Pareto set}
\KwData{parameter $a\geq 0$}
\Input{pull each arm once and set $T_i(K) = 1$ for each arm $i$} 
 \For{$t=K+1, \dots, T$}{
 
  Compute $b_t$ \eqref{eq:def_bt} and $c_t$ \eqref{eq:def_ct}

  Sample arm $a_t$ the least explored among $b_t$ and $c_t$
  
     }
 \Output{$\widehat S_T = S(T)$}
\end{algorithm}

\medskip 

We below state an upper-bound on the probability of error of \apefb{}. 
\begin{theorem}
\label{thm:ape-b}
Let $T\geq K$ and $\nu$ a bandit with $\sigma$-subgaussian marginals . Let $0\leq a\leq \frac{25}{36} \frac{T-K}{H(\nu)}$. The probability of error of \apefb{} run with parameter $a$ satisfies 
$$ e_T^\text{APE}(\nu) \leq 2(1 + \log(T))DK\exp\left(-\frac{a}{100\sigma^2}\right).$$ 
\end{theorem}

For $a=\frac{25}{36}\frac{T-K}{H(\nu)}$, we have 
$$ e_T^\text{APE}(\nu) \leq 2(1 + \log(T))DK\exp\left(-\frac{T-K}{144\sigma^2 H(\nu)}\right).$$
Compared to EGE-SH and EGE-SR, the exponent in the upper-bound is larger by an order $\log(K)$ but outside of the exponent there is an extra multiplicative $\log(T)$ term due to time uniform concentration. As this term grows very slowly with $T$ it shouldn't be too high w.r.t $K\lvert \cS^\star\lvert$ or $\log(K)\lvert \cS^\star\lvert$ and the exponential term becomes quickly predominant when $T$ is large. Thus \apefb{} improves upon both EGE-SR and EGE-SH but requires $a$ to be tuned optimally. This is an important limitation as in practice $H(\nu)$ is not known and the results reported in the main suggest that the performance of \apefb{} heavily rely on this proper tuning. On the contrary, both EGE-SR and EGE-SH are parameter-free.  

\begin{proof}
	For any arm $i$ and any component $d$, let $I_i^d := \left[ \mu_i^d - \frac12 \beta_i(t) ;  \mu_i^d + \frac12 \beta_i(t) \right]$.  
	We define the event 
	$$\cE := \left\{  \forall i \in [K],\; \forall t\leq T,\; \forall d \in [D],\; \muh_i^d(t) \in I_i^d \right\}.$$
We assume the event $\cE$ holds and we show that then $S(T) = \cS^\star$. 
Let us introduce  $$ \tau := \inf \left\{K\leq t\leq T: Z_1(t)> 0 \text{ and } Z_2(t)>0 \right\}.$$
Note that from the definition, $\tau$ can be infinite as the $\inf$ of an empty set but we will further show that indeed $\tau\leq T$. We claim that on the event $\cE$, for any $\tau \leq t\leq T$, $S(t) = \cS^\star$. We have at time $t=\tau$, $Z_1(t)\geq 0$ and $Z_2(t)\geq0$. Letting $i \in S(t)$, for any $j\neq i$, there exists $d_j$ such that 
$$ \muh_{j}^{d_j}(\tau) + \beta_{j}(\tau) < \muh_{i}^{d_j}(\tau) - \beta_{i}(\tau),$$
that is 
$$ \theta_{j}^{d_j} + \frac12 \beta_{j}(\tau) < \theta_{i}^{d_j} - \frac12\beta_{i}(\tau),$$ which yields (as $\beta$ is a decreasing function) for any $t\geq \tau$, 
$$ \theta_{j}^{d_j} + \frac12 \beta_{j}(t) <  \theta_{i}^{d_j} - \frac12 \beta_{i}(t),$$ therefore, for any $t\geq \tau$, $\muh_j^{d_j}(t) < \muh_{i}^{d_j}(t)$. Said otherwise, on the event $\cE$, if $i \in S(\tau)$ then for any $t\geq \tau$, $i\in S(t)$. Using an identical reasoning for any arm $i\in S(\tau)^\complement$ also yields that on the event $\cE$, if $i\in S(\tau)^\complement$ then for any $t\geq \tau$, $i\in S(t)^\complement$. At this point, we have proved that under the event $\cE$, if $\tau\leq T$ then $S(T) = \cS^\star$. 

\medskip 

Showing that on the event $\cE$, $\tau \leq T$ will conclude the proof. We proceed by contradiction and we assume $\tau > T$. First, note that $\tau>T$ imply that for any $t\leq T$, $\OPT(t)\neq [K]$ (otherwise we would have $Z_1(t)>0$ and since $S(t)^\complement$ would be empty, $Z_2(t)=\infty$). Therefore, the two candidate arms $b_t$ and $c_t$ are always defined as in \cite{kone2023adaptive}. 

Introducing for any pair of arms $i,j$
$$ U_{i,j}^d(t) = \mu_i^d - \mu_j^d + \frac12 \beta_i(t) + \frac12\beta_j(t) \text{ and } L_{i,j}^d(t) = \mu_i^d - \mu_j^d - \frac12 \beta_i(t) - \frac12 \beta_j(t), $$ we deduce the following result from Lemma~1 of \cite{kone2023adaptive}. 

\begin{lemma}
\label{lem:conc_beta}
	On the event $\cE$, at any round $t$ and for any pair $i,j$ we have 
	$$ \lvert \M(i,j) - \Mh(i,j; t)\lvert \leq \frac12 \left( \beta_i(t) + \beta_j(t)\right) \textrm{ and } \lvert \m(i,j) - \mh(i,j; t)\lvert \leq \frac12 \left( \beta_i(t) + \beta_j(t)\right).$$
\end{lemma}

We state below a lemma adapted from \cite{kone2023adaptive} which is important to upper-bound the sub-optimality gap of each arm by the confidence bonuses. 
%% maybe commenter ces deux lemmes 
%\begin{lemma}[Lemma~8 of \cite{kone2023adaptive}]
%If $t<\tau$ then for any $j\in [K]$, $\mh(b_t, j; t)\leq \beta_{b_t}(t) + \beta_{j}(t)$. 
%\end{lemma}
%\begin{lemma}[Lemma~9 of \cite{kone2023adaptive}] 
%For any Pareto optimal arm $i$, $\Delta_i \leq \min_{j\neq i} \M(i, j)$.
%\end{lemma}
The following lemma's proof is essentially identical to Lemma~4 of \cite{kone2023adaptive}. 
\begin{lemma}
\label{lem:fund_apeb}
	Assume $\cE$ holds. For any $t<\tau$, $\Delta_{a_t} < 3\beta_{a_t}(t)$. 
\end{lemma}
Since by assumption $\tau>T$, by Lemma~\ref{lem:fund_apeb} we have for any $t\leq T$, 
\begin{equation}
	\Delta_{a_t} < 3 \beta_{a_t}(t). 
\end{equation}
For any arm $i$, let $t_i$ be the last time $i$ was pulled. We have $T_i(T) = 1 + T_i(t_i)$. Since $i$ has been pulled at time $t_i$ and $t_i<\tau$ we have $a_{t_i} = i$ and 
$$ \Delta_i < 3 \beta_i(t_i), $$ that is $$ T_i(T) - 1 < \frac{36}{25}\frac{1}{\Delta_i^2 a},$$ therefore  

$$T - K = \sum_{i=1}^K (T_i(T)- 1) < \frac{36}{25} \frac{a}{H(\nu)}, $$
which is impossible for $a \leq \frac{25}{36}\frac{T-K}{H(\nu)}$. All put together, we have proved by contradiction that on the event $\cE$, $\tau \leq T$ and by what precedes, the recommendation of \apefb{} is correct : $S(T) = S(\tau) = \cS^\star$. The upper-bound on $e_T^\text{APE}(\nu)$ then follows by upper-bounding $\bP(\bar \cE)$ with a direct application of Hoeffding's maximal inequality (see e.g section A.2 of \cite{locatelli_optimal_2016}). 
\end{proof}

\begin{remark}
	For $D=1$, \apefb{} slightly differs from UCB-E \cite{audibert_best_2010} as it does not sample the arm with maximum upper confidence bound but the least sampled among the two arms with the largest upper-confidence bound. Moreover, in case $D=1$,  Theorem~\ref{thm:ape-b} recovers the result of \cite{audibert_best_2010} (Theorem~1 therein) with a slightly different algorithm. 
\end{remark}
\section{TECHNICAL LEMMAS}
\label{sec:tech_lemmas}
In this section we prove the lemma that allow to remove the explicit dependency on $\cS^\star$ in the expression of the sub-optimality gaps.  We use the following lemma which is taken from \cite{kone2023adaptive}. 
\begin{lemma}[Lemma~10 of \cite{kone2023adaptive}]
\label{lem:existence_dom}
For any sub-optimal arm $a$, there exists a Pareto optimal arm $a^\star$ such that $\vmu_a \prec \vmu_{a^\star}$ and $\Delta_{a} = \m(a, a^\star)>0$. Moreover, For any $i\in \bA\setminus \cS^\star$, $j \in \cS^\star$ 
\begin{enumerate}[i)]
    \item $\max_{k\in \cS^\star} \m(i,k) = \max_{k\in \bA} \m(i,k)$, 
    \item If $i \in \argmin_{k\in \bA \setminus \{j\}} \M(j,k)$ then $j$ is the unique arm such that $\vmu_i \prec \vmu_j$. 
\end{enumerate}
\end{lemma}
We now prove Lemma~\ref{lem:lem-gap-redef}. 
\lemGapRedef*
\begin{proof}
For sub-optimal arms, the result follows from Lemma~\ref{lem:existence_dom}. It remains to prove the equality for sub-optimal arms. 
By definition, for an optimal arm $i$, we have 
$$  \Delta_ i = \min(\delta_i^+, \delta_i^-),$$ 
where $$\delta_i^+:= \min_{j\in \cS^\star \setminus \{i\}} \min(\M(i,j), \M(j,i)) \ \text{and} \ \ \ \delta_i^-:= \min_{j\in [K]\setminus \cS^\star}\left[(\M(j,i))^+ + \Delta^\star_{j}\right].$$
For any optimal arm $i$, $\Delta_i^\star \leq 0$ (by direct calculation), so introducing  
\begin{eqnarray*}
	\delta_i^{-'} := \min(\delta_i^-, \min_{j\in \cS^\star\backslash\{i\}} \M(j,i)),
\end{eqnarray*}
we have 
$$ \delta_i^{-'} = \min_{j\neq i} [\M(j,i)^+ + (\Delta_j^\star)^+]$$
Then, if 
\begin{equation}
\label{eq:cond1appx}
\min_{j\in \cS^\star \setminus \{i\}} \M(i,j) = \min_{j\neq i} \M(i,j),	
\end{equation}
holds, the result simply follows as for any optimal arm $i$,
\begin{eqnarray*}
\Delta_i = \min(\delta_i^+, \delta_i^-) &=& \min(\min_{j\in \cS^\star\backslash\{i\}} \M(i,j), \delta_i^{-'}),\\
&=& \min(\min_{j\neq i} \M(i,j), \delta_i^{-'}),\\
&=& \min_{j\neq i} [\M(i,j) \land (\M(j,i)^+ + (\Delta_j^\star)^+)],\\
&=& \delta_i^\star. 
\end{eqnarray*}
In the sequel, assume \eqref{eq:cond1appx} does not hold, that is assume 
$$ \min_{j\in \cS^\star \setminus \{i\}} \M(i,j) > \min_{j\neq i} \M(i,j).$$
From Lemma \ref{lem:existence_dom} we know that in this case, there exists 
a sub-optimal arm $k$ such that $i$ is the unique arm dominating $k$ and $$\Delta_k^\star = \m(k,i)\quad \text{and}\quad  \min_{j\neq i}\M(i,j) = \M(i,k).$$
Thus, we have 

\begin{eqnarray}
	\min_{j\neq i } \M(i,j) &=& \M(i,k), \\
	&\geq& \m(k,i) = \Delta^\star_k,\\
	&\geq& \delta_i^{-'} \quad \text{(since $i$ dominates $k$, $\M(k, i)^+=0$)}. \label{eq:tech1}
\end{eqnarray}
On the other side, as $\vmu_k \prec \vmu_i$, using the definition of $\Delta_i$ in particular $\delta_i^-$ directly yields 
\begin{eqnarray}
\Delta_i &\leq& \Delta_k^\star = 	\m(k, i)\\ &\leq& \M(i, k) = \min_{j\neq i} \M(i,j)\\
&<& \min_{j\in \cS^\star \backslash\{i\}} \M(i,j)\label{eq:tech2}.
\end{eqnarray}
We recall  that 
\begin{eqnarray*}
\Delta_i = \min(\delta_i^{-'}, \min_{j\in \cS^\star \backslash\{i\}} \M(i,j)),\end{eqnarray*}
which combined with \eqref{eq:tech2} yields 
$$ \Delta_i = \delta_i^{-'}$$
and further combining with \eqref{eq:tech1} yields 
\begin{eqnarray*}
\Delta_i = \delta_i^{-'} &=& \min( \min_{j\neq i } \M(i,j), \delta_i^{-'}) 	\\
&=& \delta_i^\star,
\end{eqnarray*}
which concludes the proof. 
\end{proof}

\section{GEOMETRIC $R$-ROUND ALLOCATION}
\label{sec:alt_alloc}
In this section, we mention an alternative allocation scheme, which could be coupled with \ege{} and we derive an upper-bound on the mis-identification probability of the resulting algorithm. 
\cite{karpov2022collaborative} proposed an $R$-round allocation for fixed-budget BAI. Their allocation scheme follows a geometric grid. In our notation, their allocation is as follows: given $R \in \bN^\star$,
$$ \alpha_0 = 0 \textrm{ and }  \forall r \in \{1, \dots, R\}, \alpha_r = \left\lfloor \frac{T}{R}\cdot \frac{K^{r/R}}{K^{1+1/R}} \right\rfloor $$ and at round $r$,
$t_r := \alpha_r - \alpha_{r-1}$. The arm elimination schedule is 
$$ \forall r\in [R+1], \lambda_r = \left\lfloor \frac{K}{K^{(r-1)/R}}\right\rfloor.$$
Direct calculation shows that as defined, $\bm\lambda$ and $\bm t$ satisfy the conditions \eqref{eq:cond1} and \eqref{eq:cond2}. We denote by EGE-GG the specialization of EGE for the $\bm\lambda, \bm t$ aforementioned. We deduce the following result from Theorem~\ref{thm:main-res}. 
\begin{proposition}
	Let $R\in \bN^\star$. For any $\sigma$-subgaussian multi-variate bandit, the mis-identification probability of EGE-GG satisfies 
	$$ e_T^\text{GG}(\nu) \leq 2(K-1) \lvert \cS^\star\lvert R D\exp\left( - \frac{T}{288\sigma^2 RK^{1/R} H_2(\nu)}\right).$$ 
\end{proposition}
\begin{proof}
We assume $T\geq 2RK$. We have 
	\begin{eqnarray*}
		\widetilde T^{R,\bm t^{\textrm{GG}}, \bm\lambda^{\textrm{{GG}}}}(\nu)  &:=& \min_{r\in [R]} \left(\sum_{s=1}^r t_s\right)\Delta_{(\lambda_{r+1} +1)}^2,\\
		&=& \min_{r\in [R]} [\alpha_r \Delta_{(\lambda_{r+1} +1)}^2], \\
		&=& \min_{r \in [R] } \left[\alpha_r(\lambda_{r+1} +1) \frac{\Delta_{(\lambda_{r+1} +1)}^2}{(\lambda_{r+1} +1)}\right],\\
		&\geq& \frac{T}{2RK^{1/R}} \frac{1}{\max_{r\in [R]} [(\lambda_{r+1} +1)\Delta_{(\lambda_{r+1} +1)}^{-2}] }
			\\
			&\geq& \frac{T}{2RK^{1/R}} \frac{1}{H_2(\nu)}. 
			\end{eqnarray*}
	The results follows by Theorem~\ref{thm:main-res}. 
\end{proof}
Similar derivation can be done for any allocation scheme that satisfy 
\eqref{eq:cond1} and \eqref{eq:cond2} resulting in different instanciations of \ege{}.       
\section{IMPLEMENTATION DETAILS AND ADDITIONAL EXPERIMENTS}
\label{sec:supp_mat}
In this section, we give additional details about the experiments and we report additional experimental results.
\subsection{Implementation details and computational complexity}
\label{sec:implementation}
We discuss implementation details and give additional information on the datasets for reproducibility. 
\subsubsection{Implementation details} 
\paragraph{Setup} We have implemented the algorithms mainly in \texttt{C++17} interfaced with \texttt{python 3.10} through the \texttt{cython3} package. The experiments are run on an ARM64 8GB RAM/8 core/256GB disk storage computer. 

\paragraph{Datasets}: The COV-BOOST datasets can be found in Appendix~I of \cite{kone2023adaptive}. A link to download the  SNW dataset is given in \cite{zuluaga_active_2013}. We plot below 2 of the 4 synthetic instances used in the main paper. The means of the instance in the fourth experiment in dimension 10 are given as a \texttt{numpy} data file in the supplementary material. 

\begin{figure*}[hb]
\centering
 \begin{minipage}{\textwidth}
      \centering 
      \begin{minipage}{0.3\linewidth}
          \begin{figure}[H]
              \includegraphics[width=\linewidth]{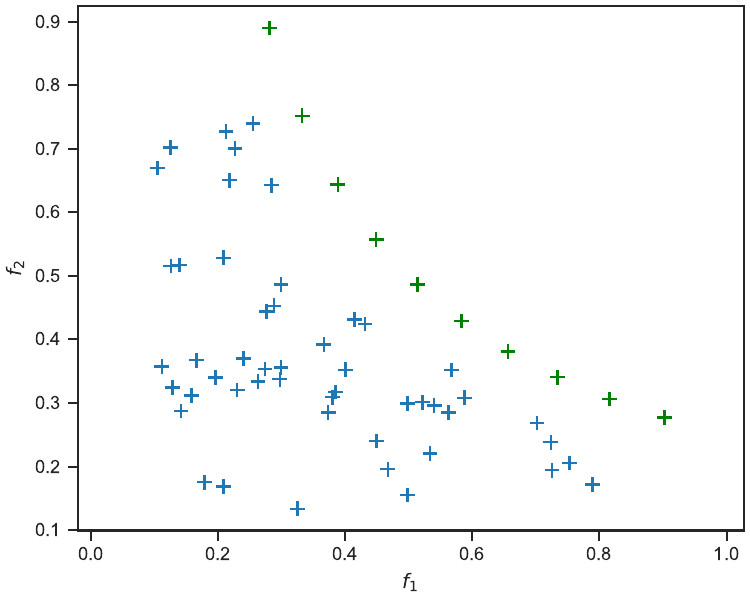}
              \caption{Synthetic Experiment 1: Group of  arms on a convex Pareto set.}
          \end{figure}
      \end{minipage}
      \hspace{0.5cm}
            \begin{minipage}{0.3\linewidth}
          \begin{figure}[H]
              \includegraphics[width=\linewidth]{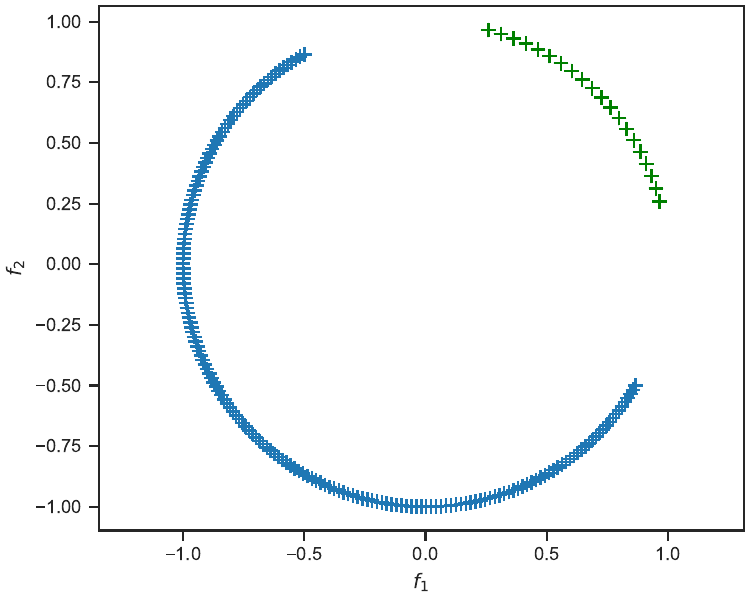}
              \caption{Synthetic Experiment 3: $K=200$ arms on the unit circle.}
          \end{figure}
      \end{minipage}
 \hspace{0.5cm} 
     \begin{minipage}{0.3\linewidth}
     \vspace{-0.85cm}
     \begin{table}[H]
	\centering
	\begin{tabular}{|c|c|c|c|}
		\hline
 & $K$ & $D$ & $\lvert \cS^\star \lvert$\\
\hline\hline 
\!\!COV-BOOST\!\! & 20 & 3 & \!\!2\!\!\\\hline \hline 
SNW& 206&2&5 \\\hline\hline 
\!\!\!\!Exp. 1\!\!\!\! & 60 & 2 & 10
\\\hline\hline 
\!\!\!\!Exp. 2\!\!\!\! & 10 & 2 & \!\!2\!\!\\\hline\hline 
\!\!\!\!Exp. 3\!\!\!\! & 200& 2& \!\!20\!\!\\\hline \hline 
\!\!\!\!Exp. 4\!\!\!\! & 50&10&\!\!18\!\!\\\hline 
	\end{tabular}
	\caption{Summary of the bandit instances used in the main paper.}
		\label{tab:datasets}
\end{table}

                \end{minipage}
      \end{minipage}
   \end{figure*}

We summarize in Table \ref{tab:datasets} the instances used in the main paper.

\subsubsection{Computational complexity}
\paragraph{Time and memory complexity}
The time and memory complexity of our implementations are dominated by the computation and the storage of $\M(i,j; r)$ and $\m(i,j;r)$ at each round $r$ for all the pairs of active arms. The time complexity of and $R$-rounds implementation of \ege{} is $\cO(RK^2D)$ and the memory complexity is $\cO(K^2)$.

\paragraph{Runtime and scaling} We report below the average runtime of our algorithms for the experiments reported in the main paper. 
We report on Fig.\ref{fig:exec-time10} and Fig.\ref{fig:exec-time10} the runtime versus the number of arms for EGE-SR/SH with varying $K$ and $T=1000$. We average the runtime over 2000 runs. 
We observe that both algorithms can be used for applications on large scale datasets with a reasonable runtime.  
\begin{figure*}[ht]
\centering
 \begin{minipage}{0.95\textwidth}
      \centering 
      \begin{minipage}{0.47\linewidth}
          \begin{figure}[H]
              \includegraphics[width=\linewidth]{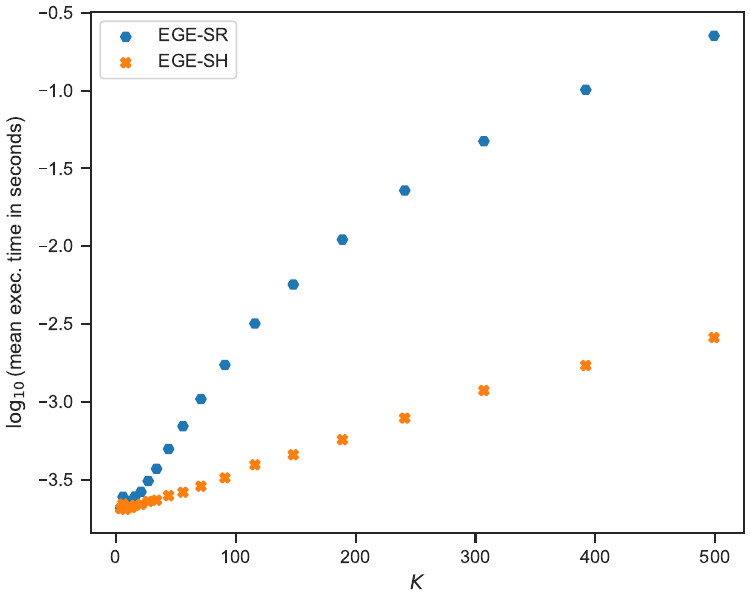}
              \caption{$\log_{10}$ of the mean execution time in seconds averaged over $2000$ trials for an instance with $D=2.$}
              \label{fig:exec-time2}
          \end{figure}
      \end{minipage}
 \hspace{0.5cm} 
            \begin{minipage}{0.47\linewidth}
          \begin{figure}[H]
              \includegraphics[width=\linewidth]{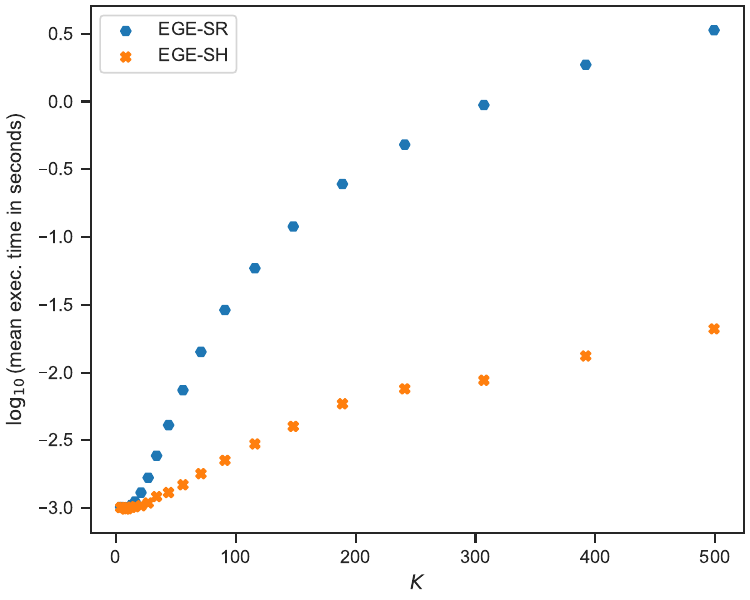}
              \caption{$\log_{10}$ of the mean execution time in seconds averaged over $2000$ trials for an instance with $D=20.$}
              \label{fig:exec-time10}
          \end{figure}
      \end{minipage}
      \end{minipage}
   \end{figure*}

\subsection{Additional experiments}
\label{sec:more_exp}
We report results of additional experiments in particular for the $k$-relaxation. We try different values of $k$ and we compute the loss or mis-identification error averaged over the trials. We plot the average loss versus the budget for each value of $k$. For additional experiments on PSI, the protocol is as described in the main paper. 
\subsubsection{Experiments on PSI-$k$}
We report the experimental result of \egesrk{} on two bandit instances with a large number of arms. For each bandit instance we test different values of $k$ and we report $\log_{10}$ of the loss averaged over 4000 independent trials. Recall that given $k$ and $T$ the PSI-$k$ loss is 
\begin{eqnarray*}
	\cL(\widehat S_T, k) := \begin{cases}
		\ind\{\widehat S_T \subset \cS^\star\} & \text{if} \quad  \lvert \widehat S_T \lvert = k, \\
		\ind\{\widehat S_T = \cS^\star\} & \text{else}. 
	\end{cases}
\end{eqnarray*}
%We also report the expected stopping time of \egesrk\! w.r.t the budget. 
\begin{figure*}[ht]
\centering
 \begin{minipage}{0.95\textwidth}
      \centering 
      \begin{minipage}{0.47\linewidth}
          \begin{figure}[H]
              \includegraphics[width=\linewidth]{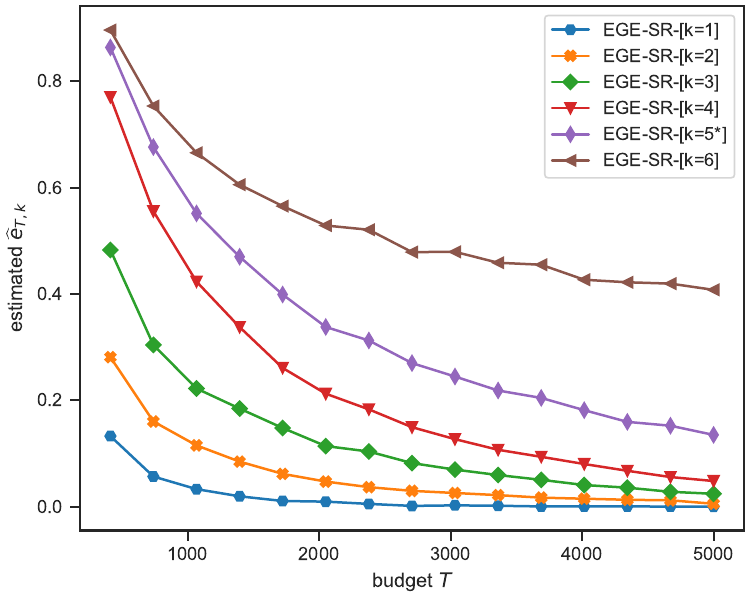}
              \caption{Estimated PSI-$k$ loss for different values of $k$ on the SNW dataset.}
              \label{fig:srk-fig-snw}
          \end{figure}
      \end{minipage}
 \hspace{0.5cm} 
            \begin{minipage}{0.47\linewidth}
          \begin{figure}[H]
              \includegraphics[width=\linewidth]{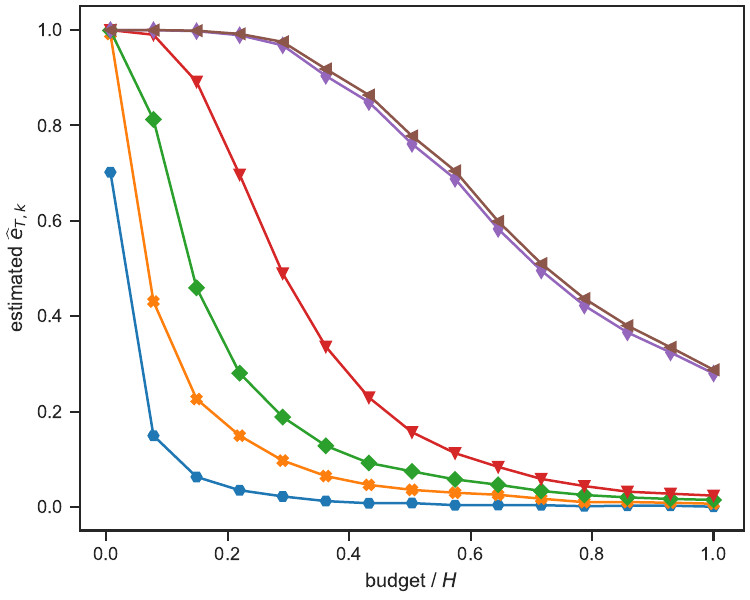}
              \caption{Estimated PSI-$k$ loss for different values of $k$ on Exp.1 (convex Pareto set)}
              \label{fig:srk-fig-2}
          \end{figure}
      \end{minipage}
      \end{minipage}
   \end{figure*}
   
We observe (Fig. \ref{fig:srk-fig-snw} and Fig.\ref{fig:srk-fig-2}) that the loss (the probability of error of \egesrk{}) decreases exponentially fast with the budget. The asterisk indicates the true size of the Pareto set of the instance. We remark that the smaller $k$ the smaller the loss, which is expected as the quantity $H_2(\nu)^{(k)}$ (Theorem~\ref{thm:main-amp}) increases with $k$. We also note that for $k=\lvert \cS^\star\lvert$, the loss can be smaller than when $k>\lvert \cS^\star\lvert$. An equivalent remark (in terms of sample complexity) has been made by \cite{kone2023adaptive} for the PSI-$k$ problem in fixed-confidence. Indeed for $k=\lvert \cS^\star\lvert$ the PSI-$k$ problem is not exactly PSI. As in PSI-$k$ we stop as soon as $k$ optimal arms have been found, in PSI, on the same run, we would continue until all the arms have been classified; thus we could make mistakes on theses additional steps. Finally, note that the expected loss of PSI-$k$ is the same as PSI for any $k>\lvert \cS^\star\lvert$. 

%and the expected fraction of samples used before stopping. 
% todo hypervolume si temps 
% commenter resultats 
\newcommand{\hv}{\text{HV}}
\paragraph{Hyper-volume metric}: The hyper-volume metric (see e.g \cite{daulton}) measures the region dominated by a set of points. Given a reference point $\bm r$ and a set $S\subset \bR^D$ the hyper-volume of $S$ is 
$$ \text{HV}(S) := \lambda\left(\bigcup_{\bm x \in S} [\bm r, \bm x]\right),$$ 
where $\lambda$ is the Lebesgue measure on $\bR^D$ and $[\bm r, \bm x]$ is the hyper-rectangle $[r^1, x^1]\times \dots \times [r^D, x^D]$. If $S'\subset S$ then $\hv(S)\geq \hv(S')$. More importantly, the hyper-volume of $S$ is equal to the hyper-volume of its Pareto set. We use this indicator to measure the quality of the set returned by \egesrk{}. We report the average value of $$ \alpha_{\hv}:= \frac{\hv(\widehat S_T)}{\hv(\cS^\star)}$$ over the trials for different values of $k$ and $T$ fixed. The larger this hyper-volume fraction the more a set $\widehat S_T\subset \cS^\star$ covers $\cS^\star$ in term of dominated region and thus is representative of $\cS^\star$. Note that $S = [K]$ also has maximal hyper-volume indicator, therefore we need to couple $\alpha_{\hv}$ with $e_{T, k}(\nu)$ (reported in Fig.\ref{fig:sorting_net} in the main and Fig.\ref{fig:srk-fig-snw-hpv}) to properly interpret the quality of the returned set. 

We observe on Fig.\ref{fig:srk-fig-snw-hpv} that the estimated value of $\bE[\alpha_\hv]$ is nearly $0.85$ for $k=1$ and it naturally increases with $k$. In Fig.\ref{fig:srk-fig-2-hpv}, for $k=1$, the arm returned by \egesrk{} covers nearly half of the region dominated by the Pareto set. These observations suggest that the arms with the highest hyper-volume contribution are easier to identify as optimal hence they will be the first added to $B_r$. 
 \begin{figure*}[ht]
\centering
 \begin{minipage}{0.95\textwidth}
      \centering 
      \begin{minipage}{0.47\linewidth}
          \begin{figure}[H]
              \includegraphics[width=\linewidth]{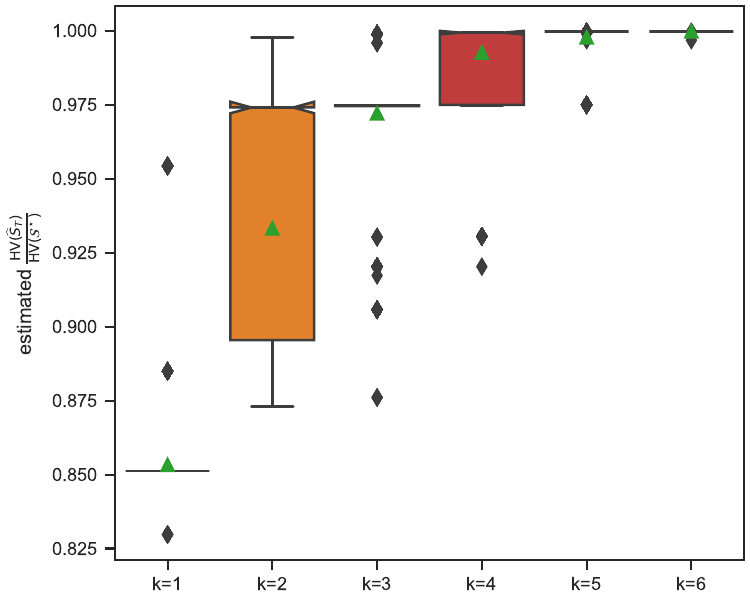}
              \caption{Hyper-volume fraction of the returned set for $T=5000$ on the SNW dataset.}
              \label{fig:srk-fig-snw-hpv}
          \end{figure}
      \end{minipage}
 \hspace{0.5cm} 
            \begin{minipage}{0.47\linewidth}
          \begin{figure}[H]
              \includegraphics[width=\linewidth]{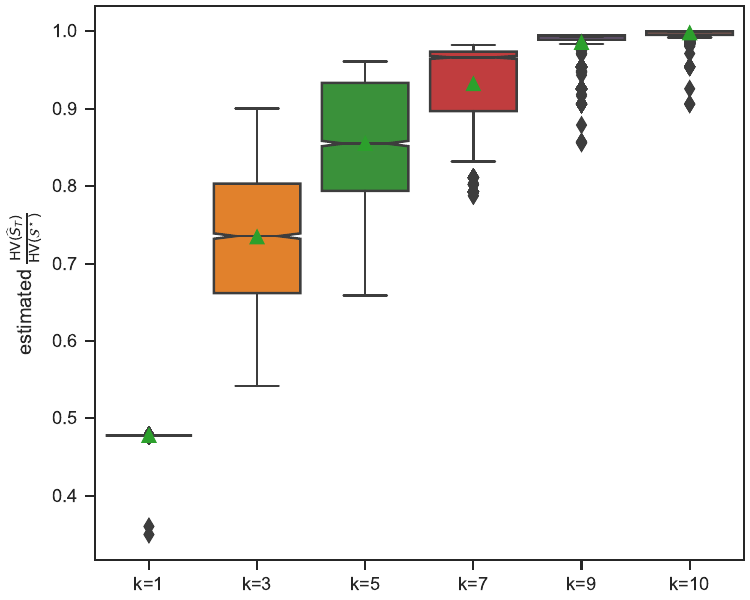}
              \caption{Hyper-volume fraction of the returned set for $T=H$ on Exp.1 (convex Pareto set)}
              \label{fig:srk-fig-2-hpv}
          \end{figure}
      \end{minipage}
      \end{minipage}
   \end{figure*}
\subsubsection{Additional experiments on PSI}
We report additional experimental results on 4 synthetic instances which complements our results reported in the main paper. The new instances are described below. 

{\bfseries Experiment 5: 2 clusters of arms.}  We set $K=20$ and we choose $(\vmu_1, \dots, \vmu_{10}) \sim \mathcal{U}\left([0.2, 0.4]^2\right)^{\otimes 10}$ and $(\vmu_{11}, \dots, \vmu_{20})\sim \mathcal{U}\left([0.5, 0.7]^2\right)^{\otimes 10}$. There are 4 optimal arms for this instance. 

{\bfseries Experiment 6: All the arms are optimal} $K=10, D=2$ and for any arm $i$, $\mu_i^1= 0.75 - 0.65^i, \mu_i^2 = 0.25+0.65^i$

{\bfseries Experiment 7: All the arms have the same sub-optimality gap.} $K=22, D=2$ and we choose $\bm\Theta$ to have $\Delta_1=\dots=\Delta_{22}$. Unlike single-objective bandit, such instances can be generated with all arms being different. We choose for $i=1,\dots 8$, $\vmu_i := (0.3+ c_i, 0.8 -c_i)^\T$ and $c_i = (i-1)*c$. For $i=9,\dots, 15$, $\vmu_i := (0.25 + c_{i-8}, 0.7-c_{i-8})^\T$ and for $i=16, \dots, 22$, $\vmu_i:=\vmu_{i-7} - (0, -0.05)^\T$.

{\bfseries Experiment 8: Geometric progression with a single optimal arm.}  We set $K=5, D=2$ and for any $i\in \{1, \dots, 5\}$, $\mu_i^1 = \mu_i^2 = 0.75 - 0.25^i$. For this instance there is a unique optimal arm.

\begin{figure*}[]
\centering
 \begin{minipage}{\textwidth}
      \centering 
      \begin{minipage}{0.23\linewidth}
          \begin{figure}[H]
              \includegraphics[width=\linewidth]{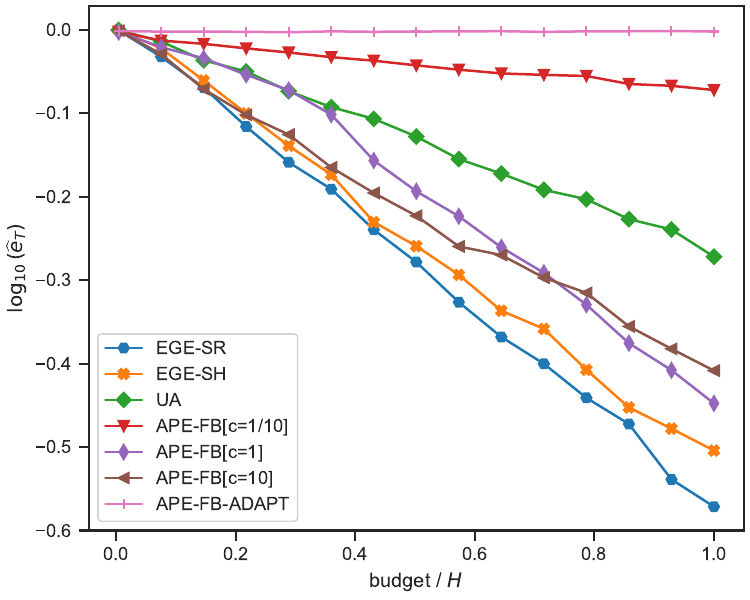}
              \caption{Synthetic Experiment 5: 2 clusters of arms}
              \label{fig:sup-psi-5}
          \end{figure}
      \end{minipage}
 \hspace{0.125cm} 
     \begin{minipage}{0.23\linewidth}
          \begin{figure}[H]
        \includegraphics[width=\linewidth]{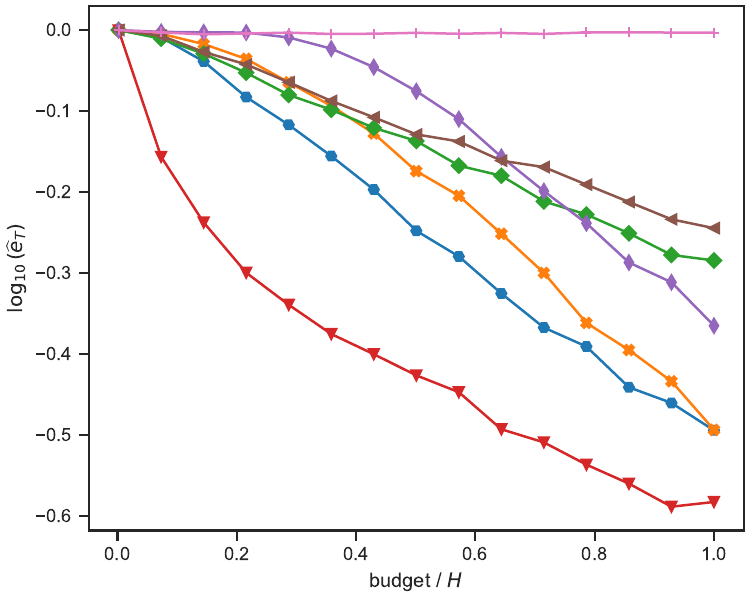}
              \caption{Synthetic Experiment 6: All the arms are optimal.}
              \label{fig:sup-psi-6}
          \end{figure}
      \end{minipage}
 \hspace{0.125cm}
            \begin{minipage}{0.23\linewidth}
          \begin{figure}[H]
              \includegraphics[width=\linewidth]{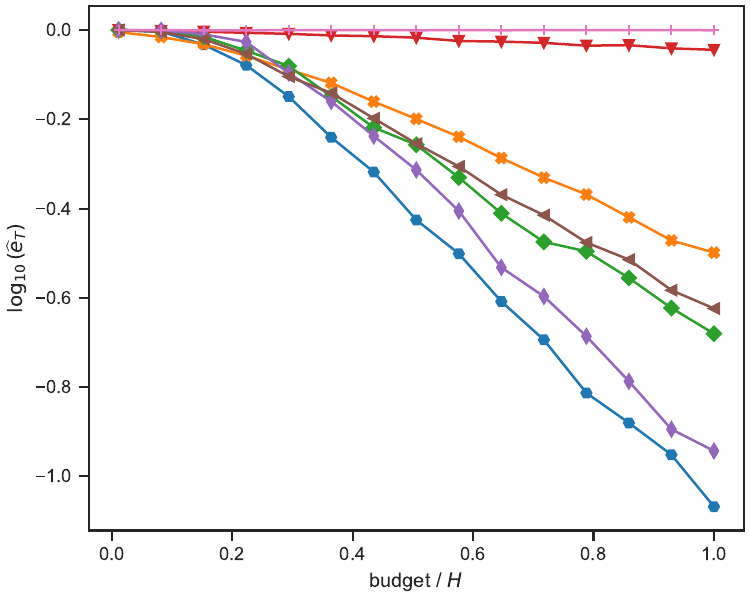}
              \caption{Synthetic Experiment 7: The same gap: $\Delta_i=\Delta$ for each arm}
              \label{fig:sup-psi-7}
          \end{figure}
      \end{minipage}
\hspace{0.125cm}
            \begin{minipage}{0.23\linewidth}
          \begin{figure}[H]
              \includegraphics[width=\linewidth]{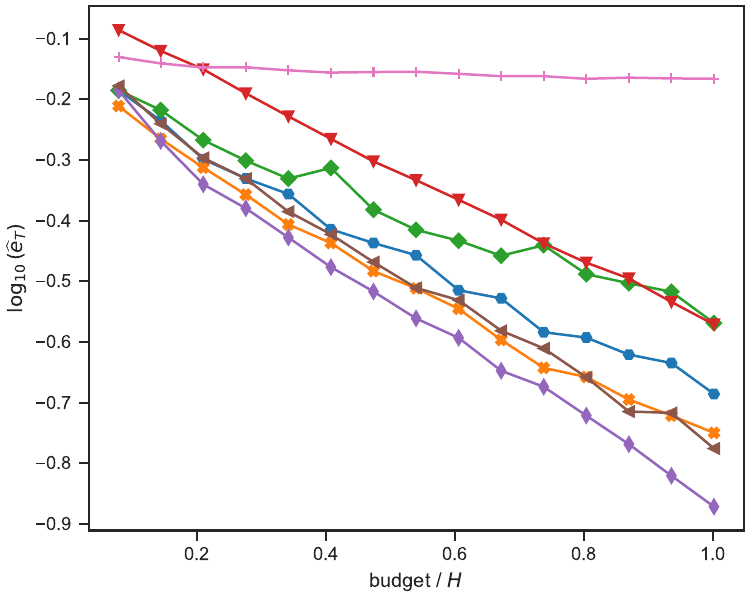}
              \caption{Synthetic Experiment 8: Geometric progression}
                  \label{fig:sup-psi-8}
          \end{figure}
      \end{minipage}
      \end{minipage}
   \end{figure*}

The experiments (Fig.\ref{fig:sup-psi-5} to  Fig.\ref{fig:sup-psi-8}) confirm the superior performance of \ege{} on general instance types. We note however that on experiment 7, where all the arms have the same sub-optimality gaps EGE-SH is slightly outperformed by Uniform allocation but  both are largely outperformed by EGE-SR and \apefb{}. This is not surprising as on such instances, the exponential decay rate of the mis-identification error of uniform allocation is smaller than that of EGE-SH and the "aggressive'' geometrical allocation of SH makes it allocate way less samples to arms that are discarded in early stages. Indeed when $\Delta_1 = \dots = \Delta_K$ we have 
 $K \Delta_{(1)}^{-2} = H_1(\nu) = H_2(\nu)$ thus 
  the result of Tab.\ref{tab:summary} becomes 
 $$ e_T^\text{UA}(\nu) \leq  2(K-1)\lvert\cS^\star\lvert D\exp\left(-\frac{T}{144\sigma^2 H(\nu)}\right),$$
 which is even tighter than \apefb{} when $\log(T)$ is larger than $\lvert \cS^\star\lvert$. For BAI (PSI with $D=1$), it is known \citep{audibert_best_2010} that on instances where the gaps are all the same, the uniform allocation is optimal. For example, by taking $\theta_1 = 0.7$ and $\theta_i = 0.4$ for $i=2,\dots, 8$ and a Bernoulli bandit, the gaps are all the same on this instance and one can observe empirically that the Uniform allocation slightly outperforms Sequential~Halving on this instance. 

\end{document}